\documentclass[9.5pt,journal,compsoc]{IEEEtran}
\ifCLASSOPTIONcompsoc
  \usepackage[nocompress]{cite}
\else
  \usepackage{cite}
\fi
\ifCLASSINFOpdf
\else
\fi
\usepackage[noend,ruled,vlined,linesnumbered]{algorithm2e}
\usepackage{amsmath,amssymb,latexsym,amsfonts}
\usepackage{makecell}
\usepackage{url}
\usepackage{graphicx}
\usepackage[font={small,it}]{subcaption}
\usepackage[font={small,it}]{caption}
\usepackage{times,amsthm}
\usepackage{adjustbox}
\usepackage{float}

\newcommand{\abs}[1]{\left| #1\right|}
\providecommand{\norm}[1]{\left\lVert#1\right\rVert}

\newcommand{\br}[1]{\left\{#1\right\}}
\newcommand{\REAL}{\ensuremath{\mathbb{R}}}

\newcommand{\eps}{\varepsilon}

\newcommand{\dist}{\mathrm{dist}}

\newcommand{\range}{\mathrm{range}}
\newcommand{\ranges}{\mathrm{ranges}}
\newcommand{\lip}{\mathrm{lip}}
\newcommand{\cost}{\mathrm{cost}}
\newcommand{\proj}{\mathrm{proj}}

\newcommand{\algname}{\textsc{Align}}
\newcommand{\matchingalgname}{\textsc{Align+Match}}
\newcommand{\secondalgname}{\textsc{Z-Configs}}
\newcommand{\fastalgname}{\textsc{Fast-Approx-Alignment }}
\newcommand{\LMSalg}{\textsc{LMS }}
\newcommand{\ransacalg}{\textsc{RANSAC }}
\newcommand{\alignments}{\textsc{Alignments}}
\newcommand{\perms}{\mathrm{Perms}}
\renewcommand{\paragraph}[1]{\medskip\noindent\textbf{{#1} }}

\newcommand{\vertiii}[1]{{\left\vert\kern-0.25ex\left\vert\kern-0.25ex\left\vert #1
    \right\vert\kern-0.25ex\right\vert\kern-0.25ex\right\vert}}

\newtheorem{theorem}{Theorem}
\newtheorem{lemma}[theorem]{Lemma}
\newtheorem{definition}[theorem]{Definition}
\newtheorem{corollary}[theorem]{Corollary}
\newtheorem{observation}[theorem]{Observation}
\newtheorem{claim}[theorem]{Claim}

\DeclareMathOperator*{\argmin}{arg\,min}

\DeclareMathOperator*{\arginf}{arg\,inf}

\begin{document}
%
\title{Aligning Points to Lines:\\Provable Approximations }
%
%
%
%

\author{Ibrahim Jubran,
        Dan Feldman
\IEEEcompsocitemizethanks{\IEEEcompsocthanksitem I. Jubran and D. Feldman are with the Robotics \& Big Data Lab, Computer Science Department, University of Haifa, Israel.\protect\\}
}

\IEEEtitleabstractindextext{%
\begin{abstract}
We suggest a new optimization technique for minimizing the sum $\sum_{i=1}^n f_i(x)$ of $n$ non-convex real functions that satisfy a property that we call piecewise log-Lipschitz. This is by forging links between techniques in computational geometry, combinatorics and convex optimization. As an example application, we provide the first constant-factor approximation algorithms whose running-time is polynomial in $n$ for the fundamental problem of \emph{Points-to-Lines alignment}: Given $n$ points $p_1,\cdots,p_n$ and $n$ lines $\ell_1,\cdots,\ell_n$ on the plane and $z>0$, compute the matching $\pi:[n]\to[n]$ and alignment (rotation matrix $R$ and a translation vector $t$) that minimize the sum of Euclidean distances $\sum_{i=1}^n \mathrm{dist}(Rp_i-t,\ell_{\pi(i)})^z$ between each point to its corresponding line.

This problem is non-trivial even if $z=1$ and the matching $\pi$ is given. If $\pi$ is given, the running time of our algorithms is $O(n^3)$, and even near-linear in $n$ using core-sets that support: streaming, dynamic, and distributed parallel computations in poly-logarithmic update time. Generalizations for handling e.g. outliers or pseudo-distances such as $M$-estimators for the problem are also provided.

Experimental results and open source code show that our provable algorithms improve existing heuristics also in practice.
A companion demonstration video in the context of Augmented Reality shows how such algorithms may be used in real-time systems.
\end{abstract}

\begin{IEEEkeywords}
Approximation Algorithms, Non-Convex Optimization, Visual Tracking, Points-to-Lines Alignment, Coresets
\end{IEEEkeywords}}

\maketitle

\section{Introduction\label{sec:probState}}

We define below the general problem of minimizing sum of piecewise log-Lipschitz functions, and then suggest an example application.

\paragraph{Minimizing sum of piecewise log-Lipschitz functions. }
We consider the problem of minimizing the sum $\sum_{g\in G}g(x)$ over $x\in\REAL^d$ of a set $G$ of $n$ real non-negative functions that may not be convex but satisfy a property called the piecewise log-Lipschitz property as in Definition~\ref{def:PieceLip}.

More generally, we wish to minimize the cost function $f(g_1(x),\cdots,g_n(x))$ where $f:\REAL^n\to [0,\infty)$ is a log-Lipschitz function as in Definition~\ref{def:lip}, and $\br{g_1,\cdots,g_n}$ is the set of piecewise log-Lipschitz functions in $G$.

As an application, we reduce the following problem to minimizing such a set $G$ of functions.

\begin{figure}[ht]
	\centering
    \includegraphics[width=0.5\textwidth,scale=0.5]{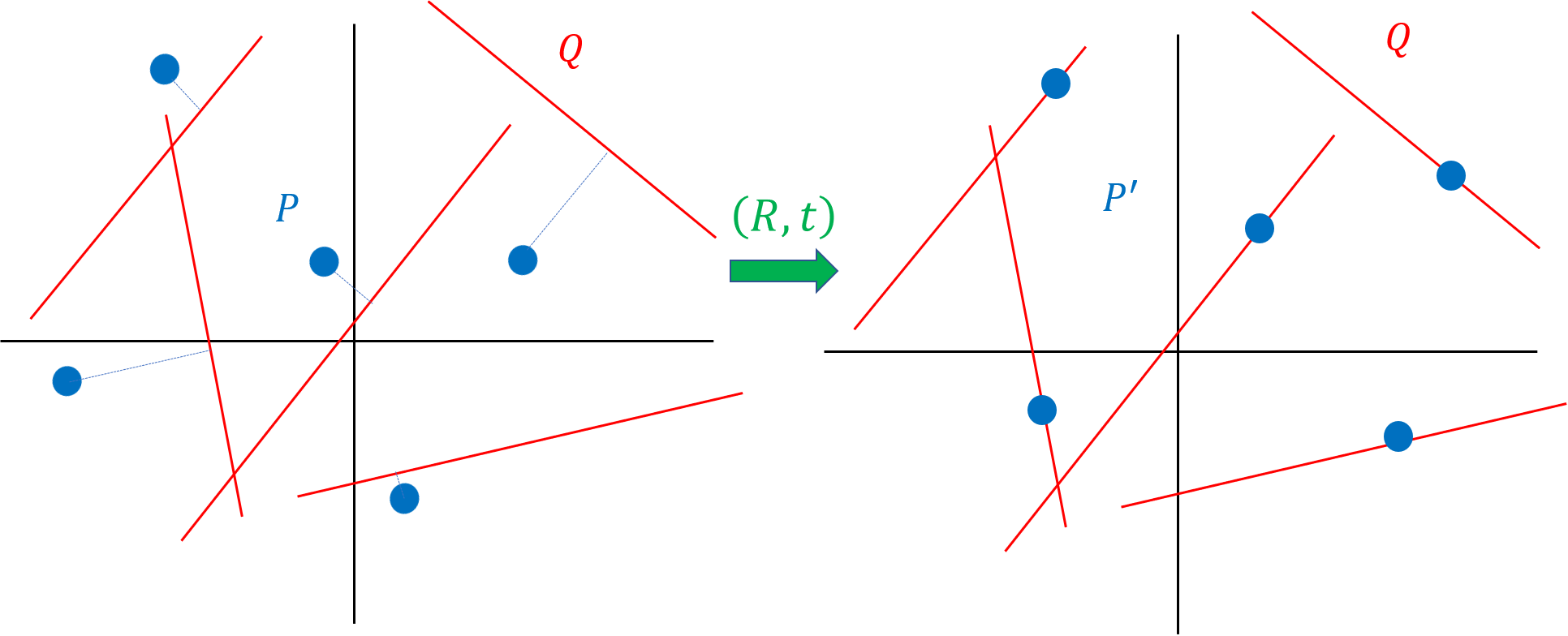}
    \caption{\textbf{Alignment of points to lines.} Finding the optimal alignment $(R,t)$ that aligns the set of points $P$ to the set $L$ of lines. The set $P'$ is the result of applying the alignment $(R,t)$ to the initial set $P$.}
    \label{fig:AligningPointsToLines}
\end{figure}

\paragraph{Points-to-Lines alignment } is a fundamental problem in computational geometry, and is a direct generalization of the widely known problem of point set registration, where one is required to align to sets of points; see~\cite{tam2013registration} and references therein. The alignment of points to lines is also known as the Perspective-n-Point (PnP) problem from computer vision, which was first introduced in~\cite{fischler1981random}, which aims to compute the position of an object (formally, a rigid body) based on its position as detected in a 2D-camera. Here, we assume that the structure of the object (its 3D model) is known. This problem is equivalent to the problem of estimating the position of a moving camera, based on a captured 2D image of a known object, which is strongly related to the very common procedure of ``camera calibration" that is used to estimate the external parameters of a camera using a chessboard.
More example applications include object tracking and localization, localization using sky patterns and many more; see Fig.~\ref{fig:AligningPointsToLines},~\ref{fig:AligningFrame} and~\ref{fig:UrsaMajor}

Formally, the input to these problems is an ordered set $L=\br{\ell_1,\cdots,\ell_n}$ of $n$ lines that intersect at the origin and an ordered set $P=\br{p_1,\cdots,p_n}$ of $n$ points, both in $\REAL^3$. Each line represents a point in the 2D image.
The output is a pair $(R,t)$ of rotation matrix $R\in \REAL^{3\times 3}$ and translation vector $t\in \REAL^3$ that minimizes the sum of Euclidean distances over each point (column vector) and its corresponding line, i.e.,
\begin{equation}\label{min}
\min_{(R,t)}\sum_{i=1}^n \dist(Rp_i-t,\ell_i),
\end{equation}
where the minimum is over every rotation matrix $R\in\REAL^{3\times 3}$ and translation vector $t\in\REAL^3$.
Here, $\dist(x,y)=\norm{x-y}_2$ is the Euclidean distance but in practice we may wish to use non-Euclidean distances, such as distances from a point to the intersection of its corresponding line with the camera's image plane.

While dozens of heuristics were suggested over the recent decades, to our knowledge, this problem is open even when the points and lines are on the plane, e.g. when we wish to align a set of GPS points to a map of lines (say, highways).

We tackle a variant of this problem, when both the points and lines are in $\REAL^2$, and the lines do not necessarily intersect at the origin.

\begin{figure*}[h!]
	\centering
    \includegraphics[width=0.9\textwidth]{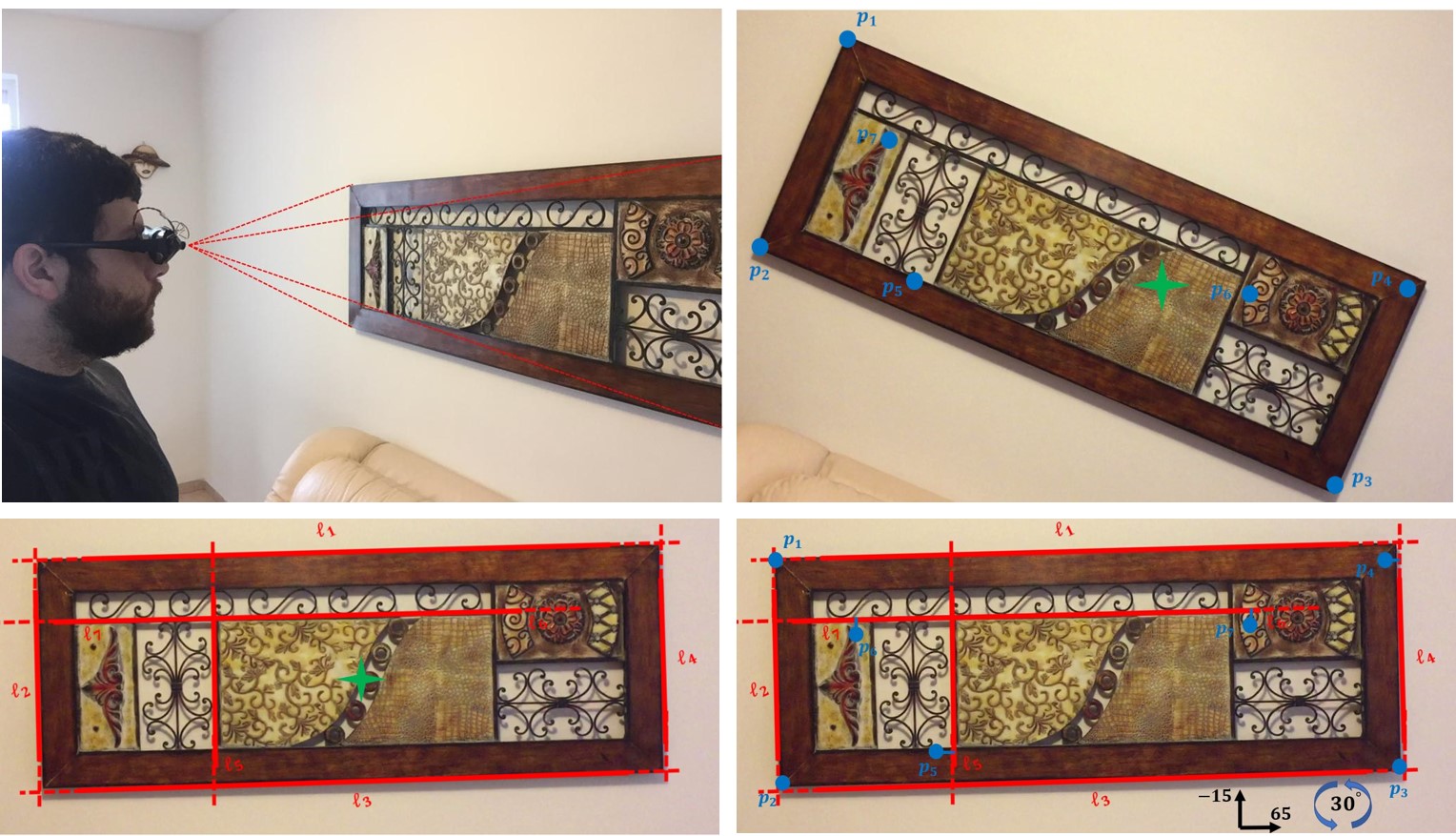}
    \caption{(Top left) A potential AR application. A user observing a scene through AR glasses. The goal is to present virtual objects to the user by placing them on top of the currently observed frame.
    Object tracking and localization for the AR glasses can be done by aligning the points in the observed image (top right) to the predefined lines of the object (bottom left). The aligned sets are presented in the bottom right image. The images were taken from our companion supplementary material video~\cite{DemoVideo}.}
    \label{fig:AligningFrame}
\end{figure*}

\subsection{Generalizations}\label{gen}
We consider more generalizations of the above problem as follows.

\paragraph{Unknown matching }is the case where the matching $\pi:[n]\to[n]$ between the $i$th point $p_i$ to its line $\ell_i$ in the PnP problem~\eqref{min} is not given for every $i\in\br{1,\cdots,n}=[n]$. E.g., in the PnP problem there are usually no labels in the observed images. Instead, the minimization is over every pair of rotation and translation $(R,t)$ \emph{and} matching (bijective function) $\pi:[n]\to [n]$,
\begin{equation}\label{min2}
\min_{(R,t), \pi} \dist(Rp_i-t,\ell_{\pi(i)}),
\end{equation}
where $\pi(i)$ is the index of the line corresponding to the $i$th point $p_i$.

\paragraph{Non-distance functions }where for an error vector $v=(v_1,\cdots,v_n)$ that corresponds to the $n$ input point-line pairs given some candidate solution, the cost that we wish to minimize over all these solutions is $f(v)$ where $f:\REAL^n\to [0,\infty)$ is a log-Lipschitz function as in Definition~\ref{def:lip}. Special cases include $f(v)=\norm{v}_\infty$ for ``worst case" error, $f(v)=\norm{v}_{1/2}$ which is more robust to noisy data, or $f(v)=\norm{v}^2$ of squared distances for maximizing likelihood in the case of Gaussian noise. As in the latter two cases, the function $f$ may not be a distance function.

\paragraph{Robustness to outliers }may be obtained by defining $f$ to be a function that ignores or at least puts less weight on very large distances, maybe based on some given threshold. Such a function $f$ is called an \emph{M-estimator} and is usually a non-convex function. M-estimators are in general robust statistics functions, such as the well known Huber loss function; see~\cite{huber2011robust}.

\paragraph{Coreset }in this paper refers to a small representation of the paired input point-line sets $P$ and $L$ by a weighted (scaled) subset. The approximation is $(1+\eps)$ multiplicative factor, with respect to the cost of any item (query). E.g. in~\eqref{min} the item (query) is a pair $(R,t)$ of rotation matrix $R$ and translation vector $t$. \emph{Composable coresets} for this problem have the property that they can be merged and re-reduced; see e.g.~\cite{AHV04,feldman2015more}.

Our main motivation for using existings coresets is (i) to reduce the running time of our algorithms from polynomial to near-linear in $n$, and (ii)  handle big data computation models as follows, which is straightforward using composable coresets.

\paragraph{Handling big data } in our paper refers to the following computation models:
 (i) \textbf{Streaming }support for a single pass over possibly unbounded stream of items in $A$ using memory and update time that is sub-linear (usually poly-logarithmic) in $n$. (ii) \textbf{Parallel computations }on distributed data that is streamed to $M$ machines where the running time is reduced by $M$ and the communication between the machines to the server should be also sub-linear in its input size $n/M$. (iii) \textbf{Dynamic data }which includes deletion of pairs. Here $O(n)$ memory is necessary, but we still desire sub-linear update time.

\begin{figure*}[hbtp]
\centering
	\begin{subfigure}[t]{0.3\textwidth}
    \centering
    \includegraphics[width=0.99\textwidth]{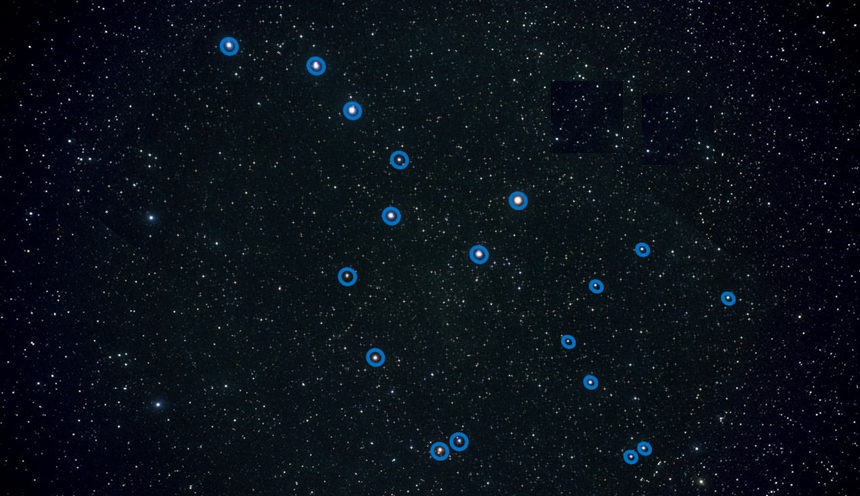}
    \caption{An observed image of the stars in Ursa Major. Each star represents a point and is marked with a blue circle.}
    \vspace{0.2cm}
    \label{fig:UrsaMajorPoints}
    \end{subfigure}
    \hfill
    \begin{subfigure}[t]{0.3\textwidth}
    \centering
    \includegraphics[width=0.99\textwidth]{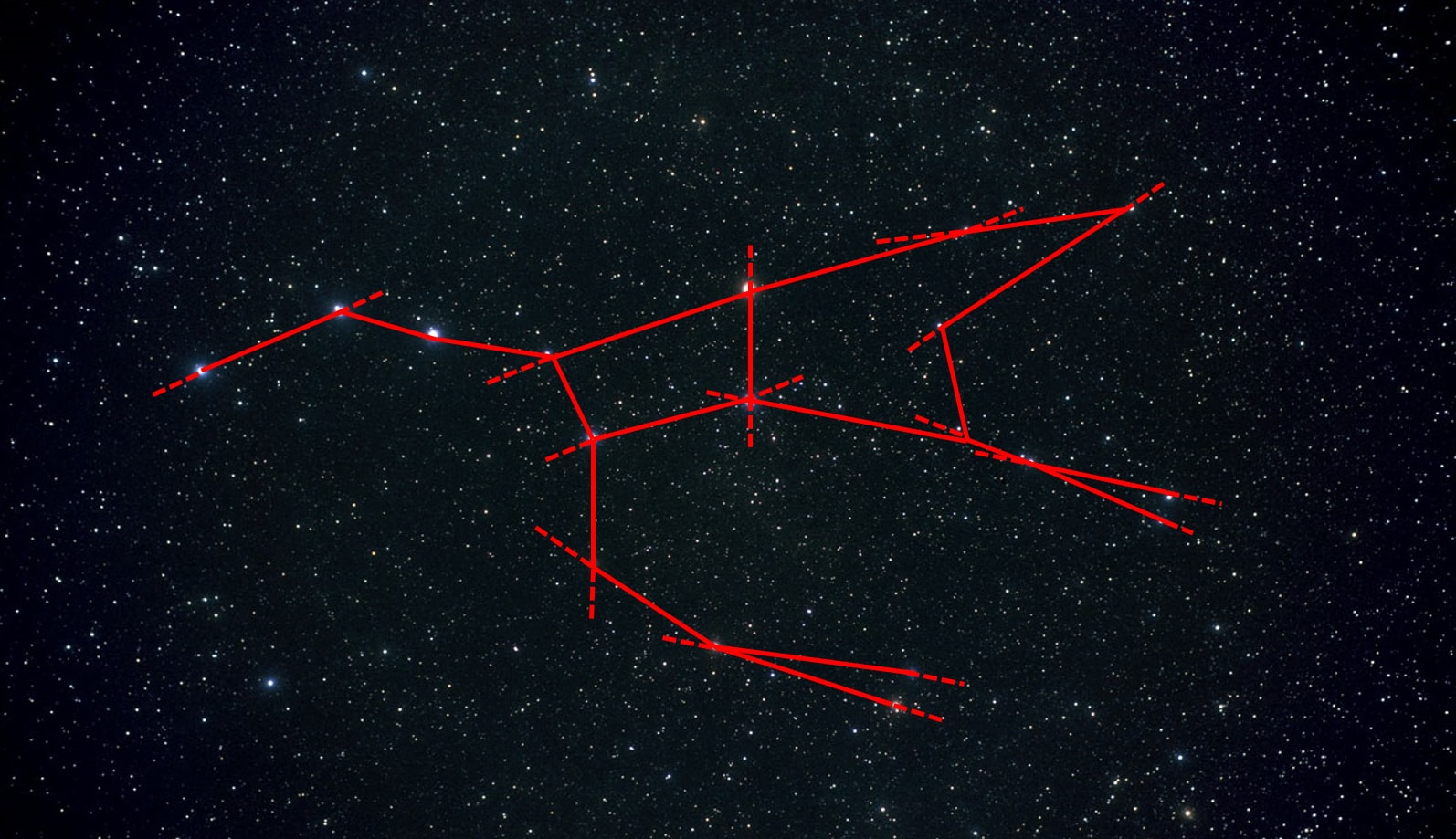}
    \caption{A known (initial) model, represented as a set of lines.}
    \label{fig:UrsaMajorLines}
    \end{subfigure}
    \hfill
    \begin{subfigure}[t]{0.3\textwidth}
    \centering
    \includegraphics[width=0.997\textwidth]{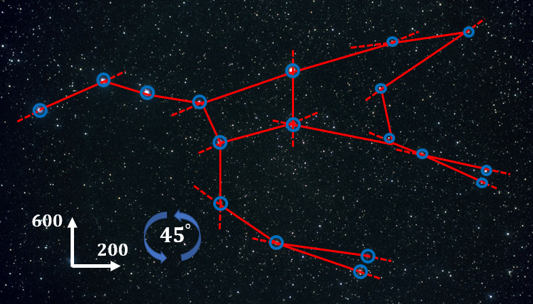}
    \caption{The alignment of the sets of points to the set of lines.}
    \label{fig:UrsaMajorLines}
    \end{subfigure}
    \vspace{-0.5cm}
    \caption{\textbf{Localization using sky patterns.}} 
    \label{fig:UrsaMajor}
\end{figure*}

\subsection{Our contribution} \label{sec:ourContrib}
The contribution of this paper is summarized as follows.

\paragraph{Optimization of piecewise log-Lipschitz functions.} Given $x^*\in\REAL$ and piecewise $O(1)$-log-Lipschitz functions $g_1,\cdots,g_n$, we prove that one of their minima $x'\in\REAL$ approximates their value $g_i(x)$ simultaneously (for every $i\in [n]$), up to a constant factor. See Theorem~\ref{Lem:piecewiseApprox}. This yields a finite set of candidate solutions (called \emph{centroid set}) that contains an approximated solution, without knowing $x^*$.

\paragraph{Generic framework }for defining a cost function $\cost(A,q)$ for any finite \emph{input subset} $A=\br{a_1,\cdots,a_n}$ of a set $X$ called \emph{ground set}, and an item $q$ (called \emph{query}) from a (usually infinite) set $Q$. We show that this framework enables handling the generalizations in Section~\ref{gen} such as outliers, M-estimators and non-distance functions. Formally, we define
\begin{equation}\label{cost}
\cost(A,q):=f\left(\lip\left(D(a_1,q)\right),\cdots,\lip\left(D(a_n,q)\right)\right),
\end{equation}
where $f:[0,\infty)^n \to [0,\infty)$ and $\lip:[0,\infty)\to[0,\infty)$ are $O(1)$-log-Lipschitz functions, and $D:X\times Q \to [0,\infty)$; See Definition~\ref{def:cost}.

We use this framework along with the first result to compute $q'\in Q$ that approximates $D(a_i,q^*)$ simultaneously  (for every $i\in[n]$), where $q^*$ is the query that minimizes $\cost(A,q)$ over every $q\in Q$. Observation~\ref{obs:distToCost} proves that $\cost(A,q')$ is the desired approximation for the optimal solution $\cost(A,q^*)$ in our framework.

\paragraph{Simultaneous optimization and matching }may be required for the special case that $A=\br{(a_i,b_i)}_{i=1}^n$ is a set of pairs, and we wish to compute the permutation $\pi^*$ and query $q^*$ that minimize $\cost(A_{\pi},q)$ over every $\pi^*:[n]\to[n]$ and $q^*\in Q$, where $A_{\pi}=\br{(a_i,b_{\pi(i)})}_{i=1}^n$ is the corresponding permutation of the pairs. We provide constant factor approximations for the case $f=\norm{\cdot}_1$ in~\eqref{cost}; See Theorem~\ref{theorem:pointsToLinesNoMatching}.

\paragraph{Approximated Points-to-Lines alignment }as defined in Section~\ref{sec:probState} is the main example application for our framework,, where $A=\br{(p_i,\ell_i)}_{i=1}^n$ is a set of point-line pairs $(p_i,\ell_i)$, both on the plane, $Q=\alignments$ is the union over every rotation matrix $R$ and a translation vector $t$, $D\left((p_i,\ell_i),(R,t)\right)= \min_{x\in \ell_i} \norm{Rp_i-t-x}$ for every $(R,t)\in Q$ and $i\in [n]$. We provide the first constant factor approximation for the optimal solution $\min_{(R,t)\in Q}\cost(A,(R,t))$ that takes time polynomial in $n$; see Theorem~\ref{theorem:costPointsToLines}. Using existing coresets we reduce this running time to near linear in $n$; see Theorem~\ref{theorem:coreset}. Such a solution can be computed for every cost function as defined in~\eqref{cost}. We also solve this problem for the simultaneous optimization and matching scenario for $f=\norm{\cdot}_1$ in time polynomial in $n$; See Theorem~\ref{theorem:pointsToLinesNoMatching}.

\begin{table*}[h!]
\begin{adjustbox}{width=1\textwidth}
\begin{tabular}{| *{9}{c|} }
\hline
Function Name & $f(v):=$ & $\lip(x):=$ & $z$ & $r$ & \makecell{Time} & \makecell{Approx. factor} & \makecell{Theorem} & Related Work\\
\hline
Squared Euclidean distances & $\norm{v}_1$ & $x^2$ & $2$ & $2$ & $O(n^3)$ & $16^2$ & \ref{theorem:costPointsToLines} & See Section~\ref{sec:RelatedWork} \\
\hline
\makecell{Euclidean distances\\without using coreset} & $\norm{v}_1$ & $x$ & $2$ & $1$ & $O(n^3)$ & $16$ & \ref{theorem:costPointsToLines} & None\\
\hline
\makecell{Euclidean distances\\using coreset} & $\norm{v}_1$ & $x$ & $2$ & $1$ & $O(n\log{n})$& $(1+\varepsilon)16$ & \ref{theorem:coreset} & None\\
\hline
$\ell_z^r$ distances & $\norm{v}_1$ & $x^r$ & $z \geq 1$ & $r \geq 1$ & $O(n^3)$ & $(\sqrt{2} \cdot 16)^r$ & \ref{theorem:costPointsToLines} & None\\
\hline
Euclidean distances, $k$ outliers & $\norm{v}_{1,n-k}$ & $x$ & $z \geq 1$ & $1$ & $O(n^3)$ & $16$ & \ref{theorem:costPointsToLines} & RANSAC~\cite{derpanis2010overview}\\
\hline
$\ell_2$ distance, unknown matching $\star$& $\norm{v}_1$ & $x$ & $2$ & $1$ & $O(n^9)$ & $16$ & \ref{theorem:pointsToLinesNoMatching} & ICP~\cite{besl1992method}\\
\hline
$\ell_z^r$ distances, unknown matching $\star$ & $\norm{v}_1$ & $x^r$ & $z\geq 1$ & $r \geq 1$ & $O(n^9)$ & $(\sqrt{2}\cdot 16)^r$ & \ref{theorem:pointsToLinesNoMatching} & None \\
\hline
\end{tabular}
\end{adjustbox}
\caption{\textbf{Main results of this paper for the problem of aligning points-to-lines.} Let $P$ be a set of $n$ points and $L$ be a set of $n$ lines, both in $\REAL^2$. Let $\pi:[n]\to [n]$ be a matching function and let $A_{\pi} = \br{(p_1,\ell_{\pi(1)}),\cdots,(p_n,\ell_{\pi(n)})}$. In this table we assume $x\in [0,\infty)$, $v \in [0,\infty)^n$ and $\varepsilon > 0$. For every point-line pair $(p,\ell) \in A$ and $(R,t)\in \alignments$, let $D_z\left((p,\ell),(R,t)\right) = \min_{x'\in \ell} \norm{Rp-t-x'}_z$ for $z\geq 1$. Let $\cost$ be as defined in Definition~\ref{def:cost}, for $D = D_z$, $A = A_{\pi}$ and $s=1$. The approximation factor is relative to the minimal value of the $\cost$ function.
Rows marked with a $\star$ have that the minimum of the function $\cost$ is computed over every $(R,t)\in \alignments$ and matching function $\pi$.
The computation time given is the time it takes to compute a set of candidate alignments. Hence, an additional simple exhaustive search might be needed.}
\label{table:ourContrib}
\end{table*}

\paragraph{Streaming and distributed }computation for the problem of aligning points-to-lines is easy by applying our algorithm on existing coresets, as
suggested in Corollary~\ref{cor:streamingCoreset}.


\paragraph{Experimental Results }show that our algorithms perform better also in practice, compared to both existing heuristics and provable algorithms for related problems. A system for head tracking in the context of augmented reality shows that our algorithms can be applied in real-time using sampling and coresets. Existing solutions are either too slow or provide unstable images, as is demonstrated in the companion video~\cite{DemoVideo}.

\section{Related Work} \label{sec:RelatedWork}
For the easier case of summing convex function, a similar framework was suggested in~\cite{FL11,feldman2007bi}. However, for the case of summing non-convex functions as in this paper, each with more than one local minima, these techniques do not hold. This is why we had to use more involved algorithms in Section~\ref{sec:pointsToLines}. Moreover, our generic framework such as handling outliers and matching can be applied also for the works in~\cite{FL11,feldman2007bi}.


Summing non-convex but polynomial or rational functions was suggested in~\cite{vigneron2014geometric}. This is by using tools from algebraic geometry such as semi-algebraic sets and their decompositions.
While there are many techniques and solvers that can solve sets of polynomials of constant rational degree $z$~\cite{chazelle1991singly,grigor1988solving,frank1956algorithm}, in this paper we show how to handle e.g. the case $z=\infty$ (max distance to the model), and non-polynomial error functions that are robust to outliers such as M-estimators and a wide range of functions, as defined in~\eqref{cost}. Our experimental results show that even for the case of constant $z$, our algorithm is much faster due to the large constants that are hidden in the $O()$ notation of those techniques compared to our running time.
For high degree polynomials, the known solvers may be used to compute the minima in Theorem~\ref{Lem:piecewiseApprox}. In this sense, piecewise log-Lipschitz functions can be considered as generalizations of such functions, and our framework may be used to extend them for the generalizations in Section~\ref{gen} (outliers, matching, etc.).

\paragraph{Aligning points to lines. }The problem for aligning a set of points to a set of lines in the plane is natural e.g. in the context of GPS points~\cite{DBLP:conf/iros/SungFR12,DBLP:conf/icra/PaulFRN14}, localizing and pose refinement based on sky patterns~\cite{needelman2006refinement} as in Fig.~\ref{fig:UrsaMajor}. Moreover, object tracking and localization can be solved by aligning pixels in a 2D image to an object that is pre-defined by linear segments~\cite{DBLP:conf/iros/MariottiniP07,hijikata2009simple}, as in augmented reality applications~\cite{geisner2015realistic}.


The only known solutions are for the case of sum of squared distances and $d=2$ dimensions, with no outliers, and when the matching between the points and lines is given. In this case, the Lagrange multipliers method can be applied in order to get a set of second order polynomials. If the matching is unknown, an Iterative Closest Point (ICP) method can be applied. The ICP is a known technique based on greedy nearest neighbours; see references in~\cite{besl1992method}.
For $d=3$, the problem is called PnP (Perspective-n-Points), where one wishes to compute the rotation and translation of a calibrated camera given 3D points and their corresponding 2D projection, which can be interpreted as 3D lines that intercept the origin. While many solutions were suggested for this problem over the years~\cite{li2012robust, lepetit2009epnp, zheng2013revisiting,wang2018simple,urban2016mlpnp}, those methods either: 1) do not have provable convergence guarantees, 2) do not provide guarantees for optimality, 3) minimize some algebraic error, which is easy to solve and somehow related to the original problem, but whose solution does not prove a bounded approximation to the original problem, or 4) assume zero reprojection / fitting error.
To handle outliers RANSAC~\cite{derpanis2010overview, fischler1981random} is heuristically used.

While this natural problem was intensively studied for the special case of $d=3$ where the set of lines intercept the origin, to the best of our knowledge, there are no known results so far  which have successfully and provably tackled this problem either in higher dimensions, in its generalized form where the lines do not necessarily intersect, with non-convex minimization functions such as M-estimators, or when the matching between the two sets is unknown.

\paragraph{Coresets }have many different definitions. In this paper we use the simplest one that is based on a weighted \emph{subset} of the input, which preserve properties such as sparsity of the input and numerical stability. Coresets for $\ell_p$ regression were suggested in~\cite{dasgupta2009sampling} using well-conditioned matrices that we cite in Theorem~\ref{theorem:coreset}. Similar coresets of size $d^{O(cp)}$ for $\ell_p$ regression were later suggested in~\cite{cohen2015p}, with a smaller constant $c>1$. Both~\cite{dasgupta2009sampling} and~\cite{cohen2015p} could have been used for coreset constructed in this paper.
We reduce the points-to-lines aligning problem to constrained $\ell_1$ optimization in the proof of Theorem~\ref{theorem:coreset}, which allows us to apply our algorithms on these coresets for the case of sum over point-line distances.

In most coresets papers, the main challenge is to compute the coreset. However, as we use existing coresets in this paper, the harder challenge was to extract the desired constrained solution from the coreset which approximate every possible alignment, with or without the constraints.

\section{Optimization Framework}
In what follows, for every pair of vectors $v=(v_1,\cdots,v_n)$ and $u=(u_1,\cdots,u_n)$ in $\REAL^n$ we denote $v\leq u$ if $v_i \leq u_i$ for every $i\in [n]$. Similarly, $f:\REAL^n \to [0,\infty)$ is non-decreasing if $f(v)\leq f(u)$ for every $v\leq u \in \REAL^d$. For a set $I \subseteq \REAL$ and $c \in \REAL$, we denote $\frac{I}{c} = \br{\frac{a}{c} \mid a \in I}$.

The following definition is a generalization of Definition 2.1 in~\cite{feldman2012data} from $n=1$ to $n>1$ dimensions, and from $\REAL$ to $I\subseteq\REAL^n$. If a function $h$ is a log-Lipschitz function, it roughly means that the functions slope does not change significantly if its input only changes slightly.
\begin{definition}[Log-Lipschitz function] \label{def:lip}
Let $r>0$ and $n\geq 1$ be an integer. Let $I$ be a subset of $\REAL^n$, and $h:I\to[0,\infty)$ be a non-decreasing function.
Then $h(x)$ is \emph{$r$-log-Lipschitz} over $x\in I$, if for every $c \geq 1$ and $x\in \displaystyle I \cap \frac{I}{c}$, we have $h(c x)\leq c^r h(x).$ The parameter $r$ is called the \emph{log-Lipschitz constant}.
\end{definition}

Unlike previous papers, the loss fitting (``distance'') function that we want to minimize in this paper is not a log-Lipshcitz function. However, it does satisfy the following property, which implies that we can partition the range of the function $g:\REAL^d\to[0,\infty)$ into $m$ subsets of $\REAL^d$ (sub-ranges), such that $g$ satisfies the log-Lipschitz condition on each of these sub-ranges; see Figure~\ref{figLip}. That is, $g$ has a single minimum in this sub-range, and increases in a bounded ratio around its local minimum. Note that $g$ might not be convex even in this sub-range. Formally, if the distance $x$ from the minimum in a sub-range is multiplied by $c>0$, then the value of the function increases by a factor of at most $c^r$ for some small (usually constant) $r>0$. For example, if we double the distance from the local minimum, then the value of the function increases by at most a constant factor of $2^r$.

\begin{definition}[Piecewise log-Lipschitz\label{def:PieceLip}]
Let $g:X\to[0,\infty)$ be a continuous function whose range is $X$, and
let $(X,\dist)$ be a metric space, i.e. $\dist:X^2\to[0,\infty)$ is a distance function. Let $r\geq 0$. The function $g$ is \emph{piecewise $r$-log-Lipschitz} if there is a partition of $X$ into $m$ subsets $X_1,\cdots, X_m$
such that for every $i\in[m]$:
\begin{enumerate}
\renewcommand{\labelenumi}{(\roman{enumi})}
\item $g$ has a unique infimum $x_i$ in $X_i$, i.e., $\br{x_i}= \argmin_{x\in X_i}g(x)$.
\item $h_i:[0,\max_{x\in X_i} \dist(x,x_i)]\to[0,\infty)$ is an $r$-log-Lipschitz function; see Definition~\ref{def:lip}.
\item $g(x)=h_i(\dist(x_i,x))$ for every $x\in X_i$.
\end{enumerate}

The union of minima is denoted by $M(g)=\br{x_1,\cdots,x_m}$.
\end{definition}

\begin{figure} [ht]
\centering
\includegraphics[width=0.3\textwidth]{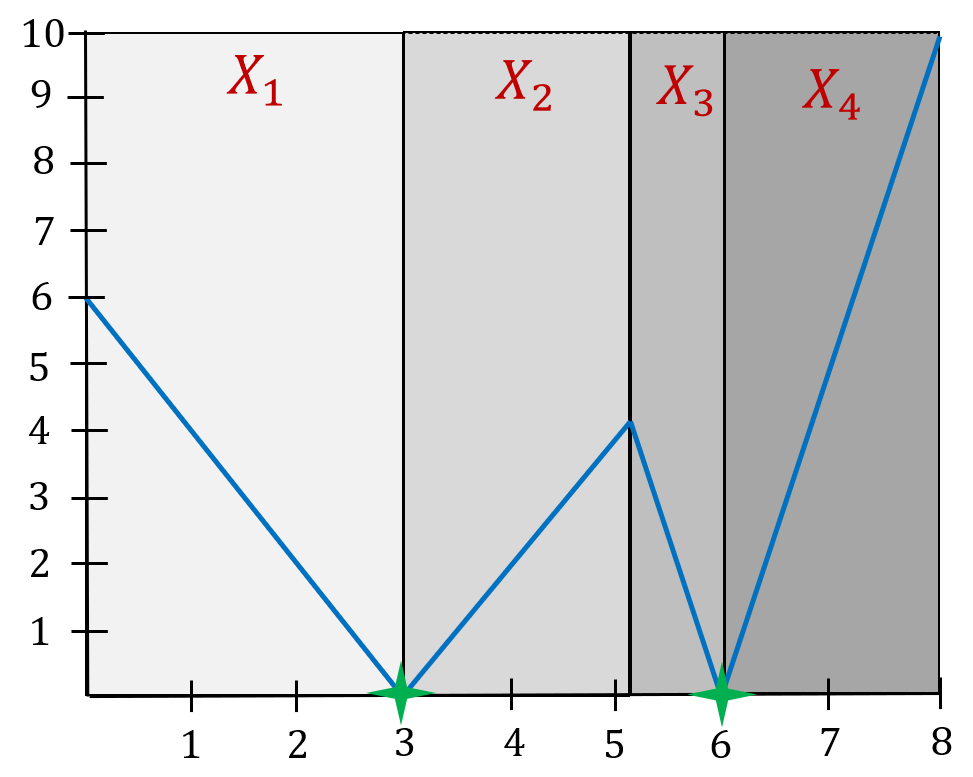}
\caption{\textbf{An example of a piecewise log-Lipschitz function: }
A function $g(x) = \min \br{2\cdot |x-3|, 5\cdot |x-6|}$ (blue graph) over the set $X = \REAL$. $X$ can be partitioned into $4$ subsets $X_1,\cdots,X_4$, where each subset has a unique infimum $x_1=3,x_2=3,x_3=6$ and $x_4=6$ respectively (green stars). There exist $4$ $1$-log-Lipschitz functions $h_1(x)=h_2(x)=2x$ and $h_3(x)=h_4(x)=5x$, such that $g(x) = h_i(|x_i-x|)$ for every $x\in X_i$.}
\label{figLip}
\end{figure}

Suppose that we have a set of piecewise $r$-log-Lipschitz functions, and consider the union $\bigcup M(g)$ over every function $g$ in this set. The following lemma states that, for every $x\in\REAL$, this union contains a value $x'$ such that $g(x')$ approximates $g(x)$ up to a multiplicative factor that depends on $r$.

\begin{theorem}[simultaneous approximation]\label{Lem:piecewiseApprox}
Let $g_1,\cdots,g_n$ be $n$ function, where $g_i:\REAL \to [0,\infty)$ is a piecewise $r$-log-Lipschitz function for every $i\in [n]$, and let $M(g_i)$ denote the minima of $g_i$ as in Definition~\ref{def:PieceLip}. Let $x\in \REAL$. Then there is $x'\in \bigcup_{i\in [n]}M(g_i)$ such that for every $i\in [n]$,
\begin{equation} \label{eqToProvePLip}
g_i(x')\leq 2^rg_i(x).
\end{equation}
\end{theorem}
\begin{proof}
Let $x'\in\bigcup_{i\in [n]}M(g_i)$ be the closest item to $x$, i.e., that minimizes $\dist(x',x)$ (Ties are broken arbitrarily).
Put $i\in [n]$. We have that $g_i$ is a piecewise $r$-log-Lipschitz function. Identify $M(g_i)=\br{x_1,\cdots,x_m}$, let $X_1,\cdots,X_m$ and $h_1,\cdots,h_m$ be a partition of $\REAL$ and a set of functions that satisfy Properties (i)--(iii) in Definition~\ref{def:PieceLip} for $g_i$. Let $j\in[m]$ such that $x\in X_j$.
The rest of the proof holds by the following case analysis: \textbf{(i)} $x'\in X_j$,  \textbf{(ii)} $x'\in X_{j+1}$ and $j \leq m-1$, and \textbf{(iii)} $x'\in X_{j-1}$ and $j \geq 2$. There are no other cases since if $x'\in X_{j-2}$ or $x' \in X_{j+2}$, it will imply that $\dist(x_{j-2},x) < \dist(x',x)$ or $\dist(x_{j+2},x) < \dist(x',x)$ respectively, which contradicts the definition of $x'$.

\noindent\textbf{Case (i) $x' \in X_j$. }We first prove that~\eqref{eqToProvePLip} holds for every $y \in X_j$ that satisfies $\dist(x,y) \leq \dist(x,x_j)$. If we do so, then ~\eqref{eqToProvePLip} trivially holds for Case (i) by substituting $y=x'$ since $\dist(x,x') \leq \dist(x,x_j)$.

Let $y \in X_j$ such that $\dist(x,y) \leq \dist(x,x_j)$. Then we have
\begin{equation}\label{ddr}
\dist(x_j,y) \leq \dist(x_j,x)+\dist(x,y)\leq 2\dist(x_j,x)
\end{equation}
by the definition of $y$ and the triangle inequality. This proves~\eqref{eqToProvePLip} as
\[
\begin{split}
g_i(y)& = h_j(\dist(x_j,y)) \leq h_j(2\dist(x_j,x))\\ & \leq 2^r h_j(\dist(x_j,x)) = 2^r g_i(x),
\end{split}
\]
where the first and last equalities hold by the definition of $h_j$, the first inequality holds since $h_j$ is a monotonic non-decreasing function and by~\eqref{ddr}, and the second inequality holds since $h_j$ is $r$-log-Lipschitz in $X_j$ by Property (ii) of Definition~\ref{def:lip}.

\noindent\textbf{Case (ii) $x'\in X_{j+1}$ and $j \leq m-1$. }In this case $x'< x_{j+1}$ by its definition and the fact that $x\in X_j$.
Hence, $x'\in(x_{j},x_{j+1})$. Let $y\in(x_{j},x_{j+1})$ such that $g_i(y) = \sup_{y'\in (x_j,x_{j+1})} g_i(y')$.
Combining the definition of $y$ with the assumption that $g_i$ is continuous and non-decreasing in $(x_j,x_{j+1})\cap X_j$, and is non-increasing in $(x_j,x_{j+1})\cap X_{j+1}$, we have that
\begin{equation} \label{eq:gilim}
\lim_{z\to y^+} g_i(z) = \lim_{z\to y^-} g_i(z).
\end{equation}

Hence,
\[
g_i(x') \leq \lim_{z\to y^+} g_i(z) = \lim_{z\to y^-} g_i(z) \leq 2^r g_i(x),
\]
where the first derivation holds by combing that $g_i$ is non increasing in $(x_j,x_{j+1})\cap X_{j+1}$ and the definition of $y$, the second derivation is by~\eqref{eq:gilim}, and since $\dist(x,z) \leq \dist(x,x') \leq \dist(x,x_j)$ the last derivation holds by substituting $y = z$ in Case (i).

\noindent\textbf{Case (iii) $x'\in X_{j-1}$ and $j \geq 2$. }The proof for this case is symmetric to Case (ii).

Hence, in any case, there is $x'\in \bigcup_{i\in [n]}M(g_i)$ such that for every $i\in [n], g_i(x')\leq 2^rg_i(x)$.
\end{proof}

In what follows we define a framework for approximating general sum of functions.
The goal is to minimize some loss function $\cost(A,q)$ over a query set $Q$. This function is a function $f$ of $n$ pseudo distances, e.g. sum, max or sum of squares. A pseudo distance between an input point $a\in A$ and a query $q\in Q$ is a Lipshitz function $\lip$ of their distance $D(a,q)$.

In the rest of the paper, for every minimization problem that is presented, we shall define such a cost function by specifying the functions $D$, $\lip$ and $f$, and show that we can minimize it up to some constant factor.

For example, in Section~\ref{sec:pointsToLines} we present the problem of aligning points to lines in 2 dimensional space, where the input set $A$ is a set of $n$ point-line pairs. The goal is to compute a 2D rotation matrix $R\in\REAL^{2\times 2}$ and a 2D translation vector $t\in \REAL^2$ that, when applied to the set of points, will minimize some cost function, for example, the sum of distances to the power of $3$, between each point-line pair. Here, the query set $Q$ is the set of all pairs of 2D rotation matrices and translation vectors $(R,t)$, the function $D((p,\ell),(R,t))$ is the Euclidean distance between a given point-line pair $(p,\ell)$ after applying the transformation $(R,t)$, i.e., $D((p,\ell),(R,t)) = \min_{q\in \ell}\norm{Rp-t-q}$, the function $\lip$ is $\lip(x) = x^3$, and the function $f$ simply sums those distances over all the input point-line pairs.

\begin{definition} [Optimization framework]\label{def:cost}
Let $X$ be a set called \emph{ground set}, $A = \br{a_1,\cdots,a_n} \subset X$ be a finite input set and let $Q$ be a set called a \emph{query set}.
Let $D:X\times Q\to [0,\infty)$ be a function.
Let $\lip:[0,\infty)\to [0,\infty)$ be an $r$-log-Lipschitz function and $f:[0,\infty)^n\to [0,\infty)$ be an $s$-log-Lipschitz function.
For every $q\in Q$ we define
\[
\cost (A,q)=f\left( \lip \left(D\left(a_1,q\right)\right),\cdots, \lip\left(D\left(a_n,q\right)\right)\right).
\]
\end{definition}

The following observation states that if a query $q'\in Q$ approximates, up to a constant factor, the function $D$ for every input element for some other query $q^* \in Q$, then $q'$ also approximates up to a constant, the cost of a more involved function $\cost$, relative to $q^*$, where $\cost$ is as defined in Definition~\ref{def:cost}.
\begin{observation} \label{obs:distToCost}
Let
\[
\cost (A,q)=f\left( \lip \left(D\left(a_1,q\right)\right),\cdots, \lip\left(D\left(a_n,q\right)\right)\right)
\]
be defined as in Definition~\ref{def:cost}.
Let $q^*,q' \in Q$ and let $c\geq 1$. If $D\left(a_i,q'\right)\leq c\cdot D\left(a_i,q^*\right)$ for every $i\in [n]$, then
\[
\cost\left(A,q'\right)\leq c^{rs} \cdot \cost\left(A,q^*\right).
\]
\end{observation}
\begin{proof}
We have that
\begin{align}
& \cost\left(A,q'\right) = f\left( \lip \left(D\left(a_1,q'\right)\right),\cdots, \lip\left(D\left(a_n,q'\right)\right)\right) \label{obsEq1}\\
& \leq f\left( \lip \left(c\cdot D\left(a_1,q^*\right)\right),\cdots, \lip\left(c \cdot D\left(a_n,q^*\right)\right)\right)\label{obsEq2}\\
& \leq f\left(c^r\cdot \lip \left(D\left(a_1,q^*\right)\right),\cdots, c^r\cdot \lip\left(D\left(a_n,q^*\right)\right)\right)\label{obsEq3}\\
& \leq c^{rs} \cdot f\left(\lip \left(D\left(a_1,q^*\right)\right),\cdots,\lip\left(D\left(a_n,q^*\right)\right)\right)\label{obsEq4}\\
&= c^{rs}\cdot \cost\left(A,q^*\right)\label{obsEq5},
\end{align}
where~\eqref{obsEq1} holds by the definition of $\cost$, \eqref{obsEq2} holds by the assumption of Observation~\ref{obs:distToCost}, \eqref{obsEq3} holds since $\lip$ is $r$-log-Lipschitz, \eqref{obsEq4} holds since $f$ is $s$-log-Lipschitz, and~\eqref{obsEq5} holds by the definition of $\cost$.
\end{proof}
Observation~\ref{obs:distToCost} implies that in order to approximate a non-trivial cost function as in Definition~\ref{def:cost}, it suffices to approximate a much simpler function, namely the function $D$, for every input element.

\section{Algorithms for Aligning Points to Lines} \label{sec:pointsToLines}

In this section, we introduce our notations, describe our algorithms, and give both an overview and intuition for each algorithm.
See Table~\ref{table:ourContrib} for example cost functions that our algorithms provably approximate.
See Sections~\ref{calcZIntuition},~\ref{mainAlgIntuition} and~\ref{noMatchingAlgIntuition} for an intuition of the algorithms presented in this section.

\paragraph{Notation for the rest of the paper.} \label{ch:notation}
Let $\REAL^{n\times d}$ be the set of all $n\times d$ real matrices. We denote by $\norm{p}=\norm{p}_2=\sqrt{p_1^2+\ldots+p_d^2}$ the length of a point $p=(p_1,\cdots,p_d) \in \REAL^d$, by $\dist(p,\ell) = \min_{x \in \ell} \norm{p-x}_2$ the Euclidean distance from $p$ to an a line $\ell$ in $\REAL^d$, by $\proj(p,X)$ its projection on a set $X$, i.e, $\proj(p,X) \in \arginf_{x \in X} \dist(p,x)$, and by $\mathrm{sp}\br{p}=\br{kp\mid k\in\REAL}$ we denote the linear span of $p$. For a matrix $V\in \REAL^{d\times m}$ whose columns are mutually orthonormal, i.e., $V^TV = I$, we denote by $V^{\bot} \in \REAL^{d\times (d-m)}$ an arbitrary matrix whose columns are mutually orthogonal unit vectors, and also orthogonal to every vector in $V$. Hence, $[V \mid V^\bot]$ is an orthogonal matrix. If $p\in \REAL^2$, we define $p^\bot=(p^\bot_1,\cdots,p^\bot_d)\in\REAL^d$ to be the point such hat $p^Tp^\bot=0$ and $p^\bot_1\geq 0$.
We denote $[n]=\br{1,\cdots,n}$ for every integer $n \geq 1$.

For the rest of this paper, every vector is a column vector, unless stated otherwise.
A matrix $R\in\REAL^{2\times 2}$ is called a \emph{rotation matrix} if it is orthogonal and its determinant is $1$, i.e., $R^TR=I$ and $\mathrm{det}(R)=1$. For $t\in\REAL^2$ that is called a \emph{translation} vector, the pair $(R,t)$ is called an \emph{alignment}. We define $\alignments$ to be the union of all possible alignments in $2$-dimensional space. The $x$-axis and $y$-axis are defined as the sets $\br{(z,0) \mid z \in \REAL}$ and $\br{(0,z) \mid z \in \REAL}$ respectively.

For a bijection (permutation) function $\pi:[n]\to [n]$ and a set $A = \br{(a_1,b_1),\cdots,(a_n,b_n)}$ of $n$ pairs of elements, $A_{\pi}$ is defined as $A_{\pi} = \br{(a_1,b_{\pi(1)}),\cdots,(a_n,b_{\pi(n)})}$.

We use the original definition of $O(\cdot)$ as a set of functions and write e.g. $t\in O(n)$ and not $t=O(n)$ to avoid dis-ambiguity as in $1+1=O(1)=3$; see discussion in~\cite{graham1989concrete}.

\paragraph{Algorithms.}
We now present algorithms that compute a constant factor approximation for the problem of aligning points to lines, when the matching is either known or unknown. Algorithm~\ref{Alg:slowApprox} handles the case where the matching between the points and lines is given, i.e. given an ordered set $P = \br{p_1,\cdots,p_n}$ of $n$ points, and a corresponding ordered set $L = \br{\ell_1,\cdots,\ell_n}$ of $n$ lines, both in $\REAL^2$, we wish to find an alignment that minimizes, for example, the sum of distances between each point in $P$ and its corresponding line in $L$.

Formally, let $A=\br{(p_i,\ell_i)}_{i=1}^n$ be a set of $n \geq 3$ point-line pairs, $z \geq 1$, and $D_z:A\times \alignments\to [0,\infty)$ such that $D_z\left((p,\ell),(R,t)\right) = \min_{q\in \ell} \norm{Rp-t-q}_z$ is the $\ell_z$ distance between $Rp-t$ and $\ell$ for every $(p,\ell)\in A$ and $(R,t) \in \alignments$. Let $\cost, s, r$ be as defined in Definition~\ref{def:cost} for $D=D_z$. Then Algorithm~\ref{Alg:slowApprox} outputs a set of alignments that is guaranteed to contain an alignment which approximates $\displaystyle\min_{(R,t)\in \alignments}\cost(A,(R,t))$ up to a constant factor; See Theorem~\ref{theorem:costPointsToLines}.

Algorithm~\ref{Alg:slowApproxNoMatching} handles the case when the matching is unknown, i.e. given unordered sets $P$ and $L$ consisting of $n$ points and $n$ lines respectively, we wish to find a matching function $\pi:[n]\to [n]$ and an alignment $(R,t)$ that minimize, for examples, the sum of distances between each point $p_i \in P$ and its corresponding line $\ell_{\pi(i)}\in L$.

Formally, let $\cost$ be as defined above but with $f = \norm{\cdot}_1$. Then Algorithm~\ref{Alg:slowApproxNoMatching} outputs a set of alignments that is guaranteed to contain an alignment which approximates $\displaystyle \min_{(R,t,\pi)}\cost(A_{\pi},(R,t))$ up to a constant factor, where the minimum is over every alignment $(R,t)$ and matching function $\pi$; See Theorem~\ref{theorem:pointsToLinesNoMatching}.

\subsection{Algorithm~\ref{Alg:calcz}: \secondalgname} \label{sec:secondAlg}
In this section we present a sub-routine called \secondalgname{} that is called from our main algorithm; see Algorithm~\ref{Alg:calcz} and our main algorithm in Algorithm~\ref{Alg:slowApprox}.

\subsubsection{Overview of Algorithm~\ref{Alg:calcz}}
The algorithm takes as input a triangle, and a line $\ell$ that intersects the origin. The triangle is defined by its three vertices $p,q,z\in\REAL^2$ and denoted by $\Delta(p,q,z)$. The line is defined by its direction (unit vector) $v$.

The usage of this algorithm in the main algorithm (Algorithm~\ref{Alg:slowApprox}) is to compute the union over every \emph{feasible configuration} $\Delta(p',q',z')$, which is a rotation and a translation of $\Delta(p,q,z)$ such that $p'$ is on the $x$-axis, and $q'$ is on the input line (simultaneously).

To this end, the output of the Algorithm~\ref{Alg:calcz} is a tuple of three $2\times 2$ matrices $P,Q$ and $Z$ such that the union of $(Px,Qx,Zx)$ over every unit vector $x$ is the desired set. That is, for every feasible configuration $\Delta(p',q',z')$ there is a unit vector $x\in\REAL^2$ such that $(p',q',z')=(Px,Qx,Zx)$, and vice versa. See illustration in Fig.~\ref{fig:ZxEllipse}.

\subsubsection{Intuition behind Algorithm~\ref{Alg:calcz}} \label{calcZIntuition}
For every matrix $Z\in\REAL^{2\times 2}$, the set $\{Zx\mid x\in\REAL^2,\norm{x}=1\}$ defines the boundary of an ellipse in $\REAL^2$. Hence, the shape formed by all possible locations of vertex $z$ in $\REAL^2$, assuming $p$ is on the $x$-axis and $q\in\ell$, is an ellipse. Furthermore, this ellipse is centered around the intersection point of the $x$-axis and $\ell$ (the origin, in this case). See Fig.~\ref{fig:ZxEllipse}.

\begin{algorithm} [ht]
\SetAlgoLined
\DontPrintSemicolon
{\begin{tabbing}
\textbf{Input:\quad} \= A unit vector $v =(v_x,v_y)^T$ such that\\\> $v_y\neq 0$,
and $p,q,z\in\REAL^2$ such that $p\neq q$.\\
\textbf{Output:} \> A tuple of matrices $P,Q,Z \in \REAL^{2 \times 2}$ \\\>that satisfy Lemma~\ref{Lemma:PQZx}.
\end{tabbing}}
\vspace{-0.1cm}
	Set $r_1\gets\norm{p-q},r_2\gets\norm{p-z},r_3\gets\norm{q-z}$ \label{Alg1line:rij}\\
	Set $d_1 \gets \displaystyle \frac{r_1^2+r_2^2-r_3^2}{2r_1}$;  $d_2 \gets \sqrt{|r_2^2-d_1^2|}$\\
	Set $R \gets \begin{pmatrix}
		0 & -1 \\
		1 & 0 \end{pmatrix}$ \label{Alg1line:R}
		\tcp{A rotation matrix that rotates the coordinates system by $\pi/2$ radians counter clockwise around the origin.}
	Set $P \gets r_1\left( \begin{array}{ccc}
		\frac{v_x}{v_y} & 1 \\
		0 & 0 \end{array} \right)$ \label{Alg1line:PDef}\\
	Set $Q \gets P^T$ \label{Alg1line:QDef}\\
	Set $b \gets \begin{cases} 1, & (z-p)^T(q-p)^\bot > 0\\ 0, & \text{otherwise} \end{cases}$\label{lline5}\\
	\tcp{$b=1$ if $z$ is in the halfplane to the right of the vector $q-p$, and $b=0$ otherwise.}
	Set $Z \gets \displaystyle P + \frac{d_1}{r_1}(Q-P) + b\cdot\frac{d_2}{r_1}R(Q-P)$ \label{Alg1line:ZDef}\\
\Return $(P,Q,Z)$
\caption{$\secondalgname(v,p,q,z)$}
\label{Alg:calcz}
\end{algorithm}

\begin{figure} [ht]
\centering
\includegraphics[width=0.4\textwidth]{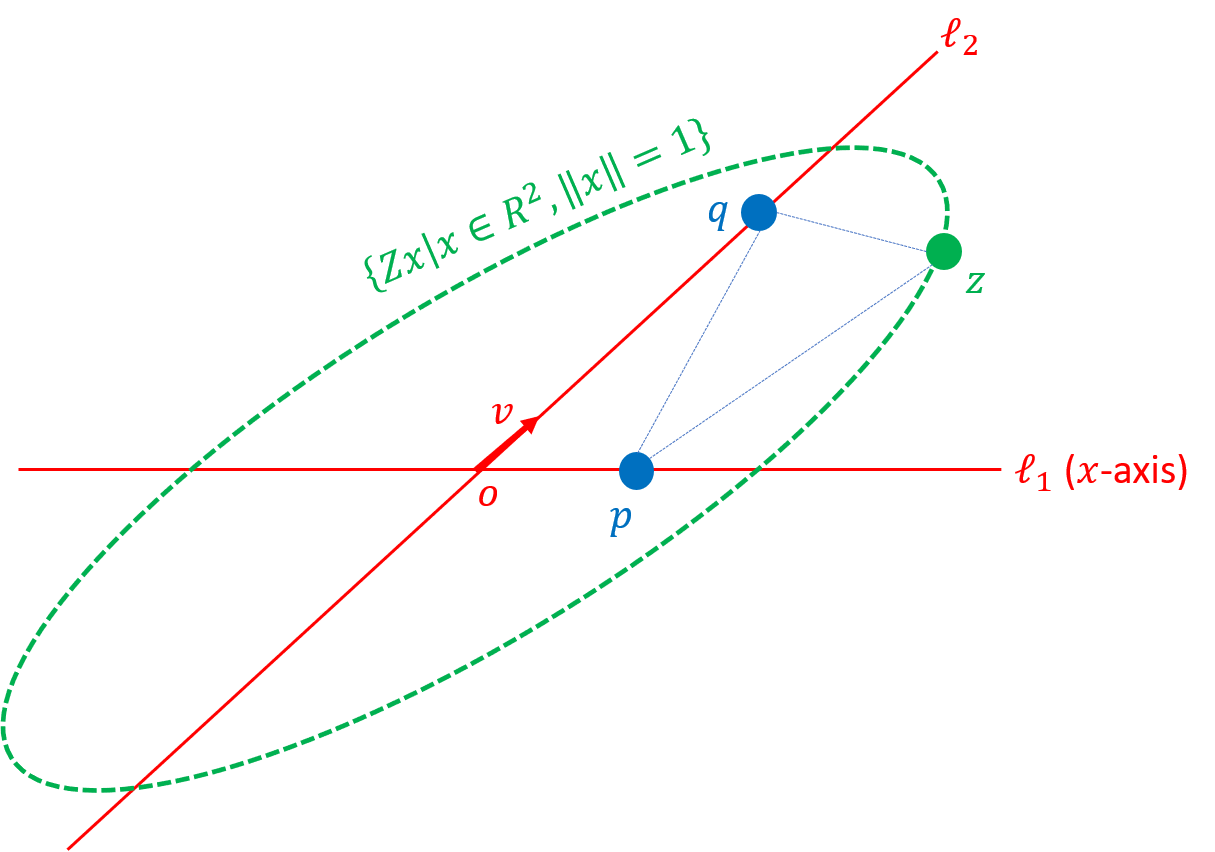}
\caption{\textbf{Illustration for Algorithm~\ref{Alg:calcz}. }A triangle $\Delta(p,q,z)$ whose vertex $p$ intersects the $x$-axis, and its vertex $q$ is on the line $\mathrm{sp}\br{v}$. There are infinitely many such triangles, and the union over every possible solution $z$ are the points on the green ellipse. If $(P,Q,Z)$ is the output of Algorithm~\ref{Alg:calcz}, then $(p,q,z)=(Px,Qx,Zx)$ for some unit vector $x$.}
\label{fig:ZxEllipse}
\end{figure}

\subsection{Algorithm ~\ref{Alg:slowApprox}: \algname} \label{sec:mainAlg}
In this section we present our main algorithm, called \algname; See Algorithm~\ref{Alg:slowApprox}.

\begin{algorithm}
\SetAlgoLined
\DontPrintSemicolon
{\begin{tabbing}
\textbf{Input:\quad} \= A set $A=\br{(p_1,\ell_1),\cdots,(p_n,\ell_n)}$ of $n$ pairs,\\ \> where for every $i\in[n]$, we have that $p_i$ is a point\\ \>and $\ell_i$ is a line, both on the plane.\\
\textbf{Output:} \>A set $C \subset \alignments$ of alignments that\\ \>satisfies Lemma~\ref{theorem:2DPointsToLines}.\\
\end{tabbing}}
\vspace{-0.5cm}
Set $C \gets \emptyset$.\\
Set $v_i\in\REAL^2$ and $b_i\geq 0$ such that $\ell_i = \br{q\in \REAL^2 \mid v_i^T q =b_i}$ for every $i\in[n]$. \label{Alg2line:idL}\\
\For {\textbf{every} $j,k,l\in[n]$ \textbf{such that} $j\neq k$}{ \label{Alg2line:ES}
Set $C_1,C_2 \gets \emptyset$. \label{Alg2line:setC2Empty}\\
\uIf{$|v_j^Tv_k| \neq 1$}{ \tcc{$\ell_j$ and $\ell_k$ are not parallel.}
	Set $v_j^\bot \gets$ a unit vector in $\REAL^2$ that is orthogonal to $v_j$. \label{Alg2line:vjOrtho}\\
Set $R_{v_j} \gets \left( \begin{array}{ccc}
		-v_j^T- \\
		-{v_j^\bot}^T- \end{array} \right)$. \label{Alg2line:Rvj}\\\tcp{$R_{v_j}$ aligns $v_j$ with the $x$-axis.}
Set $(P',Q',Z') \gets \secondalgname(R_{v_j}v_k,p_j,p_k,p_l)$. \label{Alg2line:calcZ}\\ \tcp{See Algorithm~\ref{Alg:calcz}}
Set $P\gets R_{v_j}^T P', Q\gets R_{v_j}^T Q', Z\gets R_{v_j}^T Z'$. \label{Alg2line:RjZ}\\
Set $s \gets \ell_j \cap \ell_k$. \label{Alg2line:interLjLk} \\ \tcp{$\ell_j \cap \ell_k$ contains one point since $\ell_j$ and $\ell_k$ are not parallel.}
Set $c_l \gets \dist(s,\ell_l)$. \label{Alg2line:changeBl}\\
Set $X \gets \displaystyle \argmin_{x \in \REAL^2:\norm{x}=1} |v_l^TZ x -c_l|$. \label{Alg2line:compX} \tcp{The set of unit vectors that minimize the distance between $p_l$ and $\ell_l$ while maintaining $p_j\in\ell_j$ and $p_k\in\ell_k$.}
Set $C_1 \gets \{(R,t)\in \alignments \mid Rp_j-t=Px \text{ and } Rp_k-t=Qx \text{ and } Rp_l-t=Zx \text{ for every } x\in X\}$. \label{Alg2line:compC1} \tcc{The set of alignments that align the points $(p_j,p_k,p_l)$ with points $(Px,Qx,Zx)$ for every $x\in X$.}
}
\uElse{
Set $C_2 \gets \br{(R,t)}$ such that $(R,t) \in \alignments$, $Rp_j-t \in \ell_j$ and $(R,t)\in\argmin {\dist(Rp_k-t,\ell_k)}$. \label{Alg2line:compC2}
}

Set $C \gets C \cup C_1 \cup C_2$. \label{Alg2line:endES}
}
\Return $C$
\caption{\algname$(A)$}
\label{Alg:slowApprox}
\end{algorithm}

\subsubsection{Overview of Algorithm~\ref{Alg:slowApprox}}
The input for Algorithm~\ref{Alg:slowApprox} is a set of $n$ pairs, each consists of a point and a line on the plane. The algorithm runs exhaustive search on all the possible tuples and outputs candidate set $C$ of $O(n^3)$ alignments. Alignment consists of a rotation matrix $R$ and a translation vector $t$. Theorem~\ref{theorem:2DPointsToLines} proves that one of these alignments is the desired approximation.

Line~\ref{Alg2line:idL} identifies each line $\ell_i$ by its direction (unit vector) $v_i$ and distance $b_i>0$ from the origin. Lines~\ref{Alg2line:ES}--\ref{Alg2line:endES} iterates over every triple $(j,k,l)$ of input pairs such that $j\neq k$, and turns it into a constant number of alignments $C_1$. In Lines~\ref{Alg2line:vjOrtho}--\ref{Alg2line:compC1} we handle the case where the lines $\ell_j$ and $\ell_k$ are not parallel. In Line~\ref{Alg2line:compC2} we handle the case where $\ell_j$ and $\ell_k$ are parallel.

\textbf{The case where $\ell_j$ and $\ell_k$ are not parallel.} Lines~\ref{Alg2line:vjOrtho}--\ref{Alg2line:Rvj} compute a rotation matrix $R_{v_j}$ that rotates $v_j$ to the $x$-axis. Line~\ref{Alg2line:calcZ} calls the sub-procedure Algorithm~\ref{Alg:calcz} for computing three matrices $P',Q'$ and $Z'$. In Line~\ref{Alg2line:RjZ} we revert the effect of the rotation matrix $R_{v_j}$. Lines~\ref{Alg2line:interLjLk}--\ref{Alg2line:changeBl} compute the distance between $\ell_l$ and the intersection between $\ell_j$ and $\ell_k$ since we assumed this intersection point is the origin in Algorithm~\ref{Alg:calcz}. The matrix $Z$ and the line $\ell_l$ are used to compute a set $X$ of $O(1)$ unit vectors in Line~\ref{Alg2line:compX}. Every $x\in X$ defines a possible positioning for the triplet. In Line~\ref{Alg2line:compC1} we define an alignment $(R,t)$ for each $x\in X$.
The union of the alignments in $C_1$ is then added to the output set $C$ in Line~\ref{Alg2line:endES}.

\textbf{The case where $\ell_j$ and $\ell_k$ are parallel.} In this case, we place $p_j \in \ell_j$, and place $p_k$ as close as possible to $\ell_k$. If there are more than one alignment that satisfies these conditions, then we pick an arbitrary one. This is done in Line~\ref{Alg2line:compC2}.

\subsubsection{Intuition behind Algorithm~\ref{Alg:slowApprox}} \label{mainAlgIntuition}
The idea behind the algorithm consists of three steps. At each step we reduce the set of feasible alignments by adding another constraint. Each constraint typically increases our approximation factor by another constant.

Consider any alignment of the input points to lines, and suppose that $(p,\ell_1)$ is the closest pair after applying this alignment. By the triangle inequality, translating the set $P$ of points so that $p$ intersects $\ell_1$ will increase the distance between every other pair by a factor of at most $2$. Hence, minimizing~\eqref{min} under the constraint that $p\in\ell_1$ would yield a $2$-approximation to the original (non-constrained) problem.

Similarly, we can then translate $P$ in the direction of $\ell_1$ (while maintaining $p \in \ell_1$) until the closest pair, say $(q,\ell_2)$, intersects. The result is a $4$-approximation to the initial alignment by considering all the possible alignments of $P$ such that $p\in\ell_1$ and $q\in\ell_2$.
There are still infinite such alignments which satisfy the last constraint.

Hence, we add a third step. Let $(z,\ell_3)$ be the pair that requires the minimal rotation of the vector $q-p$ in order to minimize $\dist(z,\ell_3)$ under the constraints that $p \in \ell_1$ and $q \in \ell_2$. We now rotate the vector $q-p$ and translate the system to maintain the previous constraints until $\dist(z,\ell_3)$ is minimizes.

\textbf{The result } is a $16$-approximation to~\eqref{min} by considering all the possible alignments of $P$ such that: $z$ is closest to $\ell_3$ among all alignments that satisfy $p\in\ell_1$ and $q\in\ell_2$.
Unlike the previous steps, there are only a finite number of such alignments, namely $|C| = \binom{n}{3}$.

\subsection{Algorithm~\ref{Alg:slowApproxNoMatching}: Alignment and Matching}

\begin{algorithm}[ht]
\SetAlgoLined
\DontPrintSemicolon
{\begin{tabbing}
\textbf{Input:\quad} \= A set $A = \br{(p_1,\ell_1),\cdots,(p_n,\ell_n)}$ and a cost\\ \>function as in Theorem~\ref{theorem:pointsToLinesNoMatching}.\\
\textbf{Output:} \>An element $(\tilde{R},\tilde{t},\tilde{\pi})$ that satisfies Theorem~\ref{theorem:pointsToLinesNoMatching}.
\end{tabbing}}

Set $X \gets \emptyset$.\\
\For {\textbf{every} $i_1,i_2,i_3,j_1,j_2,j_3 \in[n]$ \label{Alg4:ES}}{
	Set $X' \gets \algname\left(\br{(p_{i_1},\ell_{j_1}),(p_{i_2},\ell_{j_2}),(p_{i_3},\ell_{j_3})} \right)$ \label{Alg4:callToAlg1}.\\ \tcp{See Algorithm~\ref{Alg:slowApprox}}
	Set $X \gets X \cup X'$. \label{Alg4:addToX}\\
}
Set $S \gets \br{\left(R,t,\hat{\pi}(A,(R,t),\cost)\right) \mid (R,t)\in X}$. \label{Alg4:mappingFunc}\\
 \tcc{see Definition~\ref{def:optimalMatching}}
Set $(\tilde{R},\tilde{t},\tilde{\pi}) \displaystyle \in \argmin_{(R',t',\pi')\in S} \cost\left(A_{\pi'},(R',t')\right)$. \label{Alg4:minS}

\Return $(\tilde{R},\tilde{t},\tilde{\pi})$
\caption{$\matchingalgname(A,\cost)$}
\label{Alg:slowApproxNoMatching}
\end{algorithm}


\subsubsection{Overview of Algorithm~\ref{Alg:slowApproxNoMatching}}
Algorithm~\ref{Alg:slowApproxNoMatching} takes as input a set of $n$ paired points and lines in $\REAL^2$, and a cost function as defined in Theorem~\ref{theorem:pointsToLinesNoMatching}. The algorithm computes an alignment $(\hat{R},\hat{t}) \in \alignments$ and a matching function $\hat{\pi}$ (which rearranges the given pairing of the pairs) that approximate the minimal value of the given cost function; See Theorem~\ref{theorem:pointsToLinesNoMatching}.

In Line~\ref{Alg4:ES} we iterate over every $i_1,i_2,i_3,j_1,j_2,j_3 \in [n]$. In Lines~\ref{Alg4:callToAlg1}-~\ref{Alg4:addToX} we match $p_{i_1}$ to $\ell_{j_1},p_{i_2}$ to $\ell_{j_2}$ and $p_{i_3}$ to $\ell_{j_3}$, compute their corresponding set of alignments $X'$ by a call to Algorithm~\ref{Alg:slowApprox}, and then add $X'$ to the set $X$. Finally, in Lines~\ref{Alg4:mappingFunc}-~\ref{Alg4:minS} we compute the optimal matching for every alignment in $X$, and pick the alignment and corresponding matching that minimize the given cost function.

\subsubsection{Intuition behind Algorithm~\ref{Alg:slowApproxNoMatching}} \label{noMatchingAlgIntuition}
In Algorithm~\ref{Alg:slowApproxNoMatching}, we iterate over every triplet of points and triplet of lines from the input set $A$. Each such tuple of 3 points and 3 lines define a set of $O(1)$ alignments using Algorithm~\ref{Alg:slowApprox}. For each such alignment, we compute the optimal matching function for the given cost function, using naive optimal matching algorithms. We them return the alignment and matching that yield the smallest cost.

\section{Statements of Main Results}

The following lemma is one of the main technical results of this paper, and lies in the heart of the proof of Lemma~\ref{theorem:2DPointsToLines}. It proves that for every two paired sets $a_1,\cdots,a_n\subseteq \REAL^{2}$ and $b_1,\cdots,b_n \geq 0$ and unit vector $x\in \REAL^2$, there exists $x' \in \argmin_{\norm{y}=1} |a_k^Ty-b_k|$ for some $k\in [n]$ that approximates $|a_i^Tx-b_i|$ for every $i \in [n]$.
\begin{lemma} \label{lem:axb}
Let $a_1,\cdots,a_n\subseteq \REAL^{2}$ and $b_1,\cdots,b_n \geq 0$.
Then there is a set $C$ of $|C|\in O(n)$ unit vectors that can be computed in $O(n)$ time such that (i) and (ii) hold as follows:
\renewcommand{\labelenumi}{(\roman{enumi})}
\begin{enumerate}
\item For every unit vector $x\in\REAL^2$ there is a vector $x'\in C$ such that for every $i\in [n]$,
\begin{equation}
|a_i^Tx'-b_i|\leq 4 \cdot |a_i^Tx-b_i|.
\end{equation}
\item There is $k\in [n]$ such that $x' \in \argmin_{\norm{y}=1} |a_k^Ty-b_k|$
\end{enumerate}
\end{lemma}
\begin{proof}
See proof of Lemma~\ref{lem:axb_proof} in the appendix.
\end{proof}

Lemma~\ref{lem:axb} suggests a result of independent interest. Consider the \emph{constrained linear regression problem} of computing a unit vector $x^*\in\REAL^2$ that minimizes
\begin{align*}
x^* & \in \displaystyle\argmin_{x\in \REAL^2:\norm{x}=1} \norm{Ax-b}_1\\ &= \displaystyle\argmin_{x\in \REAL^2:\norm{x}=1} \sum_{i=1}^2 |a_i^Tx - b_i|,
\end{align*}
for a given matrix $A = (a_1 \mid ... \mid a_n)^T \in \REAL^{n\times 2}$ and a given vector $b = (b_1,\cdots,b_n) \in \REAL^2$.
Lemma~\ref{lem:axb} implies that in linear time, we can compute a candidate set of unit vectors $C$, which contains a unit vector $x'$ that approximates each of the $n$ ``distances'' $|a_i^Tx^*-b_i|$ up to a multiplicative factor of $4$. Hence, $x'$ also approximates the optimal (unknown) value of the constrained $\ell_1$ linear regression $\norm{Ax^*-b}_1$, up to a multiplicative factor of $4$. It is easy to see that a similar argument holds for the constrained $\ell_z$ regression $\norm{Ax^*-b}_z$, for $z>0$, up to a constant factor that depends only on $z$.

The following lemma proves the correctness of Algorithm~\ref{Alg:calcz}. It proves that the matrices computed in Algorithm~\ref{Alg:calcz} satisfy a set of properties, which will be useful in the proof of Lemma~\ref{theorem:2DPointsToLines}.
\begin{lemma} \label{Lemma:PQZx}
Let $v$ be a unit vector and $\ell=\mathrm{sp}\br{v}$ be the line in this direction.
Let $p,q,z\in\REAL^2$ be the vertices of a triangle such that $\norm{p-q}>0$. Let $P,Q,Z\in\REAL^{2\times 2}$ be the output of a call to \secondalgname$(v,p,q,z)$; see Algorithm~\ref{Alg:calcz}. Then the following hold:
\begin{enumerate}
\renewcommand{\labelenumi}{(\roman{enumi})}
\item $p \in x$-axis and $q \in \ell$ iff there is a unit vector $x\in\REAL^2$ such that $p=Px$ and $q=Qx$.
\item For every unit vector $x\in\REAL^2$, we have that $z=Zx$ if $p=Px$ and $q=Qx$.
\end{enumerate}
\end{lemma}
\begin{proof}
See proof of Lemma~\ref{Lemma:PQZx_proof} in the appendix.
\end{proof}

\subsection{Aligning Points-To-Lines}
Table~\ref{table:ourContrib} summarizes the main results and cost functions in the context of aligning points-to-lines, which our algorithms support.

What follows is the main technical lemma for the correctness of Algorithm~\ref{Alg:slowApprox}. It proves that for every (possibly optimal) alignment $(R^*,t^*)$, the output set $C$ of Algorithm~\ref{Alg:slowApprox} must contain an alignment $(R',t')$ that approximates each of the distances $\dist(R^*p_i-t^*,\ell_i)$ for every $i\in [n]$, up to some multiplicative constant factor. Furthermore, the set $C$ is computed in time polynomial in $n$.
\begin{lemma} \label{theorem:2DPointsToLines}
Let $A=\br{(p_1,\ell_1),\cdots,(p_n,\ell_n)}$ be set of $n \geq 3$ pairs, where for every $i \in [n]$, we have that $p_i$ is a point and $\ell_i$ is a line, both on the plane. Let $C \subseteq \alignments$ be an output of a call to \algname$(A)$; see Algorithm~\ref{Alg:slowApprox}. Then for every alignment $(R^*,t^*)\in \alignments$ there exists an alignment $(R,t)\in C$ such that for every $i\in[n]$,
\begin{equation}
\dist(Rp_i-t,\ell_i) \leq 16\cdot \dist(R^*p_i-t^*,\ell_i).
\end{equation}
Moreover, $|C| \in O(n^3)$ and can be computed in $O(n^3)$ time.
\end{lemma}
\begin{proof}
See proof of Lemma~\ref{theorem:2DPointsToLines_proof} in the appendix.
\end{proof}

\subsubsection{Generalization}
Observation~\ref{obs:distToCost} states that in order to approximate a complicated and non-convex cost function $\cost$ as in Definition~\ref{def:cost}, all needs to be done is to compute a query (alignment) that approximates the simpler (distance) function $D = \dist$ that lies in the heart of the function $\cost$. Lemma~\ref{theorem:2DPointsToLines} implies that in polynomial time we can indeed compute a query (alignment) that approximates the distance function $\dist$. Combining Observation~\ref{obs:distToCost} with Lemma~\ref{theorem:2DPointsToLines} yields that in polynomial time we can compute a candidates set of queries (alignments) which contain a query that approximates, up to a constant factor, a complex cost function, as follows.
\begin{theorem} \label{theorem:costPointsToLines}
Let $A=\br{(p_1,\ell_1),\cdots,(p_n,\ell_n)}$ be set of $n \geq 3$ pairs, where for every $i \in [n]$, we have that $p_i$ is a point and $\ell_i$ is a line, both on the plane. Let $z \geq 1$ and let $D_z:A\times \alignments\to [0,\infty)$ such that $D_z\left((p,\ell),(R,t)\right) = \min_{q\in \ell} \norm{Rp-t-q}_z$ is the $\ell_z$ distance between $Rp-t$ and $\ell$. Let $\cost, s, r$ be as defined in Definition~\ref{def:cost} for $D=D_z$. Let $w = 1$ if $z=2$ and $w=\sqrt{2}$ otherwise. Let $C$ be the output of a call to \algname$(A)$; see Algorithm~\ref{Alg:slowApprox}. Then there exists $(R',t')\in C$ such that
\[
\cost(A,(R',t')) \leq (w\cdot 16)^{rs} \cdot \min_{(R,t)\in\alignments}\cost(A,(R,t)).
\]
Furthermore, $C$ and $(R',t')$ can be computed in $n^{O(1)}$ time.
\end{theorem}
\begin{proof}
See proof of Theorem~\ref{theorem:costPointsToLines_proof} in the appendix.
\end{proof}
Since $C$ is guaranteed to contain an alignment $(R',t')$ that approximates the optimal value of $\cost$, up to some multiplicative constant factor, simple exhaustive search in $O(n^4)$ time on $C$ can find an alignment $(\hat{R},\hat{t})\in C$ with cost smaller or equal to the cost of $(R',t')$.

\subsubsection{Unknown matching}
Given an input set of $n$ unmatched pairs $A = (p_1,\ell_1),\cdots,(p_n,\ell_n)$, an alignment (query) $(R,t)$ and some cost function, we seek a matching function $\pi$, such that after rearranging the input pairs using $\pi$, the rearranged (rematched) set $A_{\pi} = (p_1,\ell_{\pi(1)}),\cdots,(p_n,\ell_{\pi(n)})$ minimizes the given cost function for the given query $(R,t)$.
A formal definition is given as follows.
\begin{definition} [\textbf{Optimal matching}] \label{def:optimalMatching}
Let $n\geq 1$ be an integer and $\perms(n)$ denote the union over every permutation (bijection functions) $\pi:[n]\to[n]$.
Let $A = \br{a_1,\cdots,a_n}$ be an input set, where $a_i = (p_i,\ell_i)$ is a pair of elements for every $i\in [n]$, and let $Q$ be a set of queries.
Consider a function $\cost$ as defined in Definition~\ref{def:cost} for $f(v) = \norm{v}_1$.
Let $q \in Q$.
A permutation $\hat{\pi}$ is called an \emph{optimal matching} for $(A,q,\cost)$ if it satisfies that
\[
\hat{\pi}(A,q,\cost) \in \argmin_{\pi \in \perms(n)} \cost(A_{\pi},q).
\]
\end{definition}

Recall that for $\pi:[n]\to [n]$ and a set $A = \br{(p_1,\ell_1),\cdots,(p_n,\ell_n)}$ of $n$ pairs of elements, $A_{\pi}$ is defined as $A_{\pi} = \br{(p_1,\ell_{\pi(1)}),\cdots,(p_n,\ell_{\pi(n)})}$.
\begin{theorem} \label{theorem:pointsToLinesNoMatching}
Let $A = \br{(p_1,\ell_1),\cdots,(p_n,\ell_n)}$ be a set of $n\geq 3$ pairs, where for every $i\in [n]$ we have that $p_i$ is a point and $\ell_i$ is a line, both on the plane. Let $z \geq 1$ and define
\[
D_z\left((p,\ell),(R,t)\right) = \displaystyle \min_{q\in \ell} \norm{Rp-t-q}_z
\]
for every point $p$ and line $\ell$ on the plane and alignment $(R,t)$.
Consider $\cost$ and $r$ to be as defined in Definition~\ref{def:cost} for $D = D_z$ and $f(v) = \norm{v}_1$. Let $w = 1$ if $z=2$ and $w=\sqrt{2}$ otherwise. Let $(\tilde{R},\tilde{t},\tilde{\pi})$ be the output alignment $(\tilde{R},\tilde{t})$ and permutation $\tilde{\pi}$ of a call to \matchingalgname$(A,\cost)$; see Algorithm~\ref{Alg:slowApproxNoMatching}. Then
\begin{equation}
\cost\left(A_{\tilde{\pi}},(\tilde{R},\tilde{t})\right) \leq (w\cdot 16)^r \cdot \min_{(R,t,\pi)}\cost\left(A_{\pi},(R,t)\right),
\end{equation}
where the minimum is over every alignment $(R,t)$ and permutation $\pi:[n]\to[n]$.
Moreover, $(\tilde{R},\tilde{t},\tilde{\pi})$ can be computed in $n^{O(1)}$ time.
\end{theorem}
\begin{proof}
See proof of Theorem~\ref{theorem:pointsToLinesNoMatching_proof} in the appendix.
\end{proof}

\subsection{Coresets for Big Data}
We now present a data reduction technique (coreset) for the points-to-lines alignment problem, based on~\cite{dasgupta2009sampling}. This reduction will enable a near-linear implementation of our algorithm by running it on the small coreset while obtaining an almost similar approximation.

\begin{theorem} [coreset for points-to-lines alignment]\label{theorem:coreset}
Let $d \geq 2$ be an integer. Let $A=\br{(p_1,\ell_1),\cdots,(p_n,\ell_n)}$ be set of $n$ pairs, where for every $i \in [n]$, $p_i$ is a point and $\ell_i$ is a line, both in $\REAL^d$, and let $w = (w_1,\cdots,w_n) \in [0,\infty)^n$.
Let $\varepsilon, \delta \in (0,1)$.
Then a weights vector $u = (u_1,\cdots,u_n)\in[0,\infty)^n$ can be computed in $nd^{O(1)} \log{n}$ such that
\renewcommand{\labelenumi}{(\roman{enumi})}
\begin{enumerate}
\item With probability at least $1-\delta$, for every $(R,t) \in \alignments$ it holds that
\[
\begin{split}
& (1-\varepsilon) \cdot \sum_{i=1}^n w_i\cdot\dist(Rp_i-t,\ell_i)\\ &\leq \sum_{i=1}^nu_i\cdot\dist(Rp_i-t,\ell_i)\\ &\leq (1+\varepsilon) \cdot \sum_{i=1}^n w_i\cdot\dist(Rp_i-t,\ell_i).
\end{split}
\]
\item The weights vector $u$ has $\frac{d^{O(1)}}{\varepsilon^2}\log{\frac{1}{\delta}}$ non-zero entries.
\end{enumerate}
\end{theorem}
\begin{proof}
See proof of Theorem~\ref{theorem:coreset_proof} in the appendix.
\end{proof}

Instead of running the approximation algorithms proposed in Section~\ref{sec:pointsToLines} on the entire input data, Theorem~\ref{theorem:coreset} implies that we can first compress the input point-line pairs in $O(n\log{n})$ time, and then apply the approximation algorithms only on the compressed data, which will take time independent of $n$.
Hence, the total running time will be dominated by the coreset computation time. However, the approximation factor in this case will increase by an additional multiplicative factor of $(1+\epsilon)$.

\begin{corollary}[streaming, distributed, dynamic data]\label{cor:streamingCoreset}
Let $A=\br{(p_1,\ell_1),(p_2,\ell_2,),\cdots}$ be a (possibly infinite) stream of pairs, where for every $i \in [n]$, $p_i$ is a point and $\ell_i$ is a line, both in the plane. Let $\varepsilon,\delta \in (0,1)$.
Then, for every integer $n>1$ we can compute with probability at least $1-\delta$ an alignment $(R^*,t^*)$ that satisfies
\[
\sum_{i=1}^{n} \dist(R^*p_i-t^*,\ell_i) \in O(1) \cdot \min_{(R,t)} \sum_{i=1}^{n} \dist(Rp_i-t,\ell_i),
\]
for the $n$ points seen so far in the stream, using $(\log(n/\delta)/\eps)^{O(1)}$ memory and update time per a new pair.
Using $M$ machines the update time can be reduced by a factor of $M$.
\end{corollary}
\begin{proof}
See proof of Corollary~\ref{cor:streamingCoreset_proof} in the appendix.
\end{proof}

\section{Conclusion and Open Problems}
We described a general framework for approximating functions under different constraints, and used it for obtaining generic algorithms for minimizing a finite set of piecewise log-Lipschitz functions. We apply this framework to the points-to-lines alignment problem. Coresets for these problems enabled us to turn our polynomial time algorithms in some cases into near linear time, and support streaming and distributed versions for Big Data.

Open problems include generalization of our results to higher dimensions, as some of our suggested coresets. We generalize our algorithms for the case when no matching permutation $\pi$ is given between the points and lines. However, the results are less general than our results for the known matching case, and we do not have coresets for these cases, which we leave for future research.


\bibliographystyle{IEEEtran}
\bibliography{references}

%

\newpage
\clearpage
\appendix
%

\section{Proofs of Main Results}
In this section we first present and prove helpful lemmas, and then present the main results and their proofs

The following lemma states that if a function $f$ is concave in an interval $X$ that contains $0$, then $f$ is $1$-log-Lipschitz.
\begin{lemma} \label{lem:fcx}
Let $X \subset \REAL$ be an interval that contains $0$, and let $f:X\to [0,\infty)$ such that the second derivative of $f$ is defined and satisfies $f''(x) \leq 0$ for every $x\in X$. Then $f$ is $1$-log-Lipschitz in $X$, i.e., for every $c\geq 1$ and $x\in \displaystyle X \cap \frac{X}{c}$, it holds that
\[
f(cx) \leq c f(x).
\]
\end{lemma}
\begin{proof}
Since $f''(x) \leq 0$ for every $x\in X$, we have that $f$ is concave in the interval $X$. Therefore, for every $a,b \in X$, it holds that
\begin{equation} \label{eq:concave}
f(b) \leq f(a) + f'(a)(b-a).
\end{equation}
Let $c\geq 1$ and put $x\in X \cap \frac{X}{c}$.
By substituting $a=x$ and $b=0$ in~\eqref{eq:concave}, we have
\[
f(0) \leq f(x)+f'(x)(0-x) = f(x)-xf'(x).
\]
Rearranging terms yields
\begin{equation} \label{eq:xdf}
xf'(x) \leq f(x)-f(0) \leq f(x),
\end{equation}
where the second derivation is since $f(0) \geq 0$ by the definition of $f$.
Hence, it holds that
\[
f(cx) \leq f(x) + f'(x)x(c-1) \leq f(x) + f(x)(c-1) = cf(x),
\]
where the first inequality holds by substituting $a=x$ and $b=cx$ in~\eqref{eq:concave}, and the second inequality holds by~\eqref{eq:xdf}.
\end{proof}

\begin{lemma} [Lemma~\ref{lem:axb}] \label{lem:axb_proof}
Let $a_1,\cdots,a_n\subseteq \REAL^{2}$ and $b_1,\cdots,b_n \geq 0$.
Then there is a set $C$ of $|C|\in O(n)$ unit vectors in $\REAL^2$ that can be computed in $O(n)$ time such that (i) and (ii) hold as follows:
\renewcommand{\labelenumi}{(\roman{enumi})}
\begin{enumerate}
\item For every unit vector $x\in\REAL^2$ there is a vector $x'\in C$ such that for every $i\in [n]$,
\begin{equation}\label{eqb}
|a_i^Tx'-b_i|\leq 4 \cdot |a_i^Tx-b_i|.
\end{equation}
\item There is $k\in [n]$ such that $x' \in \argmin_{\norm{y}=1} |a_k^Ty-b_k|$
\end{enumerate}
\end{lemma}
\begin{proof}
\textbf{Proof of (i): }
By Theorem~\ref{Lem:piecewiseApprox} it suffices to prove that $g_i(x) = |a_i^Tx-b_i|$ is piecewise $2$-log-Lipschitz over every unit vector $x$, and then define $C$ to include all the minima of $|a_i^Tx-b_i|$, over every $i\in[n]$. The proof that $|a_i^Tx-b_i|$ is piecewise $r$-log-Lipschitz for $r=2$ is based on the case analysis of whether $b_i \geq \norm{a_i}$ or $b_i < \norm{a_i}$ in Claims~\ref{claim:case1} and~\ref{claim:case2} respectively. The proofs of these cases use Lemma~\ref{lem:fcx}.

Indeed, let $i\in[n]$, $\alpha_i\in[0,2\pi)$ such that $a_i/\norm{a_i}=(\cos\alpha_i,\sin\alpha_i)$, and let $I = [\alpha_i,\alpha_i+\pi)$ be an interval. 
Here we assume $\norm{a_i}\neq 0$, otherwise~\eqref{eqb} trivially holds for $i$. For every $t\in\REAL$, let
\[
g_i(t)=\begin{cases} \norm{a_i}\cdot |\sin(t-\alpha_i)-\frac{b_i}{\norm{a_i}}|, &  t\in I\\
0, & \text{otherwise} \end{cases},
\]
and let $x(t) = (\sin{t},-\cos{t})$.
Hence, for every $t\in I$,
\begin{equation} \label{Eq:gt_axb}
\begin{split}
g_i(t)&=\norm{a_i}\cdot \left|\sin(t-\alpha_i)-\frac{b_i}{\norm{a_i}}\right| \\ & =\norm{a_i}\cdot \left|\sin(t)\cos(\alpha_i)-\cos(t)\sin(\alpha_i)-\frac{b_i}{\norm{a_i}}\right|\\
&= \norm{a_i}\cdot \left|\frac{a_i^T}{\norm{a_i}} x(t) -\frac{b_i}{\norm{a_i}} \right| =\left|a_i^Tx(t)-b_i \right|.
\end{split}
\end{equation}

\begin{claim} \label{claim:case1}
If $\frac{b_i}{\norm{a_i}} \geq 1$ then $g_i:I\to \REAL$ is piecewise $2$-log-Lipschitz.
\end{claim}
\begin{proof}
Suppose that indeed $\frac{b_i}{\norm{a_i}} \geq 1$.
In this case, $\frac{b_i}{\norm{a_i}} \geq 1 \geq \sin(t-\alpha_i)$ for every $t\in I$, so the absolute value in $g_i$ can be removed, i.e., for every $t\in\REAL$,
\[
g_i(t)=\begin{cases} \norm{a_i}\cdot \left(\frac{b_i}{\norm{a_i}}-\sin(t-\alpha_i)\right), &  t\in I\\
0, & \text{otherwise} \end{cases}.
\]
Let $I' = [0,\pi/2]$ and let $h:I' \to [0,\infty)$ such that
\[
h(x) = \norm{a_i}\cdot \left(\frac{b_i}{\norm{a_i}}-\sin\left(\frac{\pi}{2}+x\right)\right).
\]
Since $h(x) = h(-x) = g_i(\pi/2+\alpha_i + x)$ for every $x\in I'$, we have
\begin{equation} \label{eq:gicase1}
g_i(t) = \begin{cases}
h(\abs{\pi/2 + \alpha_i -t}), & t\in I\\
0, & \text{otherwise}.
\end{cases}
\end{equation}

We now prove that $h(x)$ is $2$-log-Lipschitz for every $x\in I'$. For every $x\in I'$ we have
\begin{equation} \label{eq:hx}
\begin{split}
h(x) & = \norm{a_i}\cdot \left(\frac{b_i}{\norm{a_i}}-\sin\left(\frac{\pi}{2}+x\right)\right)\\
& = \norm{a_i}\cdot \left(\frac{b_i}{\norm{a_i}}-\left(1-2\sin^2{\frac{x}{2}}\right)\right)\\
& = \norm{a_i}\cdot \left(\frac{b_i}{\norm{a_i}}-1+2\sin^2{\frac{x}{2}}\right),
\end{split}
\end{equation}
where the first equality holds by the definition of $h$ and the second equality holds since $\sin(\frac{\pi}{2}+x) = \cos(x) = (1-2\sin^2(\frac{x}{2}))$ for every $x\in\REAL$.

Let $c\geq 1$, $X = [0,\pi/2]$, and $f:X\to [0,\infty)$ such that $f(x) = \sin(x)$. Since $f''(x) = -\sin(x) \leq 0$ for every $x\in X$, we have by Lemma~\ref{lem:fcx} that
\[
\sin(cx) = f(cx) \leq c\cdot f(x) = c\cdot \sin(x)
\]
for every $x\in \displaystyle X \cap \frac{X}{c}$.

By taking the square of the last inequality, it holds that for every $x\in \displaystyle X \cap \frac{X}{c}$,
\begin{equation} \label{eq:sin2cx}
\sin^2(cx) \leq c^2 \cdot \sin^2(x).
\end{equation}
Thus, for every $x\in \displaystyle I' \cap \frac{I'}{c} = X\cap \frac{X}{c}$, it holds that
\begin{equation} \label{eq:casei}
\begin{split}
h(cx) & = \norm{a_i}\cdot \left(\frac{b_i}{\norm{a_i}}-1+2\sin^2{\frac{cx}{2}}\right)\\
& \leq \norm{a_i}\cdot \left(\left(\frac{b_i}{\norm{a_i}}-1\right)+2c^2 \sin^2{\frac{x}{2}}\right)\\
& \leq c^2\cdot \norm{a_i}\cdot \left(\frac{b_i}{\norm{a_i}}-1+2 \sin^2{\frac{x}{2}}\right) = c^2\cdot h(x),
\end{split}
\end{equation}
where the first and last equalities are by~\eqref{eq:hx}, the first inequality is by combining $I'=X$ and~\eqref{eq:sin2cx}, and the second inequality holds since $\frac{b_i}{\norm{a_i}}-1 \geq 0$. By Definition~\ref{def:lip}, it follows that $h(x)$ is $2$-log-Lipschitz for every $x\in I'$.

By substituting in Definition~\ref{def:PieceLip} $g=g_i, m=1, X=X_1=I$ and $\dist(a,b)=|a-b|$ for every $a,b\in\REAL$, Properties (i)-(iii) of Definition~\ref{def:PieceLip} hold for $g_i$ since
\renewcommand{\labelenumi}{(\roman{enumi})}
\begin{enumerate}
\item $g_i$ has a unique infimum $x_1=\pi/2+\alpha_i$ in $X_1$.
\item We have $[0,\max_{x\in I} |x-x_1|] =[0,\max_{x\in [a_i,a_i+\pi]} |x-\pi/2-a_i|]=[0,\max_{x\in [0,\pi]} |x-\pi/2|]= [0,\pi/2] = I'$
and by~\eqref{eq:casei} $h:I' \to [0,\infty)$ is $2$-log-Lipschitz.
\item $g_i(t) = h(\abs{\pi/2 + \alpha_i -t}) = h(\dist(x_1,t))$ by~\eqref{eq:gicase1}.
\end{enumerate}
Hence, $g_i:I\to \REAL$ is a piecewise $2$-log-Lipschitz function.

\end{proof}

\begin{claim} \label{claim:case2}
If $\frac{b_i}{\norm{a_i}} < 1$ then $g_i:[-\frac{\pi}{2}+\alpha_i,\frac{3\pi}{2}+\alpha_i)\to \REAL$ is piecewise $2$-log-Lipschitz.
\end{claim}
\begin{proof}
Let
\begin{align*}
& \alpha = \arcsin\left(\frac{b_i}{\norm{a_i}}\right) \in [0,\frac{\pi}{2}), \\
& \alpha^*_{1} = -\frac{\pi}{2}+\alpha_i,\\
& \alpha^*_{2} = \alpha + \alpha_i,\\
& \alpha^*_{3} = \frac{\pi}{2}+\alpha_i,\\
& \alpha^*_{4} = \pi - \alpha + \alpha_i,\\
& \alpha^*_{5} = \frac{3\pi}{2}+\alpha_i, \\
& X_1 = [\alpha^*_{1},\alpha^*_{2}), I_1 = [0,\alpha^*_{2}-\alpha^*_{1}) = [0,\pi/2+\alpha), \\
& X_2 = [\alpha^*_{2},\alpha^*_{3}), I_2 = [0,\alpha^*_{3}-\alpha^*_{2}) = [0,\pi/2-\alpha), \\
& X_3 = [\alpha^*_{3},\alpha^*_{4}), I_3 = [0,\alpha^*_{4}-\alpha^*_{3}) = [0,\pi/2-\alpha), \\
& X_4 = [\alpha^*_{4},\alpha^*_{5}), I_4 = [0,\alpha^*_{5}-\alpha^*_{4}) = [0,\pi/2+\alpha),
\end{align*}
and let
\[
\begin{split}
\frac{h_1(t)}{\norm{a_i}} & :=\frac{b_i}{\norm{a_i}}-\sin(\alpha^*_{2}-t-\alpha_i)\\ &= \frac{b_i}{\norm{a_i}}-\sin(\alpha-t) \text{ for every }t\in I_1, \\
\frac{h_2(t)}{\norm{a_i}} & := \sin(\alpha^*_{2}+t-\alpha_i)-\frac{b_i}{\norm{a_i}}\\ & = \sin(\alpha+t)-\frac{b_i}{\norm{a_i}} \text{ for every }t\in I_2, \\
\frac{h_3(t)}{\norm{a_i}} & := \sin(\alpha^*_{4}-t-\alpha_i)-\frac{b_i}{\norm{a_i}}\\ & = \sin(\pi-\alpha-t)-\frac{b_i}{\norm{a_i}} = \sin(\alpha+t)-\frac{b_i}{\norm{a_i}}\\ &\text{ for every }t\in I_3, \\
\frac{h_4(t)}{\norm{a_i}} & := \frac{b_i}{\norm{a_i}}-\sin(\alpha^*_{4}+t-\alpha_i)\\ & = \frac{b_i}{\norm{a_i}}-\sin(\pi-\alpha+t) = \frac{b_i}{\norm{a_i}}-\sin(\alpha-t)\\ &\text{ for every }t\in I_4. \\
\end{split}
\]
First, observe that $I = [\alpha_i,\alpha_i+\pi) \subseteq [-\frac{\pi}{2}+\alpha_i,\frac{3\pi}{2}+\alpha_i) = X_1\cup X_2\cup X_3 \cup X_4$. Second, for every $t \in \REAL$,
\[
g_i(t) = \begin{cases}
h_1(\alpha^*_{2}-t), & t\in X_1\\
h_2(t-\alpha^*_{2}), & t\in X_2\\
h_3(\alpha^*_{4}-t), & t\in X_3\\
h_4(t-\alpha^*_{4}), & t\in X_4\\
0, & \text{otherwise}.
\end{cases}.
\]

Let $c>1$. We now prove that $g_i(t)$ is piecewise $2$-log-Lipschitz for every $t\in [-\frac{\pi}{2}+\alpha_i,\frac{3\pi}{2}+\alpha_i)$ by proving that

\textbf{(i):} $h_1(t)$ is $2$-log-Lipschitz in $I_1$, i.e., $h_1(ct) \leq c^2 \cdot h_1(t)$ for every $t\in \displaystyle I_1 \cap \frac{I_1}{c} = \frac{I_1}{c}$.

\textbf{(ii):} $h_2(t)$ is $1$-log-Lipschitz in $I_2$, i.e., $h_2(ct) \leq c \cdot h_2(t)$ for every $t\in \displaystyle I_2 \cap \frac{I_2}{c} = \frac{I_2}{c}$.

Since $I_2=I_3$ and $h_2(t) = h_3(t)$ for every $t\in I_3$, we conclude by (ii) that $h_3(t)$ is $1$-log-Lipschitz in $I_3$, i.e., $h_3(ct) \leq c \cdot h_3(t)$ for every $t\in \displaystyle I_3 \cap \frac{I_3}{c}$.

Since $I_1=I_4$ and $h_1(t) = h_4(t)$ for every $t\in I_4$, we conclude by (i) that $h_4(t)$ is $2$-log-Lipschitz in $I_4$, i.e., $h_4(ct) \leq c^2 \cdot h_4(t)$ for every $t\in \displaystyle I_4 \cap \frac{I_4}{c}$.

The proof of (i) and (ii) above is as follows.

\textbf{Proof of (i): }Clearly, if $t=0$, (i) trivially holds as
\begin{equation} \label{eq:h1_t0}
h_1(ct)=h_1(t)\leq c\cdot h_1(t).
\end{equation}

Let $d_1:[I_1 \cap \frac{I_1}{c} \setminus\br{0}]:\to[0,\infty)$ such that $d_1(t)=\displaystyle \frac{h_1(ct)}{h_1(t)}$. The denominator is positive since $h_1(t)>0$ in the range of $d_1$.

We first prove that $d_1$ does not get its maximum at $t'=\frac{\pi/2+\alpha}{c}$. Observe that

\begin{equation} \label{eq:d1_2}
\begin{split}
& \frac{1}{\norm{a_i}^2} \left(c\cdot h_1'(ct') \cdot h_1(t')-h_1(ct') \cdot h_1'(t')\right)\\
& = c\cdot \cos(\alpha - ct') \cdot \left(\frac{b_i}{\norm{a_i}}-\sin(\alpha-t')\right)\\ &- \left( \left(\frac{b_i}{\norm{a_i}}-\sin(\alpha-ct')\right)\cdot \cos(\alpha-t') \right)\\
& = -\left(\frac{b_i}{\norm{a_i}}+1\right)\cdot\left(\cos\left(\alpha-\frac{\pi/2+\alpha}{c}\right)\right)\\
& < 0,
\end{split}
\end{equation}
where the second derivation holds since $\alpha-ct' = -\pi/2$, and the last derivation holds since $\alpha-\frac{\pi/2+\alpha}{c} \in (-\pi/2,\pi/2)$.

Since $d_1'(t') = \frac{c\cdot h_1'(ct') \cdot (h_1(t'))-h_1(ct') \cdot (h_1'(t'))}{h_1^2(t')}$, we have by~\eqref{eq:d1_2} that $d_1'(t')<0$. Hence, $d_1$ is decreasing in the neighbourhood of $t'$. Therefore, $t'$ is not an extreme point of $d_1$.

We now prove that $d_1(t) \leq c^2$ for $t \in (0,\frac{\pi/2+\alpha}{c})$.
Suppose that $t^*$ maximizes $d_1(t)$ over the open interval $(0,\frac{\pi/2+\alpha}{c})$. Since $d_1$ is continuous in an open interval, the derivation of $d_1$ at $t^*$ is zero, i.e.,
\[
0 = c\cdot h_1'(ct^*) \cdot (h_1(t^*))-h_1(ct^*) \cdot (h_1'(t^*)).
\]
We also have that $h_1'(t^*)/\norm{a_i} = \cos(\alpha-t^*) > 0$ since
\[
\alpha-t^* \in (\alpha-(\pi/2+\alpha),\alpha]\subseteq (-\pi/2,\pi/2).
\]

Hence,
\begin{equation} \label{eq:d1t1}
\begin{split}
d_1(t^*)& = \frac{h_1(ct^*)}{h_1(t^*)} = \frac{c\cdot h_1'(ct^*)}{h_1'(t^*)} = \frac{c \norm{a_i}\cdot \cos(\alpha_{i2}^*-ct^*-\alpha_i)}{\norm{a_i}\cdot \cos(\alpha_{i2}^*-t^*-\alpha_i)}\\ & = \frac{c\cdot \cos(\alpha-ct^*)}{\cos(\alpha-t^*)}.
\end{split}
\end{equation}
For every $x\in \displaystyle I_1 \cap \frac{I_1}{c}$, let $f(x) = \cos(\alpha-x) = \cos(x-\alpha)$. Observe that $f''(x) = -\cos(x-\alpha) \leq 0$ for every $x\in \displaystyle I_1 \cap \frac{I_1}{c} = \frac{I_1}{c}$. Substituting $f$ and $X = I_1$ in Lemma~\ref{lem:fcx} yields
\begin{equation} \label{eq:cosct}
\cos(\alpha-cx) = f(cx) \leq c\cdot f(x) = c\cdot \cos(\alpha-x).
\end{equation}
Hence, for every $t\in \displaystyle \frac{I_1}{c}$ it follows that
\begin{equation} \label{gipart1}
d_1(t) \leq d_1(t^*) = c\cdot\frac{\cos(\alpha-ct^*)}{\cos(\alpha-t^*)} \leq c^2,
\end{equation}
where the first derivation is by the definition of $t^*$, the second derivation is by~\eqref{eq:d1t1} and the last derivation is by substituting $x=t$ in~\eqref{eq:cosct}.
We also have that
\begin{equation} \label{eq:d1_1}
\begin{split}
\lim_{t\to 0} d_1(t)& = \lim_{t\to 0} \frac{h_1(ct)}{h_1(t)} = \lim_{t\to 0} \frac{c\cdot h_1'(ct)}{h_1'(t)}\\ & = c \cdot \lim_{t \to 0} \frac{\norm{a_i}\cdot \cos(\alpha-ct)}{\norm{a_i}\cdot \cos(\alpha-t)} = c\cdot\frac{\cos(\alpha)}{\cos(\alpha)} = c,
\end{split}
\end{equation}
where the second equality holds by L'hospital's rule since $h_1(ct) = h_1(t) = 0$ for $t=0$.

By combining~\eqref{eq:d1_2}, ~\eqref{gipart1} and~\eqref{eq:d1_1}, we get that $d_1(t) = \frac{h_1(ct)}{h_1(t)} \leq c^2$ for every $t\in \displaystyle I_1 \cap \frac{I_1}{c} \setminus\br{0}$.
By combining~\eqref{eq:h1_t0} with the last inequality, (i) holds as $h_1(ct) \leq c^2 h_1(t)$ for every $t\in \displaystyle I_1 \cap \frac{I_1}{c}$.

\textbf{Proof of (ii): }We have that $h_2(t) = \norm{a_i} \cdot \big(\sin(t+\alpha)-\frac{b_i}{\norm{a_i}}\big)$ for $t\in I_2$. Since $h_2''(t) = -\norm{a_i}\cdot \sin(t+\alpha) \leq 0$ for every $t\in \displaystyle \frac{I_2}{c}$, by substituting $f(x) = h_2(x)$ and $X = I_2$ in Lemma~\ref{lem:fcx}, (ii) holds as
\begin{equation} \label{eq:gipart2}
h_2(ct) = f(ct) \leq c\cdot f(t) = c\cdot h_2(t)
\end{equation}
for every $t \in \displaystyle I_2 \cap \frac{I_2}{c}$.

By (i) and (ii) we can now prove the lemma as follows.
Observe that $\alpha^*_{2}$ and $\alpha^*_{4}$ are the minima of $g_i(t)$ over $t\in [-\pi/2+\alpha_i,\alpha_i+3\pi/2]$ when $\frac{b_i}{\norm{a_i}} < 1$. By Definition~\ref{def:PieceLip} we have that $g_i(t)$ is piecewise $2$-log-Lipschitz function in $[-\frac{\pi}{2}+\alpha_i,\frac{3\pi}{2}+\alpha_i)$ since
\begin{enumerate}
\item $g_i$ has a unique infimum $x_1 = \alpha_{i2}^*= \alpha+\alpha_i$ in $X_1$, $h_1:[0,\max_{x\in X_1}|x_1-x|]\to [0,\infty)$ is $2$-log-Lipschitz in $I_1$, and $g_i(t) = h_1(\alpha^*_{2}-t) = h_1(|\alpha^*_{2}-t|)$ for every $t\in X_1$.
\item $g_i$ has a unique minimum $x_2 = \alpha_{i2}^*= \alpha+\alpha_i$ in $X_2$, $h_2:[0,\max_{x\in X_2}|x_2-x|]\to [0,\infty)$ is $1$-log-Lipschitz in $I_2$, and $g_i(t) = h_2(t-\alpha^*_{2}) = h_2(|t-\alpha^*_{2}|)$ for every $t\in X_2$.
\item $g_i$ has a unique infimum $x_3 = \alpha_{i4}^*= \pi - \alpha + \alpha_i$ in $X_3$, $h_3:[0,\max_{x\in X_3}|x_3-x|]\to [0,\infty)$ is $1$-log-Lipschitz in $I_3$, and $g_i(t) = h_3(\alpha^*_{4}-t) = h_3(|\alpha^*_{4}-t|)$ for every $t\in X_3$.
\item $g_i$ has a unique minimum $x_4 = \alpha_{i4}^*= \pi - \alpha + \alpha_i$ in $X_4$, $h_4:[0,\max_{x\in X_4}|x_4-x|]\to [0,\infty)$ is $2$-log-Lipschitz in $I_4$, and $g_i(t) = h_4(t-\alpha^*_{4}) = h_4(|t-\alpha^*_{4}|)$ for every $t\in X_4$.
\end{enumerate}
\end{proof}

Recall that $i\in [n]$. Since $I \subset [-\frac{\pi}{2}+\alpha_i,\frac{3\pi}{2}+\alpha_i)$, by Claims~\ref{claim:case1} and~\ref{claim:case2} it holds that $g_i$ is piecewise $2$-log-Lipschitz function for every $t\in I$.

Hence, by substituting $g(p,\cdot)$ in Theorem~\ref{Lem:piecewiseApprox} with $g_i(\cdot)$, and $r=2$, there exists a minimum $t'\in \bigcup_{j=1}^n M(g_j(\cdot))$ such that for every $t \in I$,
\begin{equation} \label{Eq:gt2_leq_gt}
g_i(t') \leq 2^r \cdot g_i(t) = 4 g_i(t).
\end{equation}

Let $x\in\REAL^2$ be a unit vector, $t\in\REAL$ such that $x=x(t)$, $x' = (\sin{t'},-\cos{t'})$, and let $k\in [n]$ be the index such that $t' \in M(g_k(\cdot))$. Hence,
\[
\begin{split}
|a_i^Tx'-b_i| & = |a_i^T x(t') - b_i| =  g_i(t') \leq 4 \cdot g_i(t)\\ &= 4 \cdot |a_i^T x(t) -b_i| = 4\cdot |a_i^T x - b_i|,
\end{split}
\]
where the first equality holds by the definition of $x'$, the second and fourth equalities holds by~\eqref{Eq:gt_axb}, the inequality holds by~\eqref{Eq:gt2_leq_gt}, and the last equality holds since $x=x(t)$.

\textbf{Proof of Lemma~\ref{lem:axb_proof} (ii): }The proof follows immediately by the definition of $k$ in (i).

Observe that it takes $O(1)$ time to compute $M(g_i(\cdot))$ for a fixed $i\in [n]$.
Hence, Lemma~\ref{lem:axb_proof} holds by letting $C = \br{x(t) \mid t\in \bigcup_{i\in [n]}M(g_i(\cdot))}$, since it takes $O(n)$ time to compute $C$.
\end{proof}

\subsection{Aligning Points-To-Lines}

\begin{lemma} [Lemma~\ref{Lemma:PQZx}] \label{Lemma:PQZx_proof}
Let $v = (v_x,v_y)$ be a unit vector such that $v_y \neq 0$ and $\ell=\mathrm{sp}\br{v}$ be the line in this direction.
Let $p,q,z\in\REAL^2$ be the vertices of a triangle such that $\norm{p-q}>0$. Let $P,Q,Z\in\REAL^{2\times 2}$ be the output of a call to \secondalgname$(v,p,q,z)$; see Algorithm~\ref{Alg:calcz}. Then the following hold:
\begin{enumerate}
\renewcommand{\labelenumi}{(\roman{enumi})}
\item We have both $p \in x$-axis and $q \in \ell$ if and only if there is a unit vector $x\in\REAL^2$ such that $p=Px$ and $q=Qx$. \label{PQZxBullet1}
\item For every unit vector $x\in\REAL^2$, we have that $z=Zx$ if $p=Px$ and $q=Qx$. \label{PQZxBullet2}
\end{enumerate}
\end{lemma}
\begin{proof}
We use the variables that are defined in Algorithm~\ref{Alg:calcz}.\\
\textbf{(i) $\Leftarrow$:} Let $\alpha\in[0,2\pi)$ and $x=(\sin{\alpha},\cos{\alpha})^T$ such that $p=Px$ and $q=Qx$. We need to prove that $p\in x$-axis and $q\in\ell$. Indeed,
\[
\dist(p,x\text{-axis}) = |(0,1) p| = |(0,1) Px| = 0,
\]
where the first equality holds since the vector $(0,1)$ is orthogonal to the $x$-axis and the last equality holds since $(0,1) P = (0,0)$ by the definition of $P$ in Line \ref{Alg1line:PDef} of Algorithm~\ref{Alg:calcz}. Hence, $p\in x$-axis. Let $s=\frac{v_x}{v_y}$. We also have
\[
\begin{split}
\dist(q,\ell) & = |{v^\bot}^T q| = |{v^\bot}^T Qx|\\
& = r_1 \left| \left( \begin{array}{ccc}
v_y \\
-v_x \end{array} \right)^T \left( \begin{array}{ccc}
s & 0 \\
1 & 0 \end{array} \right) \left( \begin{array}{ccc}
\sin{\alpha} \\
\cos{\alpha} \end{array} \right) \right|\\
& = r_1 \left| \left( \begin{array}{ccc}
0 & 0 \end{array} \right) \left( \begin{array}{ccc}
\sin{\alpha} \\
\cos{\alpha} \end{array} \right)\right| = 0 \nonumber,
\end{split}
\]
where the third equality holds by the definition of $v^\bot = (v_y,-v_x)^T$, and the fourth equality holds since $s=\frac{v_x}{v_y}$. Hence, $q\in\ell$.

\noindent \textbf{(i) $\Rightarrow$:}  Let $p\in x$-axis and $q\in\ell$. We need to prove that there exists a unit vector $x\in\REAL^2$ such that $p=Px$ and $q=Qx$.
Let $t_p,t_q \in \REAL$ such that $p=(t_p,0)^T$, $q = t_q(v_x,v_y)^T$ and $\Delta (p,q,o)$ be a triangle, where $o$ is the origin of $\REAL^2$. Let $\alpha\in [0,\pi)$ be the interior angle at vertex $p$ and let $\beta\in [0,\pi)$ be the angle at vertex $q$ of the triangle $\Delta (p,q,o)$. Observe that the sin of the angle at vertex $o$ is simply $v_y$. Let $x=(\sin{\alpha},\cos{\alpha})^T$.
By simple trigonometric identities in the triangle $\Delta (p,q,o)$, we have
\begin{equation} \label{Eq:sinBDef}
\sin{\beta} = v_y \cos{\alpha}+v_x \sin{\alpha}.
\end{equation}
We continue with the following case analysis: \textbf{(a):} $\alpha,\beta \neq 0$, \textbf{(b):} $\alpha = 0, \beta \neq 0$ and \textbf{(c):} $\alpha \neq 0, \beta = 0$.

\noindent \textbf{Case (a) $\alpha,\beta \neq 0$: }By the law of sines we have that $\norm{q-o}/\sin(\alpha) = \norm{p-o}/\sin{\beta} = \norm{p-q}/v_y$. Let $r_1=\norm{p-q}$. Hence, respectively,
\begin{equation}\label{Eq:sinADef}
\frac{t_q}{\sin{\alpha}} = \frac{r_1}{v_y}, \text{therefore }\sin{\alpha} = \frac{t_q \cdot v_y}{r_1},
\end{equation}
and
\begin{equation}\label{Eq:tpDef}
\frac{t_p}{\sin{\beta}} = \frac{r_1}{v_y}, \text{therefore } t_p = \frac{r_1 \cdot \sin{\beta}}{v_y}.
\end{equation}

Then it holds that
\begin{equation} \label{Eq:Qxq}
\begin{split}
Qx & = r_1\left( \begin{array}{ccc}
s & 0 \\
1 & 0 \end{array} \right) \left( \begin{array}{ccc}
\sin{\alpha} \\
\cos{\alpha} \end{array} \right)
= r_1\left( \begin{array}{ccc}
s\sin{\alpha} \\
\sin{\alpha} \end{array} \right)\\
& = r_1\left( \begin{array}{ccc}
\frac{v_x}{v_y} \cdot \frac{t_q \cdot v_y}{r_1} \\
\frac{t_q \cdot v_y}{r_1} \end{array} \right)
= t_q \left( \begin{array}{ccc}
v_x \\
v_y \end{array} \right)\\
& = t_q \cdot v = q,
\end{split}
\end{equation}
where the third equality holds by combining~\eqref{Eq:sinADef} with the definition $s=\frac{v_x}{v_y}$, and the last equality holds by the definition of $t_q$.
It also holds that
\begin{align}
Px & =
r_1\left( \begin{array}{ccc}
s & 1 \\
0 & 0 \end{array} \right) \left( \begin{array}{ccc}
\sin{\alpha} \\
\cos{\alpha} \end{array} \right)
= r_1\left( \begin{array}{ccc}
s\sin{\alpha}+\cos{\alpha} \\
0 \end{array} \right) \nonumber \\
& = r_1\left( \begin{array}{ccc}
\frac{v_x}{v_y}\sin{\alpha}+ \cos{\alpha} \\
0 \end{array} \right)
= r_1\left( \begin{array}{ccc}
\frac{\sin{\beta}}{v_y} \label{eqthird12}\\
0 \end{array} \right)\\
& = \left( \begin{array}{ccc}
t_p\\
0 \end{array} \right) = p \label{eqfifth}
\end{align}
where the two equalities in~\eqref{eqthird12} hold respectively by the definition $s=\frac{v_x}{v_y}$ and by~\eqref{Eq:sinBDef}, and the two equalities in~\eqref{eqfifth} hold by \eqref{Eq:tpDef} and the definition of $t_p$ respectively. Combining~\eqref{Eq:Qxq} and ~\eqref{eqfifth} yields that the vector $x = (\sin{\alpha},\cos{\alpha})^T$ satisfies $p=Px$ and $q=Qx$.

\noindent \textbf{Case (b) $\alpha = 0$: } In this case, since $\alpha = 0$, it holds that $q$ intersects the origin $o$, i.e., $q = (0,0)^T$, and $p = \pm (r_1,0)$ since $\norm{p-q}=r_1$. Then by setting $x = \pm(0,1)^T$ respectively, it holds that, $Px = \pm(r_1,0)^T = p$ and $Qx = (0,0)^T = q$. Hence, there exists a unit vector $x\in\REAL^2$ such that $p=Px$ and $q=Qx
$.

\noindent \textbf{Case (c) $\beta = 0$: } In this case, since $\beta = 0$, it holds that $p$ intersects the origin $o$, i.e., $p = (0,0)^T$. Similarly to~\eqref{Eq:Qxq}, it holds that $Qx=q$ for $x=(\sin{\alpha},\cos{\alpha})^T$. It also holds that
\begin{align}
Px & =
r_1\left( \begin{array}{ccc}
s & 1 \\
0 & 0 \end{array} \right) \left( \begin{array}{ccc}
\sin{\alpha} \\
\cos{\alpha} \end{array} \right)
= r_1\left( \begin{array}{ccc}
s\sin{\alpha}+\cos{\alpha} \\
0 \end{array} \right) \\
& = r_1\left( \begin{array}{ccc}
\frac{v_x}{v_y}\sin{\alpha}+\cos{\alpha} \\
0 \end{array} \right)
= \left( \begin{array}{ccc}
0\\
0 \end{array} \right) = p \nonumber
\end{align}
where the fourth equality holds by combining~\eqref{Eq:sinBDef} and $\beta = 0$.
Hence, there exists a unit vector $x\in\REAL^2$ such that $q = Qx$ and $p = Px$

\noindent \textbf{(ii)}: Suppose that $p=Px$ and $q=Qx$ for some unit vector $x\in\REAL^2$. By (i) we have that $p\in x$-axis and $q\in\ell$. We need to prove that $z=Zx$. Consider the triangle $\Delta (p,q,z)$. Vertex $z$ lies at one of the intersection points of the pair of circles $c_1,c_2$ whose radii are $r_2 = \norm{p-z},r_3=\norm{q-z}$, centered at $p$ and $q$, respectively; See Fig.~\ref{fig:CircCircInter}. The position of $z$ is given by
\begin{equation} \label{eq:z1}
\begin{split}
z & = p + d_1 \cdot \frac{q-p}{\norm{q-p}} + b\cdot d_2 \Big(\frac{q-p}{\norm{q-p}}\Big)^\bot \\
& = p+\frac{d_1}{r_1}\cdot(q-p) + b\cdot \frac{d_2}{r_1}\cdot(q-p)^\bot \\
& = p+\frac{d_1}{r_1}\cdot(q-p) +b\cdot \frac{d_2}{r_1}\cdot R(q-p),
\end{split}
\end{equation}
where $b$ is $1$ or $-1$ if $z$ lies respectively in the halfplane to the right of the vector $(q-p)$ or otherwise. If $z$ lies exactly on the line spanned by the vector $(q-p)$, then $b$ can take either $1$ or $-1$ since, in that case, $d_2=0$; See illustration in Fig~\ref{fig:CircCircInter}. The first equality in~\eqref{eq:z1} holds since the vector $q-p$ and $(q-p)^\bot$ are an orthogonal basis of $\REAL^2$, and the last equality holds since $R$ is a $\pi/2$ radian counter clockwise rotation matrix, i.e., for every vector $v\in\REAL^2$ it satisfies that $v^\bot = Rv$.
Substituting $p=Px$ and $q=Qx$ in~\eqref{eq:z1} proves Claim (ii) as
\[
\begin{split}
z & = Px+\frac{d_1}{r_1}\cdot(Qx-Px) +b\cdot \frac{d_2}{r_1}\cdot R(Qx-Px)\\\ & = \big(P+\frac{d_1}{r_1}\cdot(Q-P) +b\cdot \frac{d_2}{r_1}\cdot R(Q-P)\big)x = Zx.
\end{split}
\]
\end{proof}

\begin{figure}
	\centering
    \includegraphics[width=0.4\textwidth,scale=0.2]{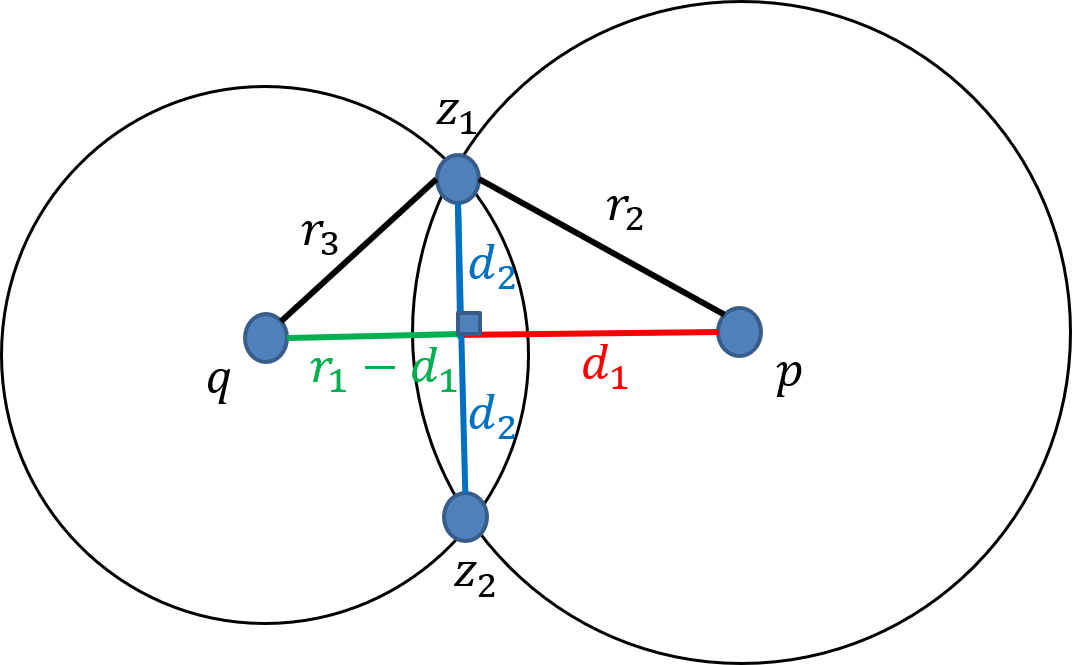}
    \caption{A pair of circles that are centered at $p$ and $q$ respectively, and intersect at $z_1$ and $z_2$. The line between $p$ and $q$ is always orthogonal to the line between $z_1$ and $z_2$ and partitions the segment between $z_1$ and $z_2$ into two equal halves.}
    \label{fig:CircCircInter}
\end{figure}

\begin{corollary} \label{Cor:RZx}
Let $\ell$ be a line on the plane that intersects the origin and is spanned by the unit vector $v\in\REAL^2$. Let $p,q\in\REAL^2$ such that $p$ is on the $x$-axis, $q$ is on line $\ell$ and $\norm{p-q} > 0$. For every $z\in \REAL^2$, let $(P,Q,Z(z))$ denote the output of a call to \secondalgname$(v,p,q,z)$; see Algorithm~\ref{Alg:calcz}. Then (i) and (ii) holds as follows.
\renewcommand{\labelenumi}{(\roman{enumi})}
\begin{enumerate}
\item There is a unit vector $x\in\REAL^2$ such that for every $z\in \REAL^2$ we have $z=Z(z)x$.
\item For every unit vector $x\in\REAL^2$, there is a corresponding alignment $(R,t)$, such that $Rz-t = Z(z)x$ for every $z\in \REAL^2$.
\end{enumerate}
\end{corollary}
\begin{proof}
\textbf{(i): }By Lemma~\ref{Lemma:PQZx_proof} (i), since $p\in x$-axis and $q\in\ell$, there is a unit vector $x\in\REAL^2$ such that $p=Px$ and $q=Qx$. By Lemma~\ref{Lemma:PQZx_proof} (ii), since $p=Px$ and $q=Qx$, we have $z=Z(z)x$ for every $z\in\REAL^2$. Hence, there is a unit vector $x$ such that $z=Z(z)x$ for every $z\in\REAL^2$.

\noindent \textbf{(ii): }Let $x\in\REAL^2$ be a unit vector. Since $\norm{Px-Qx}=\norm{p-q}$ for every unit vector $x\in\REAL^2$, there is an aligning $(R,t)$ that aligns the points $p$ and $q$ with the points $Px$ and $Qx$, i.e., $Rp-t = Px$ and $Rq-t = Qx$.
Let $z\in\REAL^2$. Consider the triangle $\Delta (p',q',z')$ after applying the alignment $(R,t)$ to triangle $\Delta (p,q,z)$. Hence, $p' = Rp-t$, $q' = Rq-t$ and $z' = Rz-t$. By the definition of $(R,t)$, we also have $p'=Px$ and $q'=Qx$. By Lemma~\ref{Lemma:PQZx_proof} (ii), the unit vector $x$ satisfies $z'=Z(z)x$. Hence, $z' = Rz-t = Z(z)x$.
\end{proof}

The proof of the following lemma is divided into 3 steps that correspond to the steps discussed in the intuition for Algorithm~\ref{Alg:slowApprox} in Section~\ref{mainAlgIntuition}.
\begin{lemma} [Lemma~\ref{theorem:2DPointsToLines}] \label{theorem:2DPointsToLines_proof}
Let $A=\br{(p_1,\ell_1),\cdots,(p_n,\ell_n)}$ be a set of $n \geq 3$ pairs, where for every $i \in [n]$, we have that $p_i$ is a point and $\ell_i$ is a line, both on the plane. Let $C \subseteq \alignments$ be an output of a call to \algname$(A)$; see Algorithm~\ref{Alg:slowApprox}. Then for every alignment $(R^*,t^*)\in \alignments$ there exists an alignment $(R,t)\in C$ such that for every $i\in[n]$,
\begin{equation} \label{eq:MainTheoremEq}
\dist(Rp_i-t,\ell_i) \leq 16\cdot \dist(R^*p_i-t^*,\ell_i).
\end{equation}
Moreover, $|C| \in O(n^3)$ and can be computed in $O(n^3)$ time.
\end{lemma}
\begin{proof}
Let $(R^*,t^*)\in \alignments$. Without loss of generality, assume that the set is already aligned by $(R^*,t^*)$, i.e., $R^*=I$ and $t^*=\mathbf{0}$.

\textbf{Step 1. }We first prove that there is $j\in[n]$, such that translating $p_j$ until it intersects $\ell_j$ does not increase the distance $\dist(p_i,\ell_i)$ by more than a multiplicative factor of $2$, for every $i \in [n]$. See Fig.~\ref{fig:mainProofStep1}.

\begin{figure}
\centering
\includegraphics[width=0.4\textwidth]{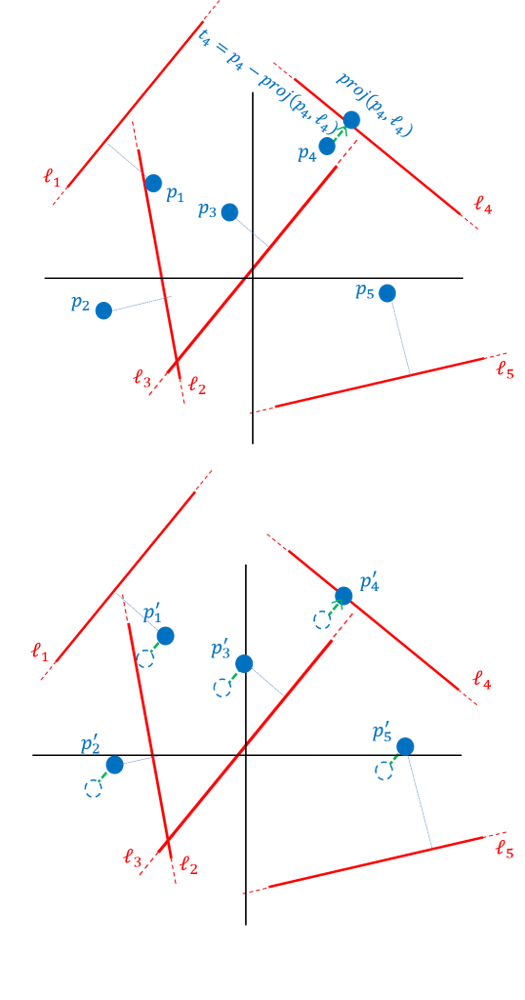}
\caption{\textbf{Illustration of Step 1 in the proof of Lemma~\ref{theorem:2DPointsToLines_proof}.} \textbf{(Top:)} The pair $(p_4,\ell_4)$ is the closest pair among the $5$ pairs $(p_1,\ell_1),\cdots,(p_5,\ell_5)$.
\textbf{(Bottom:)} Every point $p_i$ is then translated to  $p'_i=p_i-t_4$ where $t_4=p_4-\proj(p_4,\ell_4)$.
The new distance is larger by a factor of at most $2$, i.e., $\dist(p_i',\ell_i)\leq 2\dist(p_i,\ell_i)$ for every $i\in\br{1,\cdots,5}$.}
\label{fig:mainProofStep1}
\end{figure}

Put $i\in[n]$. Let $t_i=p_i-\proj(p_i,\ell_i)$ and let $j\in\argmin_{j'\in[n]} \norm{t_{j'}}$.
Hence,
\begin{align}
\dist(p_i-t_j,\ell_i) & \leq \norm{t_j} + \dist(p_i,\ell_i) \label{eq:hatt_tria}\\
& \leq \norm{t_i} + \dist(p_i,\ell_i) \label{eq:j_def}\\
& = 2\cdot\dist(p_i,\ell_i) \label{eq:normTjIsDist},
\end{align}
where \eqref{eq:hatt_tria} holds due to the triangle inequality, \eqref{eq:j_def} holds due to the definition of $j$, and \eqref{eq:normTjIsDist} holds since $\norm{t_i} = \dist(p_i,\ell_i)$.
For every $s\in[n]$, let $p_s'=p_s-t_j$ and without loss of generality assume that $p_j'$ is the origin, otherwise translate the coordinate system.

Let $v_i\in\REAL^{2}$ be a unit vector and $b_i\geq 0$ such that
\[
\ell_i = \{q\in\REAL^2 \mid v_i^Tq = b_i\}.
\]

In what follows we assume that $\ell_1,\cdots,\ell_n$ cannot all be mutually parallel, i.e., there exists at least two lines that intersect. We address the case where all the lines are parallel in Claim~\ref{claim:allLinesParallel}.

\textbf{Step 2. }We prove in Claim~\ref{2DRotation} that there is $k\in[n]$ such that translating $p_k'$ in the direction of $\ell_j$, until it gets as close as possible to $\ell_k$ will not increase the distance $\dist(p_i',\ell_i)$ by more than a multiplicative factor of $2$. See Fig.~\ref{fig:mainProofStep2}.

\begin{figure}
\centering
\includegraphics[width=0.4\textwidth]{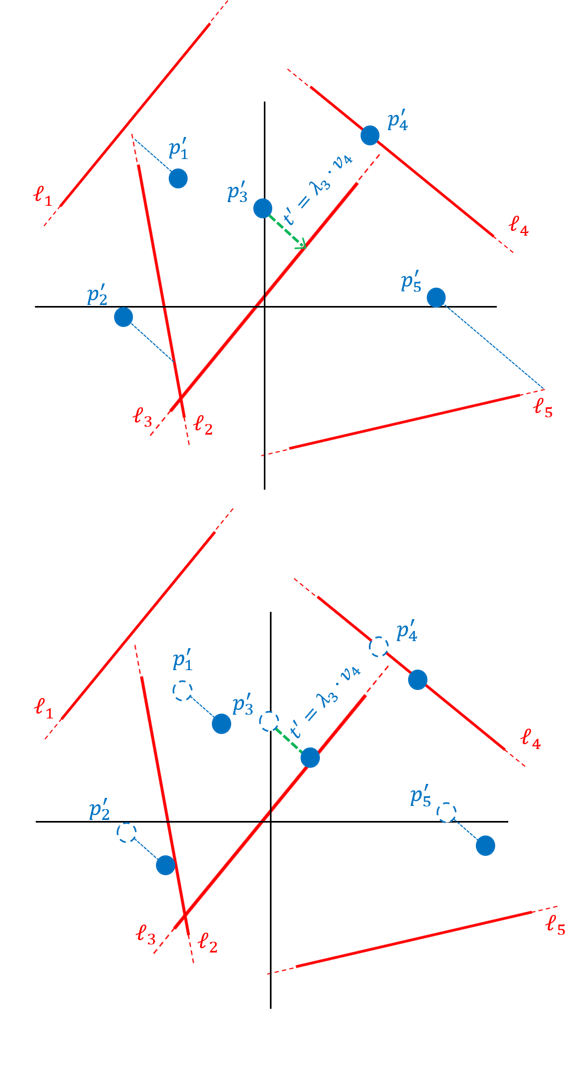}
\caption{\textbf{Illustration of Step 2 in the proof of Lemma~\ref{theorem:2DPointsToLines_proof}.} \textbf{(Top:)} The pair $(p_3',\ell_3)$ is the pair of minimal distance in the direction of $\ell_4$ among the 4 pairs $(p_1',\ell_1),(p_2',\ell_2),(p_3',\ell_3),(p_5',\ell_5)$. \textbf{(Bottom:)} The set of points is translated by the vector $t' = \lambda_3\cdot v_4$, where $v_4$ is the unit vector that spans $\ell_4$ and $\lambda_3$ is the minimal translation magnitude required in order for point $p_3'$ to get as close as possible to $\ell_3$.}
\label{fig:mainProofStep2}
\end{figure}

Let $\Lambda_i = \argmin_{\lambda \in \REAL} \dist(p_i'-\lambda\cdot v_j,\ell_i)$ and set
\[
\lambda_i = \begin{cases}
              \argmin_{\lambda \in \Lambda_i} |\lambda|, & \mbox{if } |\Lambda_i| < \infty \\
              \infty, & \mbox{otherwise}.
            \end{cases}
\]
Informally, $\lambda_i = \infty$ if and only if $\ell_j$ and $\ell_i$ are parallel lines, and it represents the distance that $p_i'$ needs to travel in the direction parallel to $\ell_i$ until it intersects $\ell_i$. Observe that there is at least one index $i \in [n]$ such that $\lambda_i \neq \infty$ since we assumed the lines are not all mutually parallel.
\begin{claim} \label{2DRotation}
There exists $k \in [n]\setminus\br{j}$ and a corresponding translation vector $t' = \lambda_k \cdot v_j$ such that for every $i \in [n]$,
\[
\dist(p_i'-t',\ell_i) \leq 2 \cdot \dist(p_i',\ell_i).
\]
\end{claim}
\begin{proof}
Let $k \in \argmin_{i \in [n]} |\lambda_i|$. Put $i\in [n]$.
Let $\ell_i' = \br{p_i'-\lambda v_j \mid \lambda \in \REAL}$. Observe that $\ell_i'$ is a line in $\REAL^2$. If $\ell_i'$ is parallel to $\ell_i$, then since both $p_i'-\lambda_k v_j$ and $p_i'$ are on $\ell_i'$, the claim trivially holds as
\[
\begin{split}
& \dist(p_i'-t',\ell_i) = \dist(p_i'-\lambda_k v_j,\ell_i)\\
& = \dist(p_i',\ell_i) \leq 2\cdot \dist(p_i',\ell_i).
\end{split}
\]
We now assume that $\ell_i'$ and $\ell_i$ are not parallel. Let $\alpha$ be the angle between $\ell_i$ and $\ell_i'$, and let $p = \ell_i'\cap \ell_i$. Observe that
\begin{equation} \label{eq:lambdai}
\lambda_i = \norm{p_i' - p},
\end{equation}
and that for every $q\in \ell_i'$,
\begin{equation} \label{eq:distBetweenLines}
\dist(q,\ell_i) = \sin(\alpha)\cdot \norm{q - p}.
\end{equation}
Then the claim holds as
\begin{align}
\dist(p_i'-t',\ell_i) & = \dist(p_i'-\lambda_k\cdot v_j,\ell_i) \nonumber\\
& = \sin(\alpha) \cdot \norm{p_i'-\lambda_k\cdot v_j - p} \label{eq1:norm}\\
& \leq \sin(\alpha) \cdot \left( \norm{p_i'- p}+\norm{\lambda_k \cdot v_j} \right) \label{eq2:triaIn}\\
& = \sin(\alpha) \cdot \left( \norm{p_i'- p}+|\lambda_k| \right) \label{eq3:vjUnit}\\
& \leq \sin(\alpha) \cdot 2\norm{p_i'- p} \label{eq4:main}\\
& = 2 \sin(\alpha) \cdot \norm{p_i'- p} \nonumber\\
& = 2\cdot \dist(p_i',\ell_i), \label{eq5:norms}
\end{align}
where~\eqref{eq1:norm} holds by substituting $q=p_i'-\lambda_k\cdot v_j$ in~\eqref{eq:distBetweenLines},~\eqref{eq2:triaIn} holds by the triangle inequality,~\eqref{eq3:vjUnit} since $\norm{v_j}=1$,~\eqref{eq4:main} holds by combining the definition of $k$ with~\eqref{eq:lambdai}, and~\eqref{eq5:norms} holds by substituting $q=p_i'$ in~\eqref{eq:distBetweenLines}.
\end{proof}

Let $k \in \argmin_{i \in [n]} |\lambda_i|$ and $t' = \lambda_k \cdot v_j$ be respectively the index and the translation vector that are computed in Claim~\ref{2DRotation}.
For every $i\in[n]$, let $p_i''$ denote the point $p_i$ after translation by $t_j+t'$, i.e., $p_i''=p_i' - t' = p_i - t_j - t'$.

Since $k$ is the index corresponding to the minimal $|\lambda_i|$, we have that $|\lambda_k| < \infty$ (since there is at least one such $|\lambda_i| \neq \infty$). By the definitions of $\lambda_i$ for every $i\in [n]$, we have that $|\lambda_i| \neq \infty$ only if $\ell_i$ and $\ell_j$ are not parallel. Combining the two previous facts yields that $\ell_j$ and $\ell_k$ are not parallel. Furthermore, since we defined $t' = \lambda_k \cdot v_j$ to be exactly the translation distance in the direction of $v_j$ needed for $p_k'$ to intersect $\ell_k$, we have that $p_k'' \in \ell_k$.

\textbf{Step 3. }In the following claim we prove that there is an alignment $(R'',t'')\in \alignments$ that satisfies the following: (i) $R''p_j''-t''\in\ell_j$ and $R''p_k''-t''\in\ell_k$, (ii) $(R'',t'')$ minimizes the distance $\dist(Rp_l''-t,\ell_l)$ for some index $l\in[n]$ over every $(R,t)$ that satisfy $Rp_j''-t\in\ell_j$ and $Rp_k''-t\in\ell_k$, and (iii) $(R'',t'')$ satisfies $\dist(R''p_i''-t'',\ell_i) \leq 4 \cdot \dist(p_i'',\ell_i)$ for every $i\in[n]$.

\begin{claim} \label{claim:2DRTapprox}
Let $\alignments_1 = \{(R,t)\in \alignments \mid Rp_j''-t \in \ell_j \text{ and } Rp_k''-t \in \ell_k\}$. Then there exists $l \in [n]\setminus \br{j,k}$ and an alignment $(R_l'',t_l'') \in \displaystyle \argmin_{(R,t)\in \alignments_1}\dist(Rp_l-t,\ell_l)$ such that for every $i\in[n]$,
\[
\dist(R_l''p_i''-t_l'',\ell_i) \leq 4\cdot \dist(p_i'',\ell_i).
\]
\end{claim}
\begin{proof}
Put $i\in[n]$. We have that $p_j''\in\ell_j$ since $p'_j$ is assumed to be the origin in Step 1, and $p_k''\in\ell_k$ by the definition $t'$. Let $R_{v_j}\in \REAL^{2\times 2}$ be a rotation matrix that aligns $v_j$ with the $x$-axis, i.e., $R_{v_j}v_j = (1,0)^T$. By substituting $p=p_j'',q=p_k'',z=p_i''$ and $v=R_{v_j}v_k$ in Lemma~\ref{Lemma:PQZx} (i) and (ii) (where $v_j$ and $v_k$ are rotated together such that $v_j$ is the origin), we have that there is a unit vector $x\in\REAL^2$ and matrices $Z_j=R_{v_j}^TP,Z_k=R_{v_j}^TQ$ and $Z_i=R_{v_j}^TZ$ such that $p_j''=Z_jx,p_k''=Z_kx$ and $p_i''=Z_ix$. The matrices $P,Q$ and $Z$ are multiplied by the inverse rotation matrix $R_{v_j}^T$ to simulate the inverse rotation which was applied on $v_j$ and $v_k$ before applying Lemma~\ref{Lemma:PQZx}. Observe that
\begin{equation} \label{eq:distIs_axb2}
\dist(p_i'',\ell_i) = |v_i^T p_i''-b_i| = |v_i^T Z_i x -b_i|,
\end{equation}
where the second equality holds since $p_i'' = Z_i x$.
By substituting $a_i = v_i^TZ_i$ in Lemma~\ref{lem:axb_proof} (i), there is a set $D$ of $O(n)$ unit vectors, $y\in D$ and an index $l\in [n]$, such that for every unit vector $x\in\REAL^2$
\begin{equation} \label{eq:axbApprox}
|v_i^TZ_i y -b_i| \leq 4 \cdot |v_i^TZ_i x -b_i|.
\end{equation}
Substituting $p=p_j''$ and $q=p_k''$ in Corollary~\ref{Cor:RZx} (ii) implies that there is an alignment $(R'',t'')\in \alignments_1$ that aligns the set $\{Z_ix \mid i\in[n]\}$ with the set $\{Z_iy \mid i\in [n]\}$, i.e. $R''p_i''-t = R''Z_ix-t'' = Z_iy$.
Hence, the following holds
\begin{align}
\dist(R''p_i''-t'',\ell_i) & = |v_i^T (R''p_i''-t'') -b_i| \nonumber \\
& = |v_i^T (R''Z_ix-t'') -b_i| \label{eq:zixDef} \\
& = |v_i^T Z_iy -b_i| \label{eq:zix'Def} \\
& \leq 4 \cdot |v_i^TZ_i x-b_i| \label{eq:xx'Approx} \\
& = 4 \cdot \dist(p_i'',\ell_i) \label{eqaxbToDist},
\end{align}
where~\eqref{eq:zixDef} holds since $p_i''=Z_ix$, \eqref{eq:zix'Def} holds since $R''Z_ix-t'' = Z_iy$, \eqref{eq:xx'Approx} holds by~\eqref{eq:axbApprox} and~\eqref{eqaxbToDist} holds by~\eqref{eq:distIs_axb2}.

It follows from the definition of $y$ and $l$ that the vector $y$ minimizes $|v_i^TZ_iy'-b_l|$ over every unit vector $y' \in \REAL^2$. Substituting $i=l$ in~\eqref{eq:zix'Def}, we have that $\dist(R''p_l''-t'',\ell_l) = |v_i^TZ_iy-b_l| = \min_{\norm{y'}=1} |v_i^TZ_iy'-b_l|$. Hence,
\[
(R'',t'') \in \displaystyle \argmin_{(R,t)\in \alignments_1}\dist(Rp_l-t,\ell_l).
\]
Hence, $l$ and $(R'',t'')$ satisfy the requirements of Claim~\ref{claim:2DRTapprox}.
\end{proof}

Hence, by Step 1, there exists an alignment that aligns a point $p_j$ with $\ell_j$ for some $j\in[n]$, and does not increase the distances of the other pairs by more than a multiplicative factor of $2$. This is the reason that all the output alignments of Algorithm~\ref{Alg:slowApprox} satisfy that one of the input points intersects its corresponding line.

By the proof in Step 2, there exists an alignment that aligns a point $p_k$ with $\ell_k$ for some $k\in[n]$, and does not increase the distances of the other pairs by more than a multiplicative factor of $2$. This is the reason that all the output alignments of Algorithm~\ref{Alg:slowApprox} satisfy that two of the input points intersect their corresponding lines, if the input lines are not all mutually parallel.

By the proof in Step 3, there exists an alignment that minimizes the distance between $p_l$ and $\ell_l$ for some $l \in [n]$, and maintains that $p_j\in\ell_j$ and $p_k\in\ell_k$. This alignment provably guarantees that the distance of the other pairs does not increase by more than a multiplicative factor of $4$.

Hence, if the input lines are not all mutually parallel, we proved there are $j,k,l\in[n]$ where $j\neq k$, and $(R,t)\in \alignments$, such that the following is satisfied
\begin{equation} \label{MainConds}
\begin{split}
&Rp_j-t\in\ell_j, Rp_k-t\in\ell_k, \dist(Rp_l-t,\ell_l) \text{ is minimized over}\\
&\text{every alignment that satisfies } Rp_j-t\in\ell_j \text{ and } Rp_k-t\in\ell_k,\\
& \text{ and } \dist(Rp_i-t,\ell_i) \leq 16 \cdot \dist(p_i,\ell_i) \text{ for every } i\in[n].
\end{split}
\end{equation}

So far we have proven Step 1 regardless of weather the lines are mutually parallel or not. However, we proved Step 2 and step 3 (Claims~\ref{2DRotation} and~\ref{claim:2DRTapprox} respectively) of Lemma~\ref{theorem:2DPointsToLines_proof} only for the case that the lines are not all mutually parallel. We now prove an alternative claim (Claim~\ref{claim:allLinesParallel}) instead of Claims~\ref{2DRotation} and~\ref{claim:2DRTapprox} for the case that \textbf{$\ell_1,\cdots,\ell_n$ are mutually parallel.}

Recall that $p_i' = p_i - \left(p_j-\proj(p_j,\ell_j)\right)$. By Step 1, we have that $\dist(p_i',\ell_i) \leq 2\cdot \dist(p_i,\ell_i)$. To handle the case where all the lines $\ell_1,\cdots,\ell_n$ are parallel, we now prove that rotating the points around $p_j'$ until the distance $\dist(p_h,\ell_h)$ for some $h\in [n]$ is minimized would increase the distance $\dist(p_i',\ell_i)$ for every $i\in [n]$ by a factor of at most $4$.
\begin{claim} \label{claim:allLinesParallel}
If $\ell_1,\cdots,\ell_n$ are mutually parallel, then there exists $h \in [n]\setminus\br{j}$ and a rotation matrix
\[
R' \in \argmin_{R} \dist(Rp_h',\ell_h),
\]
where the minimum is over every rotation matrix $R \in \REAL^{2\times 2}$, such that for every $i \in [n]$,
\[
\dist(R'p_i',\ell_i) \leq 4 \cdot \dist(p_i',\ell_i).
\]
\end{claim}
\begin{proof}
Let $x\in\REAL^2$ be an arbitrary (known) unit vector. Let $R_i \in \REAL^{2\times 2}$ be a rotation matrix that rotates $x$ to the direction of $p_i'$, i.e., $p_i'=\norm{p_i'}\cdot R_i x$, and define $a_i=\norm{p_i'}\cdot R_i^Tv_i$. Therefore
\begin{equation} \label{eq:distIs_axb}
\begin{split}
\dist(p_i',\ell_i)& = |v_i^Tp_i'-b_i|= \left|\norm{p_i'}\cdot v_i^TR_ix-b_i\right|\\ &= |a_i^Tx-b_i|.
\end{split}
\end{equation}
By Lemma~\ref{lem:axb_proof}(i), there is a set $C$ of $O(n)$ unit vectors and $x'\in C$, such that
\begin{align}
|a_i^T x' -b_i| \leq 4 \cdot |a_i^T x -b_i|. \label{2dRotxapprox}
\end{align}

Let $h\in[n]$ be the index from Lemma~\ref{lem:axb_proof}(ii).
Define $R' \in \REAL^{2 \times 2}$ to be the rotation matrix that satisfies $R'x = x'$. Hence,
\begin{align}
\dist(R'p_i',\ell_i) & = |v_i^T R'p_i' - b_i| \nonumber\\
& = \left| \norm{p_i'} \cdot v_i^T R' R_i x -b_i \right| \label{2DrotpiDef}\\
& = \left| \norm{p_i'} \cdot v_i^T R_i R' x -b_i \right| \label{2DrotChangeRot}\\
& = |a_i^T R' x - b_i| \label{2DrotaiDef}\\
& = |a_i^T x' - b_i| \label{2DrotRDef}\\
& \leq 4 \cdot |a_i^T x - b_i| \label{2Drotxapprox2}\\
& = 4 \cdot \dist(p_i',\ell_i) \label{2DrotDistIs_axb},
\end{align}
where~\eqref{2DrotpiDef} holds by the definition of $R_i$,~\eqref{2DrotChangeRot} holds since the $2D$ rotation matrix group is commutative,~\eqref{2DrotaiDef} holds by the definition of $a_i$,~\eqref{2DrotRDef} holds by the definition of $R'$,~\eqref{2Drotxapprox2} holds by~\eqref{2dRotxapprox} and~\eqref{2DrotDistIs_axb} holds by~\eqref{eq:distIs_axb}.

By the definition of $h$ and $x'$, it follows that $x'$ minimizes $|a_h^Tx'-b_h|$. Similarly to~\eqref{2DrotRDef}, it holds that $\dist(R'p_h',\ell_h) = |a_h^Tx'-b_h|$, so $R' \in \argmin_R \dist(Rp_h',\ell_h)$. Hence, $h$ and $R'$ satisfy the requirements of this Claim.
\end{proof}

\begin{figure}
\centering
\includegraphics[width=0.4\textwidth]{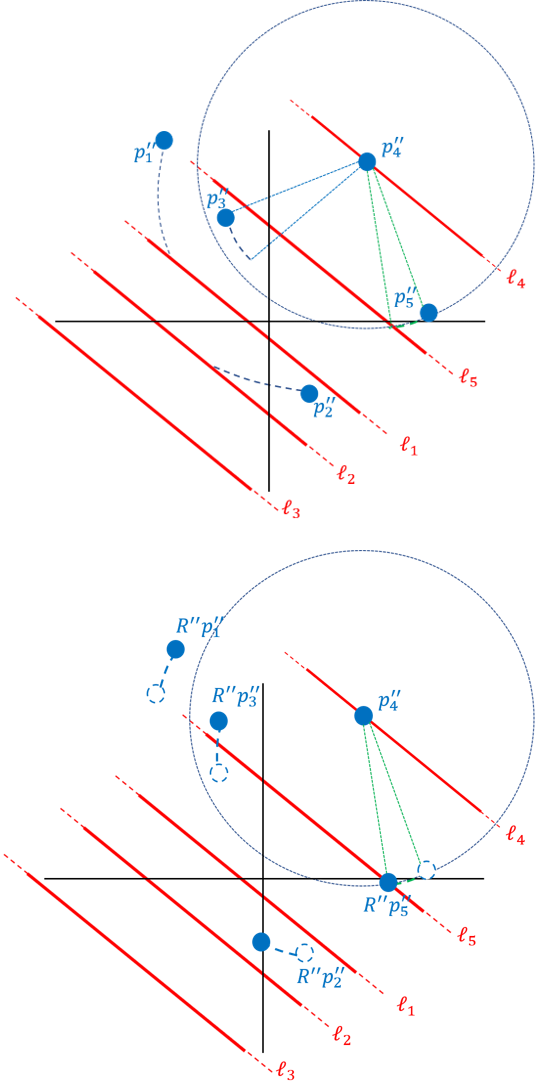}
\caption{\textbf{Illustration of Step 3 in the proof of Lemma~\ref{theorem:2DPointsToLines_proof}.} \textbf{(Top:)} The pair $(p_5'',\ell_5)$ is the pair that requires the minimal rotation angle in order to minimize the distance between $p_5''$ after the rotation and $\ell_5$, among the 4 pairs $(p_1'',\ell_1),(p_2'',\ell_2),(p_3'',\ell_3),(p_5'',\ell_5)$. \textbf{(Bottom:)} The set of points is rotated around $p_4''$ by a rotation matrix $R''$ such that $R''\in\argmin_{R}\dist(Rp_5'',\ell_5)$, where the minimum is over every rotation matrix $R\in\REAL^{2\times 2}$. The new distance is larger by a factor of 4 at most, i.e., $\dist(R''p_i'',\ell_i) \leq 4\cdot \dist(p_i'',\ell_i)$ for every $i\in\br{1,\cdots,5}$.}
\label{fig:mainProofParallelLines}
\end{figure}

Hence, by Step 1 and Claim~\ref{claim:allLinesParallel}, if all the lines are parallel, there are $j,h \in [n]$, where $j \neq h$, and $(R,t) \in \alignments$, such that the following holds
\begin{equation} \label{SecondaryConds}
\begin{split}
&Rp_j-t\in\ell_j,\\
&\dist(Rp_h-t,\ell_h) \text{ is minimized over every} (R,t) \text{ s.t. } Rp_j-t\in\ell_j,\\
& \text{ and } \dist(Rp_i-t,\ell_i) \leq 8 \cdot \dist(p_i,\ell_i) \text{ for every } i\in[n].
\end{split}
\end{equation}

Algorithm~\ref{Alg:slowApprox} iterates over every triplet $(j,k,l)\in[n]$. There are $O(n^3)$ such triplets. The alignments that satisfy~\eqref{MainConds} are computed in Lines~\ref{Alg2line:vjOrtho}~\ref{Alg2line:compC1} and stored in $C_1$. The alignments that satisfy~\ref{SecondaryConds} are computed in Line~\ref{Alg2line:compC2} and stored in $C_2$. The sets $C_1$ and $C_2$ computed at every iteration are of constant size, i.e., $|C_1|,|C_2|\in O(1)$.
We output $C$, the union of all those alignments. Hence, $C$ will contain an alignment $(R,t)$ that satisfies~\eqref{eq:MainTheoremEq}, and $|C| \in O(n^3)$. Furthermore, the running time of Algorithm~\ref{Alg:slowApprox} is $O(n^3)$ since Lines~\ref{Alg2line:setC2Empty}-~\ref{Alg2line:endES} take $O(1)$ time to compute, and are executed $O(n^3)$ times.
\end{proof}

\begin{theorem} [Theorem~\ref{theorem:costPointsToLines}] \label{theorem:costPointsToLines_proof}
Let $A=\br{(p_1,\ell_1),\cdots,(p_n,\ell_n)}$ be set of $n \geq 3$ pairs, where for every $i \in [n]$, we have that $p_i$ is a point and $\ell_i$ is a line, both on the plane. Let $z \geq 1$ and let $D_z:A\times \alignments\to [0,\infty)$ such that $D_z\left((p,\ell),(R,t)\right) = \min_{q\in \ell} \norm{Rp-t-q}_z$ is the $\ell_z$ norm between $Rp-t$ and $\ell$. Let $\cost, s, r$ be as defined in Definition~\ref{def:cost} for $D=D_z$. Let $w = 1$ if $z=2$ and $w=\sqrt{2}$ otherwise. Let $C$ be the output of a call to \algname$(A)$; see Algorithm~\ref{Alg:slowApprox}. Then there exists $(R',t')\in C$ such that
\[
\cost(A,(R',t')) \leq (w\cdot 16)^{rs} \cdot \min_{(R,t)\in\alignments}\cost(A,(R,t)).
\]
Furthermore, $C$ and $(R',t')$ can be computed in $n^{O(1)}$ time.
\end{theorem}
\begin{proof}
Put $i\in [n]$ and let
\[
(R^*,t^*) \displaystyle\in \argmin_{(R,t)\in\alignments}\cost(A,(R,t)).
\]
By Lemma~\ref{theorem:2DPointsToLines_proof}, $C$ is of size $|C| \in O(n^3)$, it can be computed in $O(n^3)$ time, and there is $(R',t')\in C$ that satisfies
\begin{equation} \label{eq:distLeq16Dist}
\begin{split}
D_2\left((p_i,\ell_i),(R',t')\right)& = \dist(R'p_i-t',\ell_i)\\
& \leq 16\cdot \dist(R^*p_i-t^*,\ell_i)\\
&= 16\cdot D_2\left((p_i,\ell_i),(R^*,t^*)\right).
\end{split}
\end{equation}

By~\eqref{eq:distLeq16Dist} and since the $\ell_2$-norm of every vector in $\REAL^2$ is approximated up to a multiplicative factor of $w$ (as defined in the theorem) by its $\ell_z$-norm, we obtain that
\[
D_z\left((p_i,\ell_i),(R',t')\right) \leq 16w D_z\left((p_i,\ell_i),(R^*,t^*)\right).
\]

By substituting $q' = (R',t'), q^* = (R^*,t^*)$ and $D = D_z$ in Observation~\ref{obs:distToCost}, Theorem~\ref{theorem:costPointsToLines_proof} holds as
\[
\begin{split}
\cost\left(A,(R',t')\right)& \leq (w \cdot 16)^{rs} \cdot \cost\left(A,(R^*,t^*)\right)\\
& = (w \cdot 16)^{rs} \cdot \min_{(R,t)\in\alignments}\cost(A,(R,t)),
\end{split}
\]
where the last derivation holds by the definition of $(R^*,t^*)$.

Furthermore, it takes $O(n\cdot |C|) = n^{O(1)}$ time to compute $(R',t')$.
\end{proof}

\begin{theorem} [Theorem~\ref{theorem:pointsToLinesNoMatching}] \label{theorem:pointsToLinesNoMatching_proof}
Let $A = \br{(p_1,\ell_1),\cdots,(p_n,\ell_n)}$ be a set of $n\geq 3$ pairs, where for every $i\in [n]$ we have that $p_i$ is a point and $\ell_i$ is a line, both on the plane. Let $z \geq 1$ and $D_z\left((p,\ell),(R,t)\right) = \displaystyle \min_{q\in \ell} \norm{Rp-t-q}_z$ for every point $p$ and line $\ell$ on the plane and alignment $(R,t)$. Consider $\cost, r$ to be a function as defined in Definition~\ref{def:cost} for $D = D_z$ and $f(v) = \norm{v}_1$. Let $w = 1$ if $z=2$. Otherwise, $w=\sqrt{2}$.
Let $(\tilde{R},\tilde{t},\tilde{\pi})$ be the output alignment $(\tilde{R},\tilde{t})$ and permutation $\tilde{\pi}$ of a call to \matchingalgname$(A,\cost)$; see Algorithm~\ref{Alg:slowApproxNoMatching}. Then
\begin{equation} \label{eq:matchingGoal}
\cost\left(A_{\tilde{\pi}},(\tilde{R},\tilde{t})\right) \leq (w\cdot 16)^r \cdot \min_{(R,t,\pi)}\cost\left(A_{\pi},(R,t)\right),
\end{equation}
where the minimum is over every alignment $(R,t)$ and permutation $\pi$.\\
Moreover, $(\tilde{R},\tilde{t},\tilde{\pi})$ can be computed in $n^{O(1)}$ time.
\end{theorem}
\begin{proof}
Let $(R^*,t^*,\pi^*) \displaystyle \in \argmin_{(R,t,\pi)} \cost(A_{\pi},(R,t))$.
By Theorem~\ref{theorem:costPointsToLines_proof}, we can compute a set $C \subseteq \alignments$ such that there exists an alignment $(R,t) \in C$ that satisfies
\begin{equation} \label{eq:Rtnomatching}
\cost\left(A_{\pi^*},(R,t)\right) \leq (w\cdot 16)^{r} \cdot \cost\left(A_{\pi^*},(R^*,t^*)\right)
\end{equation}
where $C$ is computed by a call to \algname$(A_{\pi^*})$; see Algorithm~\ref{Alg:slowApprox}.

In Line~\ref{Alg2line:ES} of Algorithm~\ref{Alg:slowApprox} we iterate over every triplet of indices $j,k,l \in [n]$. Each such index corresponds to a point-line matched pair. Hence, $j,k$ and $l$ correspond to a triplet of matched point-line pairs. In Lines~\ref{Alg2line:endES}, we add $O(1)$ alignments to the set $C$ that correspond to that triplet of point-line pairs. Therefore, every alignment $(R',t') \in C$ corresponds to some 3 matched point-line pairs $(p_j,\ell_{\pi^*(j)}),(p_k,\ell_{\pi^*(k)}),(p_l,\ell_{\pi^*(l)}) \in A$.
Let $(p_{i_1},\ell_{\pi^*(i_1)}),(p_{i_2},\ell_{\pi^*(i_2)}),(p_{i_3},\ell_{\pi^*(i_3)})$ be the triplet of matched pairs that corresponds to the alignment $(R,t)$. Hence, it holds that
\[
(R,t) \in \algname\left(\br{(p_{i_1},\ell_{\pi^*(i_1)}),(p_{i_2},\ell_{\pi^*(i_2)}),(p_{i_3},\ell_{\pi^*(i_3)})}\right).
\]
By iterating over every $j_1,j_2$ and $j_3 \in [n]$, when $j_1 = \pi^*(i_1), j_2 = \pi^*(i_2)$ and $j_3 = \pi^*(i_3)$, we get that
\[
(R,t) \in \algname\left(\br{(p_{i_1},\ell_{j_1}),(p_{i_2},\ell_{j_2}),(p_{i_3},\ell_{j_3})}\right).
\]
In Line~\ref{Alg4:ES} of Algorithm~\ref{Alg:slowApproxNoMatching} we iterate over every tuple of 6 indices $i_1,i_2,i_3,j_1,j_2,j_3\in [n]$, and compute, using Algorithm~\ref{Alg:slowApprox}, the set of alignments $X'$ that corresponds to the triplet of pairs $(p_{i_1},\ell_{j_1}),(p_{i_2},\ell_{j_2}),(p_{i_3},\ell_{j_3})$. We then add those alignments to the set $X$. Hence, it is guaranteed that the alignment $(R,t) \in C$ that satisfies~\eqref{eq:Rtnomatching} is also in $X$.

The bijection (matching) function $\pi:[n]\to[n]$ that minimizes the sum of distances (or in some other cost function) between the point $p_i$ and the line $\ell_{\pi(i)}$ over $i\in[n]$ can be computed in $O(n^3)$ time as explained in~\cite{kuhn1955hungarian}.
We use this algorithm to compute the optimal matching function $\hat{\pi}(A,(R',t'),\cost)$ for every $(R',t') \in X$ in Line~\ref{Alg4:mappingFunc} of Algorithm~\ref{Alg:slowApproxNoMatching}.
Let $\pi = \hat{\pi}(A,(R,t),\cost)$. Since $\pi$ is an optimal matching function for the pairs of $A$, the alignment $(R,t)$, and the function $\cost$, it satisfies that
\begin{equation} \label{eq:optMatching}
\cost\left(A_{\pi},(R,t)\right) \leq \cost\left(A_{\pi^*},(R,t)\right)
\end{equation}
Since $(R,t) \in X$ and $\pi = \hat{\pi}(A,(R,t),\cost)$, we obtain that $(R,t,\pi) \in S$, where $S$ is the set that is defined in Line~\ref{Alg4:mappingFunc} of Algorithm~\ref{Alg:slowApproxNoMatching}. Combining $(R,t,\pi) \in S$ with the definition of $(\tilde{R},\tilde{t},\tilde{\pi})$ in Line~\ref{Alg4:minS}, it yields
\begin{equation} \label{eq:minS}
\cost\left(A_{\tilde{\pi}},(\tilde{R},\tilde{t})\right) \leq \cost\left(A_{\pi},(R,t)\right).
\end{equation}

Hence, the following holds
\[
\begin{split}
\cost\left(A_{\tilde{\pi}},(\tilde{R},\tilde{t})\right)& \leq \cost\left(A_{\pi},(R,t)\right) \leq \cost\left(A_{\pi^*},(R,t)\right)\\
& \leq (w\cdot 16)^r \cdot \cost\left(A_{\pi^*},(R^*,t^*)\right),
\end{split}
\]
where the first derivation holds by~\eqref{eq:minS}, the second derivation holds by~\eqref{eq:optMatching} and the third derivation is by~\eqref{eq:Rtnomatching}.
Furthermore, the running time of the algorithm is $O(n^9) = n^{O(1)}$.
\end{proof}

\section{Coresets for Big Data}

The following theorem states Theorem 4 from~\cite{dasgupta2009sampling} for the case where $p=1$.
\begin{theorem}[\textbf{Theorem 4 in ~\cite{dasgupta2009sampling}}] \label{Mahoney}
Let $A$ be an $n \times d$ matrix of rank $r$. Then there exists an $n\times r$ matrix $U$ that satisfies the following
\renewcommand{\labelenumi}{(\roman{enumi})}
\begin{enumerate}
\item $U$ is a basis for the column space of $A$, i.e., there exists a $r\times d$ matrix $G$ such that $A = UG$.
\item $\vertiii{U}_1 \leq r^{1.5}$.
\item For every $z\in \REAL^r$, $\norm{z}_\infty \leq \norm{Uz}_1$.
\item $U$ can be computed in $O(ndr + nr^5\log{n})$ time.
\end{enumerate}
\end{theorem}

\begin{lemma} \label{lem:sensBound}
Let $P = \begin{bmatrix} p_1 \mid \ldots\mid p_n \end{bmatrix}^T \in \REAL^{n \times d}$, where $p_i \in \REAL^d$ for every $i\in [n]$. Then a function \\$s:P \to [0,\infty)$ can be computed in $O(nd^5 \log{n})$ time such that
\renewcommand{\labelenumi}{(\roman{enumi})}
\begin{enumerate}
\item For every $i\in [n]$,
\[
\sup_{x\in \REAL^d: \norm{x}>0} \frac{|p_i^T x|}{\norm{Px}_1} \leq s(p_i).
\]
\item
\[
\sum_{i=1}^n s(p_i) \in d^{O(1)}
\]
\end{enumerate}
\end{lemma}
\begin{proof}
Let $r \leq d$ denote the rank of $A$.

\textbf{(i)}: By substituting $A = P$ in Theorem~\ref{Mahoney}, we get that there exists a matrix $U \in \REAL^{n\times r}$ and a matrix $G\in \REAL^{r\times d}$ with the following properties:
\begin{equation*}
\begin{split}
&\text{(1) }P = UG,\\ &\text{(2) }\vertiii{U}_1 \leq r^{1.5} \leq d^{1.5},\\ &\text{(3) for every }z\in \REAL^r, \norm{z}_1 / r \leq \norm{z}_\infty \leq \norm{Uz}_1 and\\ &\text{(4) } U \text{ and } G \text{ can be computed in }\\ & O(ndr + nr^5\log{n}) = O(nd^5\log{n}) \text{ time.}
\end{split}
\end{equation*}
Let $u_i^T \in \REAL^r$ be the $i$th row of $U$ for $i\in [n]$.

Observe that for every $u,y \in \REAL^r$ such that $\norm{y} \neq 0$, we have that
\begin{equation} \label{uTy}
\left|u^T \frac{y}{\norm{y}_1}\right| = \left|u^T \frac{y}{\norm{y}_2}\right| \cdot \frac{\norm{y}_2}{\norm{y}_1} \leq \norm{u}_2\cdot \frac{\norm{y}_2}{\norm{y}_1} \leq \norm{u}_2,
\end{equation}
where the second derivation holds since $\frac{y}{\norm{y}_2}$ is a unit vector, and the last derivation holds since $\norm{y}_2 \leq \norm{y}_1$.
Let $x\in \REAL^d$ such that $\norm{x} > 0$. Then it follows that
\begin{equation} \label{eq:pTx}
\begin{split}
\frac{|p_i^Tx|}{\norm{Px}_1}& = \frac{|u_i^TGx|}{\norm{UGx}_1} \leq r\cdot  \frac{|u_i^TGx|}{\norm{Gx}_1}\\
& = r\cdot \left|u_i^T \frac{Gx}{\norm{Gx}_1}\right| \leq r\cdot \norm{u_i}_2 \leq r\cdot \norm{u_i}_1,
\end{split}
\end{equation}
where the first derivation holds by combining Property (1) and the definition of $u_i$, the second derivation holds by substituting $z=Gx$ in Property (3), and the fourth derivation holds by substituting $u=u_i$ and $y=Gx$ in~\eqref{uTy}.
Let $G^{+}$ be the pseudo inverse of $G$. It thus holds that $GG^{+}$ is the $r\times r$ identity matrix. Hence
\begin{equation} \label{eq:normUi}
\norm{u_i}_1 = \sum_{j=1}^{r} |u_i^T e_j| = \sum_{j=1}^{r} |u_i^T G G^{+}e_j| = \sum_{j=1}^{r} |p_i^T G^{+} e_j|,
\end{equation}
where $e_j$ is the $j$th column of the $r\times r$ identity matrix.
By combining~\eqref{eq:pTx} and~\eqref{eq:normUi}, we have
\begin{equation} \label{eq:pTxTosum}
\frac{|p_i^Tx|}{\norm{Px}_1} \leq r\cdot \sum_{j=1}^{r} |p_i^T G^{+} e_j|.
\end{equation}

Let
\[
s(p_i) = d\cdot \sum_{j=1}^{r} |p_i^T G^{+} e_j|.
\]
By combining the definition of $s(p_i)$ and that~\eqref{eq:pTxTosum} holds for every $x\in \REAL^d$, Property (i) of Lemma~\ref{lem:sensBound} holds as
\[
\begin{split}
\max_{x\in \REAL^d: \norm{x} > 0} \frac{|p_i^T x|}{\norm{Px}_1}&  \leq r\cdot \sum_{j=1}^{r} |p_i^T G^{+} e_j|\\
&\leq d\cdot \sum_{j=1}^{r} |p_i^T G^{+} e_j| = s(p_i).
\end{split}
\]

\textbf{(ii)}: Property (ii) of Lemma~\ref{lem:sensBound} holds as
\[
\begin{split}
& \sum_{i=1}^n s(p_i) = \sum_{i=1}^n \left( d\cdot \sum_{j=1}^{r} |p_i^T G^{+} e_j|\right) = d\cdot\sum_{j=1}^{r} \sum_{i=1}^n |p_i^T G^{+} e_j|\\
& = d\cdot \sum_{j=1}^{r} \norm{PG^{+}e_j}_1 = d\cdot \sum_{j=1}^{r} \norm{UGG^{+}e_j}_1\\
& = d\cdot \sum_{j=1}^{r} \norm{Ue_j}_1 \leq d^{2.5} \cdot \sum_{j=1}^{r} \norm{e_j}_1
= d^{3.5} \in d^{O(1)},
\end{split}
\]
where the first inequality holds since $\vertiii{U}_1 \leq r^{1.5}$ by Property (2).

Furthermore, the time needed to compute $s$ is dominated by the computation time of $U$ and $G$, which is bounded by $O(nd^5 \log{n})$ due to Property (4).
\end{proof}

\begin{definition} [\textbf{Definition 4.2 in~\cite{braverman2016new}}] \label{def:querySpace}
Let $P$ be a finite set, and let $w:P\to [0,\infty)$. Let $Q$ be a function that maps every set $S\subseteq P$ to a corresponding set $Q(S)$, such that $Q(T) \subseteq Q(S)$ for every $T\subseteq S$. Let $f:P\times Q(P) \to \REAL$ be a cost function. The tuple $(P,w,Q,f)$ is called a \emph{query space}.
\end{definition}


\begin{definition} [\textbf{VC-dimension}]
For a query space $(P,w,Q,f)$, $q\in Q$ and $r\in [0,\infty)$ we define
\[
\range(q,r) := \br{p\in P \mid w(p)\cdot f(p,q) \leq r},
\]
and
\[
\ranges := \br{\range(q,r) \mid q\in Q, r\in [0,\infty)}.
\]

The dimension of the query space $(P,w,Q,f)$ is the $VC$-dimension of $(P,\ranges)$.
\end{definition}

\begin{theorem} [\textbf{Theorem 5.5 in~\cite{braverman2016new}}] \label{theorem:howToCoreset}
Let $(P,w,Q,f)$ be a query space; see Definition~\ref{def:querySpace}. Let $s:P\to [0,\infty)$ such that
\[
\sup_{q\in Q(P)} \frac{w(p)f(p,q)}{\sum_{p\in P} w(p)f(p,q)} \leq s(p),
\]
for every $p\in P$ such that the denominator is non-zero. Let $t = \sum_{p\in P} s(p)$ and Let $d_{VC}$ be the dimension of query space $(P,w,Q,f)$. Let $c \geq 1$ be a sufficiently large constant and let $\varepsilon, \delta \in (0,1)$. Let $S$ be a random sample of
\[
|S| \geq \frac{ct}{\varepsilon^2}\left(d_{VC}\log{t}+\log{\frac{1}{\delta}}\right)
\]
points from $P$, such that $p$ is sampled with probability $s(p)/t$ for every $p\in P$. Let $u(p) = \frac{t\cdot w(p)}{s(p)|S|}$ for every $p\in S$. Then, with probability at least $1-\delta$, for every $q\in Q$ it holds that
\[
\begin{split}
&(1-\varepsilon)\sum_{p\in P} w(p)\cdot f(p,q)\\ &\leq \sum_{p\in S} u(p)\cdot f(p,q)\\ &\leq (1+\varepsilon)\sum_{p\in P} w(p)\cdot f(p,q).
\end{split}
\]
\end{theorem}

\begin{theorem} [Theorem~\ref{theorem:coreset}]\label{theorem:coreset_proof}
Let $d \geq 2$ be an integer. Let $A=\br{(p_1,\ell_1),\cdots,(p_n,\ell_n)}$ be set of $n$ pairs, where for every $i \in [n]$, $p_i$ is a point and $\ell_i$ is a line, both in $\REAL^d$, and let $w = (w_1,\cdots,w_n) \in [0,\infty)^n$.
Let $\varepsilon, \delta \in (0,1)$.
Then in $nd^{O(1)} \log{n}$ time we can compute a weights vector $u = (u_1,\cdots,u_n)\in[0,\infty)^n$ that satisfies the following pair of properties.
\renewcommand{\labelenumi}{(\roman{enumi})}
\begin{enumerate}
\item With probability at least $1-\delta$, for every $(R,t) \in \alignments$ it holds that
\[
\begin{split}
& (1-\varepsilon) \cdot \sum_{i\in [n]} w_i\cdot\dist(Rp_i-t,\ell_i)\\
& \leq \sum_{i\in [n]}u_i\cdot\dist(Rp_i-t,\ell_i)\\
& \leq (1+\varepsilon) \cdot \sum_{i\in [n]} w_i\cdot\dist(Rp_i-t,\ell_i).
\end{split}
\]
\item The weights vector $u$ has $\frac{d^{O(1)}}{\varepsilon^2}\log{\frac{1}{\delta}}$ non-zero entries.
\end{enumerate}
\end{theorem}
\begin{proof}
Put $(R,t)\in \alignments$. We represent a line $\ell$ by a basis of its orthogonal complement $V\in \REAL^{(d-1)\times d}$ and its translation $b$ from the origin. Formally, let $V \in\REAL^{(d-1)\times d}$ be a matrix whose rows are mutually orthogonal unit vectors, and $b \in \REAL^{d-1}$ such that
\[
\ell = \left\{q\in\REAL^d \big| \norm{Vq-b} = 0\right\}.
\]
Identify the $d$ rows of $R$ by $R_{1*},...R_{d*}$, $b = (b_1,\cdots,b_{d-1})^T$ and let
\[
x(R,t) = (R_{1*}\mid\cdots\mid R_{d*}\mid -t^T\mid -1)^T.
\]
Let $p\in \REAL^d$, $\omega \geq 0$ and for every $k\in [d-1]$ let
\[
s_{k} = \omega\cdot \left(V_{k,1}p^T \mid\cdots\mid V_{k,d}p^T\mid V_{k*}\mid b_{k}\right)^T,
\]
where the $k$th row of $V$ is $V_{k*}=(V_{k,1},\cdots V_{k,d})$
Then it holds that
\begin{align}
& \omega\dist(Rp-t,\ell)
= \omega\norm{V(Rp-t)-b} \nonumber \\
& = \omega\norm{V(R_{1*} p,\cdots,R_{d*} p)^T -V t -b}= \nonumber\\
& \omega\norm{
\begin{pmatrix}
  V_{1*} (R_{1*} p,\cdots,R_{d*} p)^T -V_{1*} t -b_{1} \\
  \vdots \\
  V_{(d-1)*} (R_{1*} p,\cdots,R_{d*} p)^T -V_{(d-1)*} t -b_{d-1} \\
\end{pmatrix}} = \nonumber\\
& \omega\norm{
\begin{pmatrix}
  x(R,t)^T \left( V_{1,1} p^T\mid\cdots\mid V_{1,d}p^T\mid V_{1*}\mid b_{1}\right)^T \\
  \vdots \\
  x(R,t)^T \left( V_{d-1,1} p^T\mid \cdots\mid V_{d-1,d}p^T\mid V_{(d-1)*}\mid b_{d-1}\right)^T \\
\end{pmatrix}} \label{eq33}\\
& = \norm{
\begin{pmatrix}
  x(R,t)^T s_{1} \\
  \vdots \\
  x(R,t)^T s_{d-1} \\
\end{pmatrix}}, \label{eq4}
\end{align}
where~\eqref{eq33} holds by the definition of $x(R,t)$ and~\eqref{eq4} holds by the definition of $s_{k}$ for every $k\in [d-1]$.
Hence,
\begin{equation} \label{eq:dist}
\omega\cdot \dist(Rp-t,\ell) = \norm{
\begin{pmatrix}
  x(R,t)^T s_{1} \\
  \vdots \\
  x(R,t)^T s_{d-1} \\
\end{pmatrix}}.
\end{equation}

We have $\frac{\norm{v}_1}{\sqrt{d-1}} \leq \norm{v}_2 \leq \norm{v}_1$ for every $v\in \REAL^{d-1}$. Plugging this in~\eqref{eq:dist} yields
\begin{equation} \label{eq:distToNorm}
\begin{split}
\frac{1}{\sqrt{d}}\cdot\norm{
\begin{pmatrix}
  x(R,t)^T s_{1} \\
  \vdots \\
  x(R,t)^T s_{d-1} \\
\end{pmatrix}}_1&  \leq \omega\cdot\dist(Rp-t,\ell)\\
& \leq \norm{
\begin{pmatrix}
  x(R,t)^T s_{1} \\
  \vdots \\
  x(R,t)^T s_{d-1} \\
\end{pmatrix}}_1.
\end{split}
\end{equation}

For every $i\in [n]$ and $k\in [d-1]$, let $V_i \in\REAL^{(d-1)\times d}$ and $b_i \in \REAL^{d-1}$ such that $\ell_i = \left\{q\in\REAL^d \big| \norm{V_iq-b_i} = 0\right\}$, and let $s_{i,k} = w_i \cdot \left({V_i}_{(k,1)}p_i^T,\cdots,{V_i}_{(k,d)}p_i^T,{V_i}_{k*},{b_i}_{(k)}\right)^T$.
Similarly to~\eqref{eq:distToNorm}, for every $i\in [n]$,
\begin{equation} \label{eq:distToNorm_i}
\begin{split}
\frac{1}{\sqrt{d}}\cdot\norm{
\begin{pmatrix}
  x(R,t)^T s_{i,1} \\
  \vdots \\
  x(R,t)^T s_{i,d-1} \\
\end{pmatrix}}_1& \leq w_i \cdot\dist(Rp_i-t,\ell_i)\\
& \leq \norm{
\begin{pmatrix}
  x(R,t)^T s_{i,1} \\
  \vdots \\
  x(R,t)^T s_{i,d-1} \\
\end{pmatrix}}_1.
\end{split}
\end{equation}
Let
\[
S_i = \begin{bmatrix} s_{i,1} \mid\ldots \mid s_{i,d-1} \end{bmatrix} \in \REAL^{(d^2+d+1) \times (d-1)},
\]
and let
\[
S = \begin{bmatrix} S_1 \mid\ldots \mid S_n \end{bmatrix} \in \REAL^{(d^2+d+1) \times (n(d-1))}.
\]
Then the following holds
\begin{align}
&\sup_{(R,t)} \frac{w_i\cdot\dist(Rp_i-t,\ell_i)}{\sum_{j\in [n]}w_j\cdot\dist(Rp_j-t,\ell_j)}\\
& \leq \sup_{(R,t)} \frac{\norm{
\begin{pmatrix}
  x(R,t)^T s_{i,1} \\
  \vdots \\
  x(R,t)^T s_{i,d-1} \\
\end{pmatrix}}_1}{\sum_{j\in [n]} \frac{1}{\sqrt{d}} \cdot\norm{
\begin{pmatrix}
  x(R,t)^T s_{j,1} \\
  \vdots \\
  x(R,t)^T s_{j,d-1} \\
\end{pmatrix}}_1} \label{eqSens1}\\
& = \sqrt{d} \cdot \sup_{(R,t)} \frac{\sum_{k \in [d-1]}|x(R,t)^Ts_{i,k}|}{\sum_{j \in [n]} \sum_{m\in [d-1]} |x(R,t)^Ts_{j,m}|} \nonumber\\
& = \sqrt{d} \cdot \sup_{(R,t)} \frac{\sum_{k \in [d-1]}|x(R,t)^Ts_{i,k}|}{\norm{x(R,t)^T S}_1} \label{eqSens2}\\
& = \sqrt{d} \cdot  \sup_{(R,t)} \left(\frac{|x(R,t)^Ts_{i,1}|}{\norm{x(R,t)^T S}_1} +\cdots + \frac{|x(R,t)^Ts_{i,d-1}|}{\norm{x(R,t)^T S}_1} \right) \nonumber\\
& \leq \sqrt{d} \cdot \left( \sup_{(R,t)} \frac{|x(R,t)^Ts_{i,1}|}{\norm{x(R,t)^T S}_1} +\cdots + \sup_{(R,t)} \frac{|x(R,t)^Ts_{i,d-1}|}{\norm{x(R,t)^T S}_1} \right) \label{eqSens3}\\
& = \sqrt{d} \cdot \sum_{k\in [d-1]} \sup_{(R,t)} \frac{|x(R,t)^Ts_{i,k}|}{\norm{x(R,t)^T S}_1}\\
& \leq \sqrt{d} \cdot \sum_{k\in [d-1]} \sup_{x \in \REAL^{d^2+d+1} :\norm{x}>0} \frac{|x^Ts_{i,k}|}{\norm{x^T S}_1} \label{eqSens4},
\end{align}
where the supremum is over every alignment $(R,t)$ such that $\sum_{j\in [n]} w_j\cdot \dist(Rp_j-t,\ell_j) \neq 0$, and where~\eqref{eqSens1} holds by~\ref{eq:distToNorm_i}, \eqref{eqSens2} holds by the definition of $S$, \eqref{eqSens2} holds since the maximum of a sum is at most the sum of maxima, and~\eqref{eqSens4} holds since $\br{x(R,t) \mid (R,t) \in \alignments} \subseteq \REAL^{d^2+d+1} \setminus{\vec{0}}$.

Therefore,
\begin{equation} \label{eq:distToxTp}
\begin{split}
& \sup_{(R,t)} \frac{w_i\cdot\dist(Rp_i-t,\ell_i)}{\sum_{j\in [n]}w_j\cdot\dist(Rp_j-t,\ell_j)}\\
& \leq \sqrt{d} \cdot \sum_{k\in [d-1]} \sup_{x \in \REAL^{d^2+d+1}:\norm{x}>0} \frac{|x^Ts_{i,k}|}{\norm{x^T S}_1}.
\end{split}
\end{equation}

By substituting $P = S^T$ in Lemma~\ref{lem:sensBound}, we get that in $nd^{O(1)}\log{n}$ time we can compute a function $s$ that maps every column $s_{i,k}$ of $S$ to a non-negative real such that the following two properties hold for every $i\in [n]$ and $k\in [d-1]$:
\begin{equation} \label{eq:prop_i}
\sup_{x \in \REAL^{d^2+d+1}:\norm{x}>0} \frac{|x^Ts_{i,k}|}{\norm{x^T S}_1} \leq s(s_{i,k}), \text{ and }
\end{equation}
\begin{equation} \label{eq:prop_ii}
\sum_{i\in [n],k\in [d-1]} s(s_{i,k}) \in d^{O(1)}.
\end{equation}

Let $s':A\to [0,\infty)$ such that $s'\left((p_i,\ell_i)\right) = \sqrt{d} \cdot \sum_{k=1}^{d-1} s(s_{i,k})$ for every $i\in [n]$. Then we have
\begin{equation}
\begin{split}
& \sup_{(R,t)} \frac{w_i\cdot\dist(Rp_i-t,\ell_i)}{\sum_{j\in [n]}w_j\cdot\dist(Rp_j-t,\ell_j)}\\
&\leq \sqrt{d} \cdot \sum_{k\in [d-1]} \sup_{x \in \REAL^{d^2+d+1}:\norm{x}>0} \frac{|x^Ts_{i,k}|}{\norm{x^T S}_1}\\
&\leq \sqrt{d} \cdot \sum_{k\in [d-1]} s(s_{i,k})
= s'\left((p_i,\ell_i)\right),
\end{split}
\end{equation}
where the first derivation holds by~\eqref{eq:prop_i}, and the last is by the definition of $s'$.

We also have that
\begin{equation}
\begin{split}
\sum_{i\in [n]} s'\left((p_i,\ell_i)\right)& = \sum_{i\in [n]} \left(\sqrt{d} \cdot \sum_{k\in [d-1]} s(s_{i,k})\right)\\
& \in \sqrt{d} \cdot \sum_{i\in [n],k\in [d-1]} s(s_{i,k}) \in \sqrt{d}\cdot {d}^{O(1)}\\
& = d^{O(1)},
\end{split}
\end{equation}
where the first derivation is by the definition of $s'$ and the third derivation is by~\eqref{eq:prop_ii}.

The dimension of the corresponding query space is bounded by $d_{VC} \in O(d^2)$ by~\cite{anthony2009neural}.

Let $t = \sum_{i\in [n]} s'\left((p_i,\ell_i)\right)$.
Let $Z \subseteq A$ be a random sample of
\[
|Z| \in \frac{d^{O(1)}}{\varepsilon^2}\left(d^2\log{d}+\log{\frac{1}{\delta}}\right)
=\frac{d^{O(1)}}{\varepsilon^2}\log{\frac{1}{\delta}}
\]
pairs from $A$, where $(p,\ell) \in A$ is sampled with probability $s'\left((p,\ell)\right)/t$ and let $u = (u_1,\cdots,u_n)$ where
\[
u_i = \begin{cases} \frac{t\cdot w_i}{s'\left((p_i,\ell_i)\right)|Z|}, & \text{ if } (p_i,\ell_i)\in Z\\
0, & \text{ otherwise} \end{cases},
\]
for every $i\in [n]$.
By substituting $P=A$, $Q(\cdot)\equiv\alignments$, $f\left((p_i,\ell_i),(R,t)\right)=w_i\cdot\dist(Rp_i-t_i,\ell_i)$ for every $i\in [n]$, $s=s'$, $t \in d^{O(1)}$ and $d_{VC}=O(d^2)$, in Theorem~\ref{theorem:howToCoreset}, Property (i) of Theorem~\ref{theorem:coreset} holds as
\[
\begin{split}
& (1-\varepsilon) \cdot \sum_{i\in [n]} w_i\cdot\dist(Rp_i-t,\ell_i) \leq \sum_{i\in [n]}u_i\cdot\dist(Rp_i-t,\ell_i)\\
& \leq (1+\varepsilon) \cdot \sum_{i\in [n]} w_i\cdot\dist(Rp_i-t,\ell_i).
\end{split}
\]

Furthermore, Property (i) of Theorem~\ref{theorem:coreset_proof} holds since the number of non-zero entries of $u$ is equal to $|Z| \in \frac{d^{O(1)}}{\varepsilon^2}\log{\frac{1}{\delta}}$.

The time needed to compute $u$ is bounded by the computation time of $s$, which is bounded by $nd^{O(1)}\log{n}$.
\end{proof}

\begin{corollary} [Corollary~\ref{cor:streamingCoreset}]\label{cor:streamingCoreset_proof}
Let $A=\br{(p_1,\ell_1),(p_2,\ell_2,),\cdots}$ be a (possibly infinite) stream of pairs, where for every $i \in [n]$, $p_i$ is a point and $\ell_i$ is a line, both in the plane. Let $\varepsilon,\delta \in (0,1)$.
Then, for every integer $n>1$ we can compute with probability at least $1-\delta$ an alignment $(R^*,t^*)$ that satisfies
\[
\sum_{i=1}^{n} \dist(R^*p_i-t^*,\ell_i) \in O(1) \cdot \min_{(R,t)\in \alignments} \sum_{i=1}^{n} \dist(Rp_i-t,\ell_i),
\]
for the $n$ points seen so far in the stream, using $(\log(n/\delta)/\eps)^{O(1)}$ memory and update time per a new pair.
Using $M$ machines the update time can be reduced by a factor of $M$.
\end{corollary}
\begin{proof}
The coreset from Theorem~\ref{theorem:pointsToLinesNoMatching_proof} is composable by its definition (see Section~\ref{gen}).
The claim then follows directly for the traditional merge-and-reduce tree technique as explained in many coreset papers, e.g.~\cite{AHV04,FSS13,bent,barger2016k,feldman2015more}.
\end{proof}

\subsection{Experimental Results}
\label{ch:ER}

The main contribution of this paper is the theorems and theoretical results in previous sections. Nevertheless, since the algorithms are relatively simple we were able to implement them in an open source library~\cite{opencode} in order to demonstrate their correctness and robustness. In this section we suggest preliminary experimental results on synthetic and real data, and some comparison to existing software.

\paragraph{Hardware.} All the following tests were conducted using MATLAB R2016b on a Lenovo Y700 laptop with an Intel i7-6700HQ CPU and 16GB RAM.

\subsubsection{Synthetic Data}
\paragraph{The input data-set} was generated by first computing a set $L = \{\ell_1,\cdots,\ell_n\}$ of $n$ random lines in the plane. then, a set $P_0$ of $n$ random points was generated, each in the range $[0,100]^2$. A set $P_1$ was then computed by projecting the $i$th point of the points in $P_0$ onto its corresponding line $\ell_i$ in $L$. The set $P_1$ was then translated and rotated, respectively, by a random rotation matrix $R\in\REAL^{2\times 2}$ and a translation vector $t\in[0,10]^2$ to obtain a set $P_2$. Then, a random noise $N_i\in[0,1]^2$ was added for every point $p_i\in P_2$, and  additional noise $\hat{N}_i\in[0,r]^2,r=200$ was added to a subset $S\subseteq P_2$ of $|S| = k\cdot |P_2|$ points (outliers), where $k\in[0,1]$ is the percentage of outliers. The resulting set was denoted by $P=\br{p_1,\cdots,p_n}$.

The randomness in the previous paragraph is defined as follows. Let  $N(\mu,\sigma)$ denote a Gaussian distribution with mean $\mu$ and standard deviation $\sigma$. Similarly, let $U(r)$ denote uniform distribution over $[0,r]$. For every $i\in[n]$ and $j\in[2]$, the line $\ell_i=\br{x\in\REAL^2 \mid a_i^Tx=b_i}$ was constructed such that $a_{i,j} \in N(1/2,1/2), \norm{a_i}=1$ and $b_i \in U(10)$. For every $i\in[2],j\in[2]$, we choose $t_j\in U(10)$, $N_{i,j} \in N(1/2,1/2)$, and $\hat{N}_{i,j}\in N(r/2,r/2)$.
The rotation matrix was chosen such that $R_{1,1}\in U(1)$ and $R^TR=I$.

\paragraph{Experiment. }The goal of each algorithm in the experiment was to compute $\tilde{R}$ and $\tilde{t}$, approximated rotation $R$ and translation $t$, respectively, using only the set $L$ of lines and the corresponding noisy set $P$ of points. Two tests were conducted, one with a constant number of input points $n$, with increasing percentage $k$ of outliers, and another test with a constant value for $k$, while increasing the number $n$ of input points. For the first test we compare the threshold M-estimator of the sum of fitting distances
\begin{equation} \label{minHuber}
\cost(P,L,\tilde{R},\tilde{t}) = \sum_{i=1}^n \min \{\dist(\tilde{R}p_i-\tilde{t},\ell_i),th\}.
\end{equation}
For the second test we compare the running time (in second) of the suggested methods.

\paragraph{Algorithms. }We apply each of the following algorithms that gets the pair of sets $P$ and $L$ as input.\\
\textbf{Our \fastalgname} samples a set of $\sqrt[3]{n}$ points $P'$ from $P$. Let $L'\subseteq L$ be lines in $L$ corresponding to the points in $P'$. It then runs Algorithm~\ref{Alg:slowApprox} with the set of corresponding point-line pairs from $P'$ and $L'$. The algorithm returns a set  $C$ of alignments. We then chose the alignment $(\tilde{R},\tilde{t})\in C$ that minimizes Eq.~\eqref{minHuber} using $th=10$. \label{ER:ApproxAlg}\\
\textbf{Algorithm \LMSalg}returns an alignment $(\tilde{R},\tilde{t})$ that minimizes the sum of \emph{squared} fitting distances (Least Mean Squared), i.e., $(\tilde{R},\tilde{t}) \in \argmin_{(R,t)}\sum_{i=1}^n \dist^2(Rp_i-t,\ell_i)$, where the minimum is over every valid alignment $(R,t)$. This is done by solving the set of polynomials induced by the sum of squared fitting distances function using the Lagrange Multipliers method. \label{ER:OPTAlg}\\
\textbf{Adaptive \ransacalg + Algorithm \LMSalg}is an iterative method that uses $m$ iterations, proportional to the (unknown) percentage of outliers. See~\cite{raguram2008comparative}. At the $j^{th}$ iteration it samples 3 pairs of corresponding points and lines from $P$ and $L$, respectively, it then computes their optimal alignment $(R_j,t_j)$ using Algorithm \LMSalg, and then updates the percentage of inliers found based on the current alignment. A point $p_i$ is defined as an inlier in the $j^{th}$ iteration if $\dist(R_jp_i-t_j,\ell_i)<th$, where $th=10$. It then picks the alignment $(\tilde{R},\tilde{t})$ that has the maximum number of inliers over every alignment in $\{(R_1,t_1),\cdots,(R_m,t_m)\}$. \label{ER:RANSACOPTAlg}\\
\textbf{Adaptive \ransacalg + our \fastalgname}similar to the previous Adaptive \ransacalg Algorithm, but instead of using Algorithm \LMSalg, it uses our \fastalgname to compute the alignment of the triplet of pairs at each iteration. \label{ER:RANSACApproxAlg}

Each test was conducted $10$ times. The average sum of fitting distances, along with the standard deviation are shown in Fig.~\ref{fig:Error_N50}, Fig.~\ref{fig:Error_N100} and Fig.~\ref{fig:Error_N300}. The average running times, along with the standard deviation, are shown in Fig.~\ref{fig:time_OL0}, Fig.~\ref{fig:time_OL01} and Fig.~\ref{fig:time_OL05}.

\begin{figure*}[t!]
    \centering
    \begin{subfigure}[t]{0.32\textwidth}
        \centering
        \includegraphics[width=\textwidth]{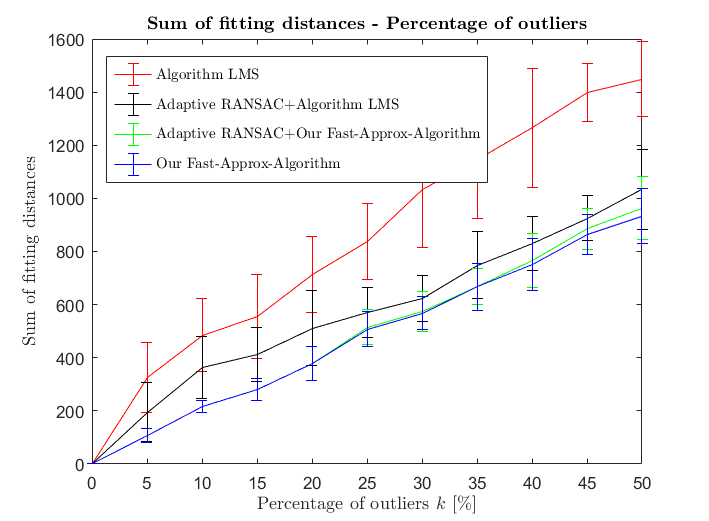}
        \caption{$n=50$.}
        \label{fig:Error_N50}
    \end{subfigure}%
    \begin{subfigure}[t]{0.32\textwidth}
        \centering
        \includegraphics[width=\textwidth]{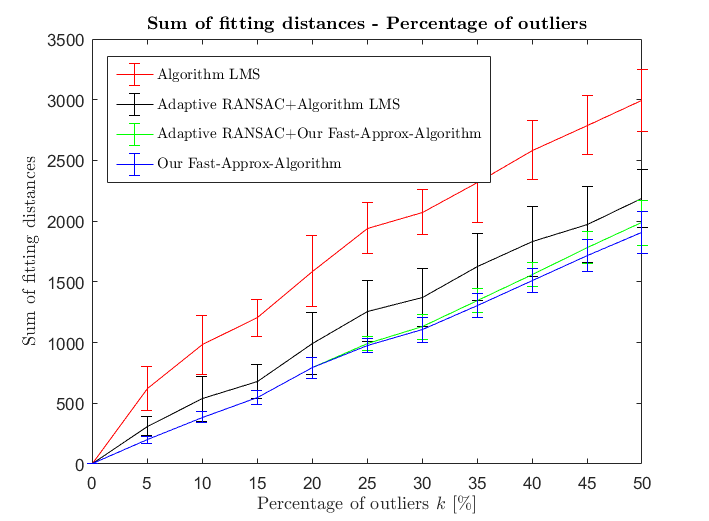}
        \caption{$n=100$.}
        \label{fig:Error_N100}
    \end{subfigure}
    \begin{subfigure}[t]{0.32\textwidth}
        \centering
        \includegraphics[width=\textwidth]{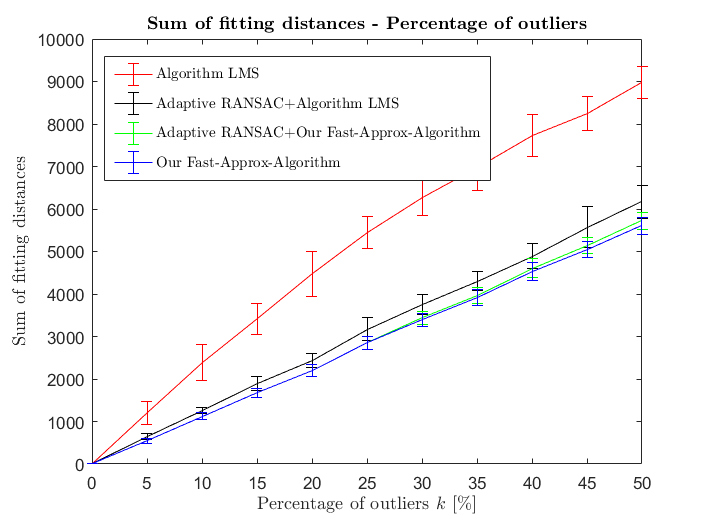}
        \caption{$n=300$.}
        \label{fig:Error_N300}
    \end{subfigure}
    \caption{Error Comparison. The $x$-axes represents the percentage $k\in [0,100]$ of outliers. The $y$-axes represents the average sum of fitting distances over $10$ tests. The number of input points is denoted by $n$. The test was conducted several times with different values for $n$.}
   	\label{fig:ErrorFig}
\end{figure*}

\begin{figure*}[t!]
    \centering
    \begin{subfigure}[t]{0.32\textwidth}
        \centering
        \includegraphics[width=\textwidth]{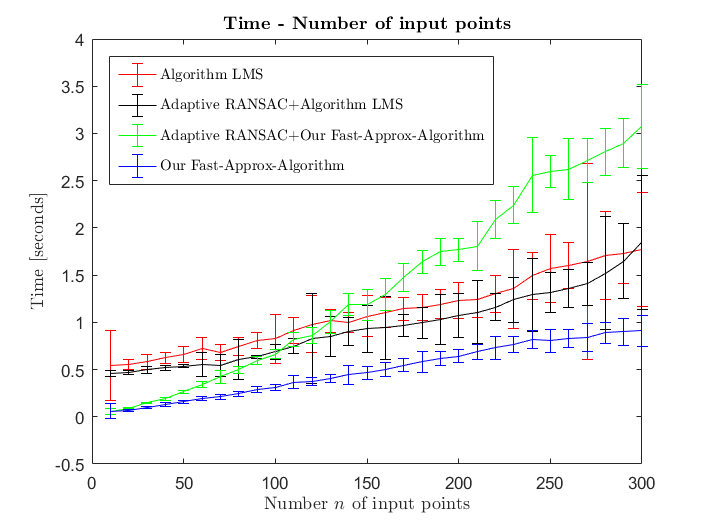}
        \caption{$k=0\%$.}
        \label{fig:time_OL0}
    \end{subfigure}%
    \begin{subfigure}[t]{0.32\textwidth}
        \centering
        \includegraphics[width=\textwidth]{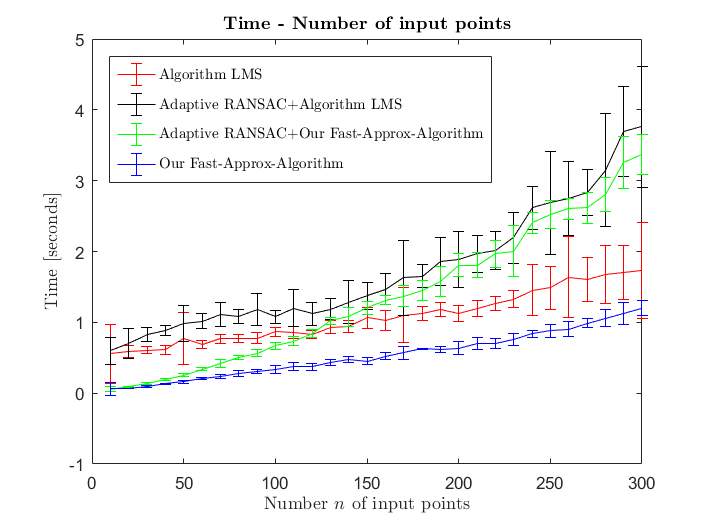}
        \caption{$k=10\%$.}
        \label{fig:time_OL01}
    \end{subfigure}
    \begin{subfigure}[t]{0.32\textwidth}
        \centering
        \includegraphics[width=\textwidth]{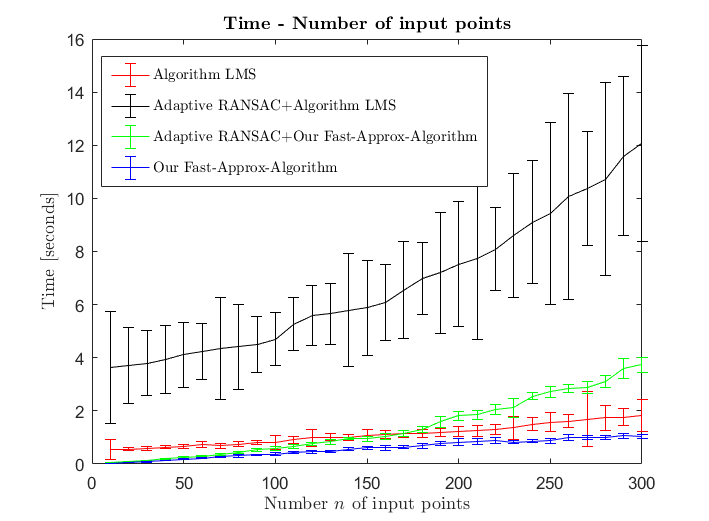}
        \caption{$k=50\%$.}
        \label{fig:time_OL05}
    \end{subfigure}
    \caption{Time Comparison. The $x$-axes represents the number $n$ of input points. The $y$-axes represents the average time, over $10$ tests, required to compute an alignment $(\tilde{R}.\tilde{t})$. The percentage of outliers is denoted by $k$. The test was conducted several times with different values for $k$.}
	\label{fig:TimeFig}
\end{figure*}

\subsubsection{Real Data}
We have conducted a test with real-world data to emphasize the potential use of our algorithm in real-world applications; See video in~\cite{DemoVideo} or in the supplementary material.


\paragraph{Experiment: Potential application for Augmented Reality.} A small camera was mounted on a pair glasses. The glasses were worn by a person (me), while observing the scene in front of him, as shown in Fig~\ref{fig:UrsaMajor}. The goal was to insert 2D virtual objects into the video of the observed scene, while keeping them aligned with the original objects in the scene.

\paragraph{The experiment.} At the first frame of the video we detect a set $P$ of ''interesting points" (features) using a SURF feature detector~\cite{bay2006surf}, and draw virtual objects on top of the image. We track the set $P$ throughout the video using the KLT algorithm~\cite{lucas1981iterative}. Let $Q$ denote the observed set of points in a specific frame. In every new frame we use Hough Transform~\cite{duda1972use} to detect a set $L$ of lines, and match them naively to the set $P$: every point $p\in P$ is matched to its closest line $\ell \in L$ among the yet unmatched lines in $L$. We then apply the following algorithms. In practice, we noticed that the approach of detecting a set of \textbf{lines} in the currently observed image rather than detecting a set of interest \textbf{points}, is more robust to noise, since a line in the image is much more stable than a point / corner.

\noindent\textbf{Algorithm \LMSalg}begins similar to Algorithm \LMSalg from the previous test. We then apply the output alignment of this algorithm to the initial virtual object, in order to estimate its location in the current frame.

\noindent\textbf{Adaptive \ransacalg -Homography }gets the paired sets $P$ and $Q$, and computes a Homography mapping~\cite{marchand2016pose} represented as a matrix $H\in\REAL^{3\times 3}$. This is an iterative method that uses $m$ iterations, proportional to the (unknown) percentage of outliers. At the $j^{th}$ iteration it samples $4$ pairs of corresponding points from $P$ and $Q$, respectively, it then computes a Homography mapping $H_j$ using those $4$ pairs, and then updates the percentage of inliers found based on the current alignment. The $i$th pair $(p_i,q_i)$ for every $i\in [n]$ is defined as an inlier in the $j^{th}$ iteration if $\norm{p_i'-q_i}<1$, where $p_i'$ is the result of applying $H_j$ to the point $p_i$. It then picks the Homography $\tilde{H}$ that has the maximum number of inliers over every Homography in $\{H_1,\cdots,H_m\}$.
We estimate the location of the virtual object in the current frame by applying $\tilde{H}$ to the initial virtual object.
We estimate the location of the virtual object in the current frame using $H$.
\textbf{Our \fastalgname}begins similar to our \fastalgname from the previous test. We then apply the output alignment of this algorithm to the initial virtual object, in order to estimate it's location in the current frame.


\section{Discussion}
As shown in Fig~\ref{fig:ErrorFig}, the sum of fitting distances (error) and the error standard deviation of our \fastalgname algorithm is consistently lower than the error and standard deviation of the state of the art methods. The error and the standard deviation of the state of the art methods go respectively up to $3$ times and $5$ times higher than those of our algorithm. The significance of the error and standard deviation gaps between the methods is emphasized in the attached video clip where the fitting error of the competitive algorithms looks very noisy, while the shape in our algorithms seems stable.

As shown in Fig~\ref{fig:TimeFig}, also the running time and its standard deviation for our \fastalgname algorithm is consistently and significantly lower than the state of the art methods.
In most cases, the running time and its standard deviation for the state of the art methods go respectively up to $5$ times and $4$ times higher compared to those of our algorithm. In other cases the gap is even larger.
Due to that, we had to compute the results of the competitive algorithms off-line, as they do not support the real-time update.

Furthermore, as demonstrated in the graphs, the sum of fitting distances, the running time and the standard deviation of the \ransacalg+\fastalgname was significantly reduces compared to the \ransacalg+Algorithm \LMSalg. This implies that our algorithm can also be used to enhance other state of the art methods.
 Notice that the error ratio between our algorithm and Algorithm \LMSalg can be made arbitrary large by increasing the value of $r$ (the range of the noise added to the outliers).
This holds since Algorithm \LMSalg is sensitive to outliers, while our algorithm is robust to outliers, due to the use of the M-estimator.

\end{document}